\documentclass[11pt]{article}
\usepackage[margin=1.5in]{geometry}
\RequirePackage[T1]{fontenc}
\usepackage{times}
\RequirePackage{amssymb,amsbsy,amsmath,amscd,amsthm,amsfonts,bbm, mathrsfs}
\RequirePackage[numbers]{natbib}
\RequirePackage[colorlinks,citecolor=black,urlcolor=black]{hyperref}
\usepackage{graphicx}
\usepackage{enumerate,xspace}
\usepackage{amsfonts}
\usepackage{amssymb}
\usepackage{fancyhdr}
\usepackage{comment}
\usepackage{parskip}
\usepackage{float}
\usepackage{natbib}
\usepackage{longtable}
\usepackage{tablefootnote}
\usepackage{url}
\usepackage{standalone}
\restylefloat{table}
\newtheorem{definition}{Definition}
\newtheorem{theorem}{Theorem}

\newtheorem{example}{Example}
\newtheorem{lemma}{Lemma}
\newtheorem{remark}{Remark}

\newtheorem{corollary}{Corollary}
\newtheorem{assumption}{Assumption}

\newcommand{\Binom}{\mathrm{Bin}}

\usepackage{tikz}
\usetikzlibrary{matrix,arrows,calc,shapes,backgrounds}
\usetikzlibrary{shapes.callouts,decorations.text} 
\usetikzlibrary{shapes.misc}

\newcommand{\Unif}{\mathrm{Uniform}}
\newcommand{\Beta}{\mathrm{Beta}}

\newcommand{\Expect}{\mathbb{E}}
\newcommand{\expect}[1]{\mathbb{E}\left[#1\right]}
\newcommand{\Prob}{\mathbb{P}}
\newcommand{\prob}[1]{\mathbb{P}\left[#1\right]}

\newcommand{\Var}{\mathsf{Var}}

\newcommand{\Indc}{\mathbbm{1}}
\newcommand{\indc}[1]{\Indc_{\left\{{#1}\right\}}}

\newcommand{\calD}{{\mathcal{D}}}

\newcommand{\calF}{{\mathcal{F}}}

\newcommand{\calM}{{\mathcal{M}}}

\newcommand{\calP}{{\mathcal{P}}}

\newcommand{\calS}{{\mathcal{S}}}
\newcommand{\calT}{{\mathcal{T}}}

\newcommand{\Th}{{^{\rm th}}}

\usepackage{prettyref,xspace}
\newrefformat{eq}{(\ref{#1})}
\newrefformat{thm}{Theorem~\ref{#1}}
\newrefformat{sec}{Section~\ref{#1}}
\newrefformat{algo}{Algorithm~\ref{#1}}
\newrefformat{fig}{Fig.~\ref{#1}}
\newrefformat{tab}{Table~\ref{#1}}
\newrefformat{rmk}{Remark~\ref{#1}}
\newrefformat{def}{Definition~\ref{#1}}
\newrefformat{cor}{Corollary~\ref{#1}}
\newrefformat{lmm}{Lemma~\ref{#1}}
\newrefformat{prop}{Proposition~\ref{#1}}
\newrefformat{app}{Appendix~\ref{#1}}
\newrefformat{ex}{Example~\ref{#1}}
\newrefformat{ass}{Assumption~\ref{#1}}

\DeclareMathOperator*{\argmin}{arg\,min}
\DeclareMathOperator*{\argmax}{arg\,max}

\usepackage{float}
\usepackage[font=small,labelfont=bf]{caption}
\usepackage[font=small,labelfont=bf]{subcaption}

\usepackage{xr}
\def\ba{\mathbf{a}}
\def\bb{\mathbf{b}}

\def\bt{\mathbf{t}}
\def\bx{\mathbf{x}}

\def\bX{\mathbf{X}}

\def\bmu{\boldsymbol{\mu}}
\def\bbeta{\boldsymbol{\beta}}

\def\bSigma{\boldsymbol{\Sigma}}

\providecommand{\keywords}[1]{\textbf{\text{Index terms ---}} #1}
\usepackage{algorithm2e}

\restylefloat{figure}
\title{Analyzing CART}

\author{
  Jason M.~Klusowski\thanks{This research was supported in part by NSF Grant DMS-1915932.} \\
  Department of Statistics\\
  Rutgers University -- New Brunswick\\
  Piscataway, NJ, USA, 8019  \\
  \href{mailto:jason.klusowski@rutgers.edu}{jason.klusowski@rutgers.edu}
}

\date{June 24, 2019}

\begin{document}

\maketitle

\begin{abstract}
Decision trees with binary splits are popularly constructed using Classification and Regression Trees (CART) methodology. For binary classification and regression models, this approach recursively divides the data into two near-homogenous daughter nodes according to a split point that maximizes the reduction in sum of squares error (the impurity) along a particular variable.
This paper aims to study the bias and adaptive properties of regression trees constructed with CART. In doing so, we derive an interesting connection between the bias and the mean decrease in impurity (MDI) measure of variable importance---a tool widely used for model interpretability---defined as the sum of impurity reductions over all non-terminal nodes in the tree. In particular, we show that the probability content of a terminal subnode for a variable is small when the MDI for that variable is large and that this relationship is exponential---confirming theoretically that decision trees with CART have small bias and are adaptive to signal strength and direction. Finally, we apply these individual tree bounds to tree ensembles and show consistency of Breiman's random forests. The context is surprisingly general and applies to a wide variety of multivariable data generating distributions and regression functions. The main technical tool is an exact characterization of the conditional probability content of the daughter nodes arising from an optimal split, in terms of the partial dependence function and reduction in impurity.

\end{abstract}


\section{Introduction}
Decision trees are the building blocks of some of the most important and powerful algorithms in statistical learning. For example, ensembles of decision trees are used for some bootstrap aggregated prediction rules (e.g., random forests \citep{breiman2004}). In addition, at each iteration of gradient tree boosting (e.g., TreeBoost \citep{friedman2001greedy}), the pseudo-residuals are fit with decision trees as base learners. From an applied perspective, decision trees have an appealing interpretability and are accompanied by a rich set of analytic and visual diagnostic tools. These attributes make tree-based learning particularly well-suited for applied sciences and related disciplines---which may rely heavily on understanding and interpreting output from a black-box model and the system that generated the data.

Although, as with many aspects of statistical learning, good empirical performance often comes at the expense of rigor and transparency.\footnote{For a delightful read and intriguing perspective on this trade-off, see \citep{breiman2001statistical}.} Tree-structured learning with decision trees is no exception---statistical guarantees for popular variants, i.e., those that are actually used in practice, are hard to find. The complicated recursive way in which decision trees are constructed makes them unamenable to analysis. While a complete and thorough picture of decision trees and their role in ensemble learning may be far away or even unattainable, in this paper, we take a significant step forward and aim to tackle the following three questions.

\begin{itemize}
\item Why do decision trees have small bias?
\item Why are decision trees locally adaptive to signal strength?
\item Can one connect bias and adaptivity with quantities used to assess variable importance?
\end{itemize}

The first two questions above are supported by an abundance of empirical evidence \citep{breiman1984}, yet remain to be explained or answered by a sensible mathematical theory, especially when the tree construction involves the input and output data. For example, decision trees are known to have small bias when they are grown deeply---a defining characteristic of random forests. But how and why? Moreover, tree-based learning is known to be particularly effective in high-dimensional sparse settings when the distribution of the output depends only on a few predictor variables. What mechanism of trees enables this?

Let us emphasize that we are not content with merely a study of the standard certificates for good predictors (e.g., consistency or rates of convergence). Rather, we aim to identity and explore the \emph{unique} advantages of tree-structured learning and, in doing so, develop a unifying theory that connects bias and adaptivity via two well-known data-analytic tools for model interpretability---\emph{partial dependence functions} and \emph{variable importance measures}. 

To make our work informative to the applied user of decision trees, we strive to make the least departure from practice and therefore focus specifically on Classification and Regression Tree (CART) \citep{breiman1984} methodology---by far the most popular for regression and classification problems. With this methodology, the tree construction depends on both the input and output data and is therefore \emph{data-dependent}. This aspect lends itself favorably to the empirical performance of CART, but poses unique mathematical challenges. As far as we know, the only other work that studies the bias of CART is \citep{scornet2015}, who used it to show asymptotic consistency of Breiman's random forests for additive regression models. One goal of the present paper is to extend this theory to other response surfaces.

Because individual decision trees are unstable and subject to large sampling variability\footnote{This is particularly the case for classification trees.}, we do not concern ourselves with a study of their variance. Indeed, such an endeavor is worthwhile only in the presence of some sort of variance reduction technique like pruning (i.e., removing portions of the tree in order to reduce its complexity) or ensemble averaging (e.g., bagging \citep{breiman1996bagging}, random forests, boosting), and even so, the bias remains the most challenging aspect of the analysis. This is, of course, not to diminish variance as an essential component of a theoretical investigation. For instance, in the sequel, we will apply our individual tree bias bounds to Breiman's random forests (which use ensembles of trees) in conjunction with existing results for its variance.

Let us now describe the statistical setting and framework that we will operate under for the rest of the paper. For clarity and ease of exposition, we focus specifically on \emph{regression trees}, where the target outcome is a continuous real value. Although, many of our results hold for binary classification trees as well.

%\citep{lin2006} show that, for nonadaptive forests, unless the number of observations $ K $ per terminal node grows with the sample size, one is stuck with suboptimal rates of the form $ K^{-1}(\log n)^{-(d-1)} $ which may be very slow if $ K $ is small.

%Relevant papers: \citep{torgo2001study, breiman1996some, wager2015, scornet2016asymptotics, meinshausen2006quantile, lin2006}.

%In the \texttt{randomForest} package in R, $ K $ is governed by \texttt{nodesize}, or the minimum size of terminal nodes. The default value is $ 1 $ for classification and $ 5 $ for regression. For gradient boosting with trees, the minimum size of a terminal node is denoted by \texttt{n.minobsinnode} in the \texttt{gbm} package.

We assume the learning (training) data is $ \calD_n = \{(\bX_1, Y_1), \dots, (\bX_n, Y_n) \} $, where $(\bX_i , Y_i ) $, $1\leq i \leq n $ are i.i.d. with common joint distribution $\mathbb{P}_{\bX, Y} $ and joint density $ p_{\bX, Y} $ with respect to Lebesgue measure (with marginal distribution and density, $ \mathbb{P}_{\bX} $ and  $p_{\bX}$, defined analogously). Here, $\bX_i \in [0, 1]^d $ is the input (feature, covariate, or predictor vector) and $Y_i \in \mathbb{R} $ is a continuous outcome (response or output variable). A generic pair of variables will be denoted as $ (\bX, Y) $. A generic coordinate of $ \bX $ will be denoted by $ X $, unless there is a need to highlight the dependence on the coordinate index, say $ X_j $, where $ j = 1, 2, \dots, d $. We will use the terms feature, predictor, or input variable to refer to $ X $ interchangeably.
%For convenience, we will often simply refer to $X$ as a variable. 
The statistical model is $Y_i = f(\bX_i) + \varepsilon_i $, for  $ i = 1, \dots, n $, where $ f(\bx) = \expect{Y \mid \bX = \bx} $ is an unknown regression function and $ \{ \varepsilon_i \}_{1 \leq i\leq n} $ are i.i.d. errors. The conditional average of $ Y $ given $ \bX = \bx $ is optimal for prediction if one uses squared error loss $ L(\bX, Y, \widetilde f) = (Y-\widetilde f(\bX))^2 $ since it minimizes the conditional risk $ \widetilde f \mapsto \expect{(Y-\widetilde f(\bX))^2 \mid \bX = \bx} = \int (y-\widetilde f(\bx))^2 \mathbb{P}_{Y \mid \bX = \bx}(dy) $. The bias of a prediction rule $ \widehat Y(\bx) = \widehat Y(\bx; \calD_n) $ at a point $ \bX = \bx $ is henceforth defined to be $ \expect{Y \mid \bX = \bx} - \Expect_{\calD_n}[\widehat Y(\bx)] $.

One key strength of decision trees is that they can exploit, if present, low local dimensionality of the response surface. This is particularly useful since many real-world input/output systems admit or are approximated well by local sparse representations of simple model forms, e.g., wavelets and neural networks. That is, even though the input/output relationship is determined by a large number of variables overall, the dependence may be locally characterized by only a small subset of variables. Such local adaptivity is made possible by the recursive partitioning of the input space, in which optimal splits are increasingly affected by local qualities of the data as the tree is grown. Hence, variables that locally have more influence on determining the response are much more likely to be included as candidates for further splitting. In essence, decision trees have a built-in local variable subset selection mechanism, which allows them to overcome many of the undesirable consequences (e.g., overfitting and large sample requirements) of high-dimensional modeling.

In this paper, we shall work with a more simplistic, yet still exemplary model and assume that the conditional mean response $ f(\bx)=\expect{Y \mid \bX=\bx} $ depends only on a small, unknown subset $ \calS $ of the $ d $ features. In other words, $ f $ is almost surely equal to its restriction $ f{\mid}_{[0, 1]^S} $ to the subspace $ [0, 1]^S $ of its ``strong'' variables $ \bx_{\calS}=(x_j : j \in \calS) $, where $ S = \#\calS \ll d $. Conversely, the output of $ f $ does not dependent on ``weak'' variables that belong to $ \calS^c $. Of course, the set $ \calS $ is not known a priori and must be learned from the data. 

\section{Organization}

This paper is organized according to the following schema. In \prettyref{sec:notation}, we establish some basic notation and definitions that we will use throughout the paper. \prettyref{sec:prelim} reviews some of the terminology and quantities associated with growing regression trees. In \prettyref{sec:mdi}, we review two important data-analytic quantities associated with tree-ensembles for visualization and interpretability. A summary of the main results is given in \prettyref{sec:results}, including some accompanying examples and simulations studies. \prettyref{sec:ass} discusses the main assumptions on the regression function that we use to obtain our bounds. In \prettyref{sec:rf}, we apply our bounds for individual decision trees from \prettyref{sec:results} to show consistency of Breiman's random forests. \prettyref{sec:finite} contains some finite sample results. Finally, in \prettyref{sec:classification}, we briefly discuss how our results can be extended to binary classification. Proofs of the main statements from \prettyref{sec:results} are given in \prettyref{sec:proofs} and \prettyref{app:proofs} contains proofs of some lemmas and examples from the body of the paper.

\section{Notation and definitions} \label{sec:notation}
For $ \calS \subset \{1, 2, \dots, d\} $, let $ \bx_{\calS} = (x_j : j\in \calS) $. With a slight abuse of notation, we write $ \bx_{\setminus j} = (x_{j'}: j' \neq j) $ instead of $ \bx_{\{1, 2, \dots, d\}\setminus \{j\}} $ for brevity. If $ A $ is a subset of $ \mathbb{R}^d $, we let $ A_{\calS} = \{ \bx_{\calS} : \bx \in A \} $ and $ A_{\setminus j} = \{ \bx_{\setminus j} : \bx \in A \} $. For two subsets $ A, B \subset \mathbb{R} $, $ \text{dist}(A, B) = \inf_{a\in A,\; b\in B}|a-b| $. For $ \calS \subset \{1, 2, \dots, d\} $ and $ A \subset \mathbb{R}^d $, we define $ \text{diam}_{\calS}(A) = \sup_{\bx, \; \bx' \in A}\|\bx_{\calS}-\bx_{\calS}'\| $, where $ \|\bx_{\calS}\|^2 = \sum_{j\in\calS}x^2_j $. 
%We let $ \bx_{\calS} = (x_j : j\in \calS) $, $ t_{\calS} = \{ \bx_{\calS} : a_j \leq x_j \leq b_j, \; j \in \calS \} $ and $ N_{\calS} = \sum_{i=1}^n \indc{\bX_{i\calS} \in t_{\calS}} $ (with analogous definitions for $ N_{L\calS} $ and $ N_{R\calS} $).

For a function $ g: \mathbb{R}^d \to \mathbb{R} $ and subset $ A \subset \mathbb{R}^d $, we define the oscillation of $ f $ on $ A $ by $ \omega(g; A) = \sup_{\bx, \;\bx' \in A}|g(\bx)-g(\bx')| $. The total variation of a function $ g: \mathbb{R} \to \mathbb{R} $ on an interval $ [a,b] $ is defined by $ \text{TV}(g; [a,b]) = \sup_{\mathscr{P}} \sum_{\ell=1}^L |g(x_{\ell})-g(x_{\ell-1})| $, where $ \mathscr{P} = \{ \calP = \{ x_0,x_1, \dots, x_L\} : \calP\; \text{is a partition of}\; [a,b] \} $. If $ g: \mathbb{R}^d \to \mathbb{R} $, we will write $ \frac{\partial^r}{\partial x^r_j}g(\bx) $ to denote the $r^{\Th} $ order partial derivative of $ g $ with respect to the $ j^{\Th} $ variable at the point $ \bx $. If $ g: \mathbb{R} \to \mathbb{R} $, we will write $ g'(x) $ to denote the first derivative of $ g $ at the point $ x $. The $ r^{\Th} $ order derivative of $ g $ at the point $ x $ is denoted by $ g^{(r)}(x) $. 
%Finally, we use the convention that $ 0/0 = 0 $.

%For mixed partial derivatives, we use multi-index notation. Thus, if $ \mathbf{r} = (r_1, \dots, r_d) $ and $ r = \|\mathbf{r}\|_1 $, then $ \partial^{\mathbf{r}}f(\bx) $ denotes the mixed partial derivative $ \frac{\partial^r}{\partial x^{r_1}_1\cdots \partial x^{r_d}_d}f(\bx) $.

\section{Preliminaries} \label{sec:prelim}

As mentioned earlier, regression trees are commonly constructed with Classification and Regression Tree (CART) \citep{breiman1984} methodology. The primary objective of CART is to find partitions of the input variables that produce minimal variance of the response values (i.e., minimal sum of squares error with respect to the average response values). Because of the computational infeasibility of choosing the best overall partition, CART trees are greedily grown with a procedure in which binary splits recursively partition the tree into near-homogeneous terminal nodes. That is, an effective binary split partitions the data from the parent tree node into two daughter nodes so that the resultant homogeneity of the daughter nodes, as measured through their \emph{impurity} (within node sum of squares error), is improved from the homogeneity of the parent node.
%Regression trees are constructed using a recursive partitioning algorithm. This algorithm builds a tree by recursively splitting the training sample into
%smaller subsets. 
%The algorithm has three key issues:
%\begin{enumerate}
%\item A rule to select a split test.
%\item A rule to determine when a tree node is terminal.
%\item A rule to assign a value to every terminal node.
%\end{enumerate}
Under the least squares error criterion, it can easily be shown that if one desires to have constant output in the terminal nodes of the tree, then the constant to use in each terminal node should be the average of the response values within the terminal node \cite[Proposition 8.10]{breiman1984}. Hence, these models produce a histogram estimate of the regression surface. As such, they are often referred to as piecewise constant regression models, since the tree output is constant on each terminal node.

Let us now describe the algorithm with additional precision. Consider splitting a regression tree $T$ at a node $\bt$. Let $s$ be a candidate split for a variable $X$ that splits $\bt$ into left and right daughter nodes $\bt_L$ and $\bt_R$ according to whether $ X \leq s $ or $ X > s $. These two nodes will be denoted by $\bt_L = \{\bX \in \bt: X \leq s\} $ and $\bt_R = \{\bX \in \bt: X > s\} $. As mentioned previously, a tree is grown by recursively reducing node impurity, which, for regression trees grown with CART, is determined by within node sample variance 
\begin{equation} \label{eq:impurity} \widehat\Delta(\bt) = \frac{1}{N(\bt) }\sum_{\bX_i \in \bt}(Y_i - \overline Y_{\bt})^2, \end{equation}
where $\overline Y_{\bt} = \frac{1}{N(\bt)}\sum_{\bX_i \in \bt}Y_i $ is the sample mean for $\bt$ and $N(\bt) =\# \{ \bX_i \in \bt \}  $ is the number of observations in $\bt$. Similarly, the within sample variance for a daughter node is $$ \widehat\Delta(\bt_L) = \frac{1}{N(\bt_L)}\sum_{\bX_i \in \bt_L}(Y_i - \overline Y_{\bt_L})^2, \quad \widehat\Delta(\bt_R) = \frac{1}{N(\bt_R)}\sum_{\bX_i \in \bt_R}(Y_i - \overline Y_{\bt_R})^2, $$ where $\overline Y_{\bt_L}$ is the sample mean for $\bt_L$ and $N(\bt_L)$ is the sample size of $\bt_L$ (similar definitions apply to $\bt_R$). The parent node $ \bt $ is split into two daughter nodes using the variable and split point producing the largest decrease in impurity (or impurity reduction). For a candidate split $s$ for $X$, this decrease in impurity equals \cite[Definition 8.13]{breiman1984} \begin{equation} \label{eq:data} \widehat\Delta(s ; \bt) = \widehat\Delta(\bt) - [\widehat P(\bt_L)\widehat\Delta(\bt_L) + \widehat P(\bt_R)\widehat\Delta(\bt_R)] \end{equation} where $ \widehat P(\bt_L) = N(\bt_L) /N(\bt) $ and $\widehat P(\bt_R) = N(\bt_R) /N(\bt) $ are the proportions of observations in $ \bt $ that are contained in $\bt_L$ and $\bt_R$, respectively. Note that maximizing $\widehat\Delta(s ; \bt)$ is also equivalent to minimizing \begin{equation} \label{eq:pop} \widehat P(\bt_L)\widehat\Delta(\bt_L) + \widehat P(\bt_R)\widehat\Delta(\bt_R), \end{equation} which means that CART seeks the split point that minimizes the weighted sample variance. Yet another way to view $ \widehat\Delta(s; \bt) $ is via its equivalent representation $ \widehat\Delta(\bt)\hat\rho^2 $, where $ \hat\rho $ is the empirical correlation between $ Y $ and the decision stump $ \widehat Y = \overline Y_{\bt_L}\indc{X \leq s} + \overline Y_{\bt_R}\indc{X > s} $ within $ \bt $. Hence, at each node, CART seeks the step function most correlated (in magnitude) with the response variable.

To reiterate, the tree $T$ is grown recursively by finding the split point $s$ that maximizes $\widehat\Delta(s ; \bt)$. The particular variable chosen is the one that gives the largest reduction in impurity over $ \bt $. We denote an optimized split point by $ \hat s $ (breaking ties arbitrarily) and the optimally split daughter nodes with $\hat s $ by $\hat \bt_L$ and $\hat \bt_R $, i.e., $ \hat \bt_L = \{\bX \in \bt : X \leq \hat s\} $ and $ \hat \bt_R = \{\bX \in \bt : X > \hat s\} $, respectively. The tree output at a terminal node $ \bt $ is $ \widehat Y = \overline Y_{\bt} $.

The use of \prettyref{eq:data} to evaluate each candidate split involves several passes over the training data with the consequent computational costs when handling problems with a large number of predictor variables and observations. This is particularly serious in the present setting of continuous variables---the major computational bottleneck of growing trees. Fortunately, the use of the least squares error criterion with averages in the terminal nodes permits further simplifications of the formulas described above. That is, using the sum of squares decomposition, $ \widehat\Delta(s ; \bt) $ can equivalently be expressed as \citep[Section 9.3]{breiman1984}
\begin{equation} \label{eq:rep1}
\widehat P(\bt_L)\widehat P(\bt_R)\bigg[\frac{1}{N(\bt_L)}\sum_{\bX_i \in \bt_L}Y_i  - \frac{1}{N(\bt_R)}\sum_{\bX_i \in \bt_R}Y_i \bigg]^2.
\end{equation}
%or
%\begin{equation} \label{eq:rep2}
%\frac{\widehat P(\bt_L)}{\widehat P(\bt_R)}\bigg[\frac{1}{N(\bt_L)}\sum_{\bX_i \in \bt_L}Y_i  - \frac{1}{N(\bt) }\sum_{\bX_i \in \bt}Y_i \bigg]^2.
%\end{equation}

This expression implies that one can find the best split for a continuous variable with just a single pass over the data, without the need to calculate multiple averages and sums of squared differences for these averages. In fact, not only does this representation have computational benefits---it will also be crucial in establishing our main results. It should be stressed that this alternative expression is unique to the least squares error criterion with averages in the terminal nodes of the tree.

A direct analysis of regression trees using the finite sample splitting criterion $ \widehat\Delta(\cdot; \bt) $ is challenging and obfuscates some of the inner and subtle mechanisms that makes tree-based learning with CART desirable. Instead, to remain true to the original CART procedure while still being able theoretically study its dynamics, we work under an asymptotic data setting for determining splits and therefore replace $ \widehat\Delta(s ; \bt) $ with its analog based on population parameters:
\begin{equation} \label{eq:infinite} \widehat\Delta(s ; \bt) \underset{n\rightarrow+\infty}{\rightarrow} \Delta(s ; \bt) \triangleq \Delta(\bt) - [P(\bt_L)\Delta(\bt_L) + P(\bt_R)\Delta(\bt_R)], \end{equation} where $ \Delta(\bt) $ is the conditional population variance $ \Delta(\bt) = \Var[Y \mid \bX \in \bt] $, $\Delta(\bt_L)$ and $\Delta(\bt_L)$ are the daughter conditional variances $$ \Delta(\bt_L) = \Var[Y \mid X \leq s, \; \bX \in \bt], \quad \Delta(\bt_R) = \Var[Y \mid X > s, \; \bX \in \bt], $$ and $P(\bt_L)$ and $P(\bt_R)$ are the conditional probabilities (probability content) of the daughter nodes arising from the split $ s $, i.e., $$ P(\bt_L ) = \prob{X \leq s \mid \bX \in \bt}, \quad P(\bt_R) = \prob{X > s \mid \bX \in \bt}. $$ 
These conditional probabilities can also be interpreted as the infinite sample proportion of observations in the parent node that are contained in the daughter nodes. Note that $ P(\bt_L) $ is also the distribution function of $ X \mid \bX \in \bt $. We will occasionally write $ P(s | \bt) $ instead of $ P(\bt_L) $ to highlight dependence on the split $ s $. The density function of $ X \mid \bX \in \bt $, that is, $ \frac{\partial}{\partial s}\prob{X \leq s \mid \bX \in \bt} $, will be denoted by $ p(\bt_L) $ or $ p(s| \bt) $. The infinite sample analog of \prettyref{eq:rep1} is 
%$$ P(\bt_L)\Delta(\bt_L) + P(\bt_R)\Delta(\bt_R), $$
\begin{equation} \label{eq:simple0}
P(\bt_L)P(\bt_R)\left[\expect{Y \mid \bX \in \bt, \; X \leq s} - \expect{Y \mid \bX \in \bt, \; X > s}\right]^2.
\end{equation}
%and
%$$
%\frac{P(\bt_L)}{P(\bt_R)}\left[\expect{Y \mid \bX \in \bt, \; X \leq s} - \expect{Y \mid \bX \in \bt}\right]^2,
%$$
%respectively.  
We let $\bt^*_L=\{\bX \in \bt: X \leq s^*\}$ and $\bt^*_R=\{\bX \in \bt: X > s^*\}$ denote the optimally split daughter nodes with optimal split $ s^* $, respectively. Note that any node $ \bt $ is a Cartesian product of intervals, which we call \emph{subnodes}. The subnode of variable $ X $ within node $ \bt $ is denoted by $ [a(\bt), b(\bt)] $, where $ a(\bt) < b(\bt) $.
%If we split along a generic coordinate $ X $ in a node $ \bt $, its parent endpoints are denoted as $ a $ and $ b $, where $ a < b $.

Under our framework, we optimize the infinite sample splitting criterion, namely, $ \Delta(s; \bt) $, instead of the empirical one \prettyref{eq:data}. Optimizing the splitting protocol $ \Delta(s ; \bt) $ can also be viewed as querying an oracle for the optimal (population) split values. A maximizer of $ \Delta(s; \bt) $ is denoted by $ s^* $, i.e., $ s^* \in \argmax_s \Delta(s ; \bt) $ (again, breaking ties arbitrarily). Despite this tweak, we stress that only the splits are determined from an infinite sample quantity; all other aspects of the regression tree (e.g., terminal node values, depth, variables selected for candidate splits) are determined from the learning sample. More specifically, the regression tree still outputs $ \widehat Y = \frac{1}{N(\bt)}\sum_{\bX_i \in \bt} Y_i $, except that $ \bt $ is determined by an infinite sample (population) objective. If the number of observations within $ \bt $ is large and $ \Delta(\cdot; \bt) $ has a unique global maximum, then we can expect $ \hat s  \approx s^*$ (via an empirical process argument) and hence our infinite sample setting is a good approximation to CART with empirical splits. Indeed, if $ s^* $ is unique,  \citep{banerjee2007confidence} and \citep[Section 3.4.2]{buhlmann2002analyzing} show \emph{cube root asymptotics} (i.e., $ n^{1/3}(\hat s - s^*) $ converges in distribution) of split points for multi-level decision trees using the CART sum of squares error criterion. While these results show that $ \hat s $ and $ s^* $ are close, they do not reveal what sort of partition of the input space is induced by a sequence of optimal splits---the goal of the present paper. Since this partition is constructed from the distribution of $ (\bX, Y) $, we will need to impose conditions on the regression function and input distribution so that the partition yields a predictor with small bias. It is our hope that the applied user of decision trees can benefit from knowing these limitations in the idealized setting of population-level splits.

Since the recursive partition obtained from $ \widehat\Delta(s ; \bt) $ governs the bias of the tree, our study of the partitions created from the infinite sample version $ \Delta(s ; \bt) $ is in much the same vein as the study of kernel or nearest neighbors regression with (oracle) infinite sample, optimal parameters (e.g., kernel bandwidth or number of nearest neighbors), obtained by minimizing the asymptotic mean integrated squared error (AMISE).

\subsection{A comment on notation} Up until now, we have assumed that the split occurs along a generic coordinate $ X $. However, in the multi-dimensional setting, there will be a need to specify the coordinate index. Therefore, we will write $ \Delta(j, s; \bt) $ to denote $ \Delta(s ; \bt) $ when the coordinate being split is $ X_j $. For the same reason, we write $ P_j(\bt_L) $ and $ P_j(s | \bt) $ in lieu of $ P(\bt_L) $ and $ P(s | \bt) $, respectively. Similar definitions will hold for the subnode $ [a_j(\bt), b_j(\bt)] $ when we want to emphasize the dependence on the variable index. We write $ j_{\bt} $ to denote the splitting variable chosen in node $ \bt $, i.e., $ j_{\bt} \in \argmax_{j=1,2,\dots, d}\Delta(j, s^*; \bt) $, breaking ties arbitrarily. Finally, note that an optimal split $ s^* $ for a node $ \bt $ depends on $ \bt $ and therefore should technically be written as $ s^*_{\bt} $, however, we suppress this dependence for brevity and assume that it holds implicitly. We also use similar notation for the finite sample versions of these quantities.
As a general rule, if the index $ j $ is omitted on any quantity, it should be understood that we are considering a generic variable $ X $.

We now survey two popular data-analytic quantities that are also intimately connected to the forthcoming results in \prettyref{sec:results}.

\section{Data-analytic quantities for interpretation and visualization} \label{sec:mdi}
Pertinent to many tasks in data analysis is interpretability of the representation of the model. For black-box models such as random forests, this is usually performed through graphical renderings of the prediction surface as a function of one or two predictor variables and assessing the role each variable plays in determining the final output. Here we review two of the most popular quantities associated with decision tree ensembles for this purpose.

\subsection{Partial dependence functions} One useful device for visualizing the influence of a particular variable on the output is the so-called \emph{partial dependence function}.\footnote{Visualizing the effect of more than two predictor variables has limited descriptive value.} By looking simultaneously at the \emph{partial dependence plots}---i.e., a trellis of plots of the partial dependence functions for each variable---one can obtain a visual description of a multivariable predictor. In the same spirit, for assessing how each variable affects the output conditional on a node, we define the \emph{conditional partial dependence function} given node $ \bt $ via
\begin{equation} \label{eq:partialdef}
\overline F_j(x_j ; \bt) = \expect{Y \mid \bX \in \bt, \; X_j = x_j} = \int f(x_j, \bx_{\setminus j})\mathbb{P}_{\bX_{\setminus j}\mid \bX \in \bt, \; X_j = x_j}(d\bx_{\setminus j}),
\end{equation}
where $ \mathbb{P}_{\bX_{\setminus j} \mid \bX \in \bt, \; X_j = x_j}(d\bx_{\setminus j}) $
%= \frac{\mathbb{P}_{(x_j,\bX_{\setminus j})\mid \bX \in \bt}(d\bx_{\setminus j})}{\frac{\partial}{\partial x_j}\prob{X \leq x_j \mid \bX \in \bt}} $ 
is the conditional probability measure with conditional density $ \bx_{\setminus j} \mapsto  p_{\bX}(x_j, \bx_{\setminus j})/\int_{\bt_{\setminus j}}p_{\bX}(x_j, \bx'_{\setminus j})d\bx'_{\setminus j} $.
Note that $ \overline F_j(x_j ; \bt) $ is the best least squares approximation to $ f(\bX) $ in $ \bt $ as a function of $ X_j $ alone, i.e.,
$ \overline F_j(x_j ; \bt) = \argmin_F \expect{(Y-F(X_j))^2\mid \bX \in \bt, \; X_j = x_j} $. For brevity, we write $ \overline F'_j(x_j ; \bt) $ to denote the derivative $ \frac{\partial}{\partial x_j}\overline F_j(x_j ; \bt) $, provided it exists.
We also define the \emph{mean-centered conditional partial dependence function} as
\begin{align*}
\overline G_j(x_j ; \bt) & = \overline F_j(x_j ; \bt) - \int \overline F_j(x_j ; \bt)\mathbb{P}_{X_j \mid \bX \in \bt}(dx_j) \\ & = \expect{Y \mid \bX \in \bt, \; X_j = x_j} - \expect{Y \mid \bX \in \bt};
\end{align*}
that is, $ \overline G_j(x_j ; \bt) $ is $ \overline F_j(x_j ; \bt) $ minus its mean value over the node.

Strictly speaking, $ \overline F_j(\cdot; \bt) $ is \emph{not} the same as the partial dependence function in the sense of \citep[Section 8.2]{friedman2001greedy} or \citep[Section 10.13.2]{hastie2009elements}, where, instead of ignoring the effects of $ \bX_{\setminus j} $, one looks at the dependence of $ f $ on $ X_j $ in $ \bt $ \emph{after} averaging with respect to the marginal distribution of the other covariates $ \bX_{\setminus j} $, i.e., $ \overline F_j(x_j ; \bt) =  \int f(x_j, \bx_{\setminus j})\mathbb{P}_{\bX_{\setminus j}\mid \bX \in \bt}(d\bx_{\setminus j}) $, where $ \mathbb{P}_{\bX_{\setminus j}\mid \bX \in \bt}(d\bx_{\setminus j}) $ is the conditional probability measure with conditional density $ \bx_{\setminus j} \mapsto p_{\bX_{\setminus j}}(\bx_{\setminus j})/\int_{\bt_{\setminus j}}p_{\bX_{\setminus j}}(\bx'_{\setminus j})d\bx'_{\setminus j}  $.
%, where $ \bt_{\setminus j} $ is node $ \bt $ without the $ j^{\Th} $ subnode.
%$ \Expect_{\bX_{\setminus j}\mid \bX_{\setminus j} \in \bt_{\setminus j}}[Y \mid \bX_{\setminus j} \in \bt_{\setminus j}] $. 
However, when $ X_j $ and $ \bX_{\setminus j} $ are independent (for example, as in the uniform case), the two definitions coincide with each other.

Of course, in practice, one would integrate an estimate $ \widehat f $ of $ f $ against the empirical distribution, so that, for example,
\begin{equation} \label{eq:partialdata}
\overline F_j(x_j ; \bt) = \frac{1}{N(\bt) }\sum_{\bX_i \in \bt} \widehat f(x_j, \bX_{i\setminus j}).
\end{equation}
When the tree is fully grown so that each terminal node contains only a single observation, i.e., $ N(\bt) = 1 $ for all $ \bt $, \prettyref{eq:partialdata} is the Individual Conditional Expectation (ICE) \citep{goldstein2015peeking}.

Let us finally mention that a desirable property of tree-based models (i.e., if $ \widehat f $ is a tree-structured predictor) is that the partial dependence function \prettyref{eq:partialdata} can be quickly computed from the tree itself, without passing over the data each time it is to be evaluated. This fact has implications for the rapid production of a trellis of plots for each variable.

\subsection{Variable importance measures and variable influence ranking}
%Decision trees with binary splits are popularly constructed using Classification and Regression Trees (CART) methodology. For regression models, at each node of the tree, the data is divided into two daughter nodes according to a split point that maximizes the decrease in variance (impurity) along a particular variable. 
Another attractive feature of tree-based ensembles is that one can compute, essentially for free, various measures of variable importance (or influence) using the optimal split points and their corresponding impurities. These measures are used for sensitivity analysis and to rank the influence of each variable in determining the output, which in turn, can be used to identify and select the most relevant predictors for further investigation (such as plotting their partial dependence functions). For random forests, one canonical and widely used measure is the Mean Decrease in Impurity (MDI) \citep[Section 8.1]{friedman2001greedy}, \citep[Section 5.3.4]{breiman1984}, \citep[Sections 10.13.1 \& 15.3.2]{hastie2009elements}, defined as the weighted sum of largest impurity reductions (i.e., either Gini impurity for classification or variance impurity for regression) over all non-terminal nodes in the tree, averaged over all trees in the forest, i.e.,
\begin{equation} \label{eq:empMDI}
 \widehat{\text{MDI}}(X_j) = \frac{1}{\#T}\sum_T\widehat{\text{MDI}}(X_j; T), 
\end{equation}
where
$$
\widehat{\text{MDI}}(X_j; T) = \sum_{\substack{\bt' \\ j_{\bt'}=j}}\widehat P(\bt')\widehat\Delta(j, \hat s ; \bt'),
$$
and the sum extends over all non-terminal (internal) nodes $ \bt' $ such that $j_{\bt'}=j$ and $\widehat P(\bt') = N(\bt') /n $ is the proportion of observations that land in node $ \bt' $.\footnote{Another commonly used and possibly more accurate measure is Mean Decrease in Accuracy (MDA), defined as the average difference in out-of-bag error before and after randomly permuting the values of $ X_j $ in out-of-bag samples over all trees. However, as with all permutation based methods, computational issues are present. Compare MDA with MDI, which can be computed as each tree is grown with no additional cost. Both MDI and MDA are calculated in the R package \texttt{randomForest} with \texttt{randomForest(..., importance = TRUE)} and in the Python library \texttt{scikit-learn} with attribute
 \texttt{feature\char`_ importances\char`_}}
 
The next definition of variable importance generalizes and localizes MDI. Instead of weighting each $ \widehat\Delta(j, \hat s; \bt') $ by the proportion of observations in $ \bt' $, more flexible weights are allowed. Moreover, rather than summing over all non-terminal nodes, we condition on a particular terminal node so that only nodes that are ancestors of the terminal node are considered.

\begin{definition}[Conditional variable importance]
Let $ \bt $ be a terminal node and let $ j_{\bt'} $ denote the index of the variable selected to split along at an ancestor node $ \bt' $ of $ \bt $. The conditional variable importance of $ X_j $ given $ \bt $ is defined by
%$$ \text{MDI}(X_j; \bt) = \sum_{\bt' \supset \bt}4P(\bt'^*_L)P(\bt'^*_R)\indc{j_{\bt'}=j} = \sum_{\bt' \supset \bt}\frac{\Delta(s^*_j; \bt')}{|\overline G(s^*_j; \bt')|^2+\Delta(s^*_j; \bt')}\indc{j_{\bt'}=j}, $$
\begin{equation}
\widehat{\text{MDI}}(X_j; \bt) = \sum_{\substack{\bt' \supset \bt \\ j_{\bt'}=j}}\widehat w(j, \hat s; \bt')\widehat\Delta(j, \hat s ; \bt'),
\end{equation}
where the sum extends over all ancestor nodes $ \bt' $ of $ \bt $ such that $j_{\bt'}=j$ and the (possibly data-dependent) weights $ \{\widehat w(j, \hat s; \bt')\} $ are nonnegative. That is, $ \widehat{\text{MDI}}(X_j; \bt) $ is a weighted sum of largest impurity decreases $ \widehat\Delta(j, \hat s ; \bt') $ among all ancestor nodes $ \bt' $ of terminal node $ \bt $ such that $ X_j $ was selected for a split. The infinite sample version of $ \widehat{\text{MDI}}(X_j; \bt) $, denoted by $ \text{MDI}(X_j; \bt) $, is defined by
\begin{equation} \label{eq:imp}
\text{MDI}(X_j; \bt) = \sum_{\substack{\bt' \supset \bt \\ j_{\bt'}=j}}w(j, s^*; \bt')\Delta(j, s^* ; \bt'),
\end{equation}
where the weights $ \{w(j, s^*; \bt')\} $ are nonnegative.
\end{definition}

Note that $ \widehat{\text{MDI}}(X_j) $ is a \emph{global} measure of importance. Thus, $ \widehat{\text{MDI}}(X_j) $ is relevant for answering questions such as ``what variables are most important in predicting blood pressure?'' and $ \widehat{\text{MDI}}(X_j; \bt) $ is useful for answering questions like ``for subjects ages 65 and up, what variables are most important in predicting blood pressure?''. A few other variants of local or case-wise (permutation) importance have been proposed to cope with the bias of variable importance measures towards correlated predictor variables. See, for example, \citep{strobl2008conditional}.

\section{Summary of results} \label{sec:results}
This section contains our main results. Due to space constraints, proofs of most of the statements are deferred until \prettyref{sec:proofs} and \prettyref{app:proofs}.

\subsection{Distributional assumptions}

Before we continue, let us first state the main distributional assumptions on the predictor variables. Let us remark, however, that many of the subsequent statements and results (e.g., \prettyref{sec:rf}, \prettyref{sec:finite}, and \prettyref{sec:proofs}) hold without them.

\begin{assumption} \label{ass:density}
For each variable $ X_j $ and node $ \bt $, the distribution function $ \prob{X_j \leq x_j \mid \bX \in \bt} $ is strictly increasing.
\end{assumption}

\prettyref{ass:density} is quite mild and holds if the joint density $ p_{\bX} $ never vanishes.\footnote{This does \emph{not} necessarily mean that the joint density is uniformly bounded below by a positive constant.}

The next assumption is needed for \prettyref{thm:main}, our main result.

\begin{assumption} \label{ass:indep}
For each variable $ X_j $, node $ \bt $, and split $ s $,
\begin{align}
& \prob{X_j \leq s \mid a_j(\bt) \leq X_j \leq b_j(\bt)} \geq \eta\prob{X_j \leq s \mid \bX \in \bt}\; \text{and} \\
&  \prob{X_j > s \mid a_j(\bt) \leq X_j \leq b_j(\bt)} \geq \eta\prob{X_j > s \mid \bX \in \bt},
\end{align}
for some universal constant $ \eta \in (0, 1] $.
\end{assumption}
\prettyref{ass:indep} is more restrictive than \prettyref{ass:density}, yet it still allows for some degree of dependency structures among the predictor variables. For example, it holds if the joint density function of the features is uniformly bounded above and below by constant multiples of the product of its marginal densities, i.e.,
$$
c_2p_{X_1}(x_1)\cdots p_{X_d}(x_d) \leq p_{\bX}(\bx) \leq c_1p_{X_1}(x_1)\cdots p_{X_d}(x_d),
$$
for all $ \bx \in [0, 1]^d $, where $ c_1 $ and $ c_2 $ are positive constants. In this case, $ \eta $ can be taken to be $ c_2/c_1 $.

\subsection{Probability content of terminal subnodes}

The next theorem, our main result, gives a clean, interpretable bound on the $ \mathbb{P}_{\bX} $-probability of a terminal subnode in terms of the variable importance measure, which in turn, controls the bias of the tree. Specifically, it shows that conditional terminal subnode probability content is (exponentially) small in the importance measure attributed to the splitting variable. This bound also corroborates with empirical evidence showing that, although decision trees are highly unstable, their bias tends to be small. For brevity, we furnish the proof in \prettyref{sec:proofs}.
\begin{theorem} \label{thm:main}
Suppose \prettyref{ass:density} and \prettyref{ass:indep} hold and $ [a_j(\bt), b_j(\bt)] $ is a subnode along the $ j^{\Th} $ direction for a terminal node $ \bt = \prod_{j=1}^d [a_j(\bt), b_j(\bt)] $. Then,
\begin{equation} \label{eq:nodesize}
\Prob_{\bX}[a_j(\bt) \leq X_j \leq b_j(\bt)] \leq \exp\left\{-\frac{\eta }{4}\text{MDI}(X_j; \bt)\right\},
\end{equation}
where $ \text{MDI}(X_j; \bt) = \sum_{\substack{\bt' \supset \bt \\ j_{\bt'}=j}}w(j, s^*; \bt')\Delta(j, s^* ; \bt') $ is the conditional variable importance \prettyref{eq:imp} with weights given by 
\begin{equation} \label{eq:w1}
w(j, s^*; \bt') = [|\overline G_j(s^*; \bt')|^2 + \Delta(j, s^* ; \bt')]^{-1}.
\end{equation}
Furthermore, if the first-order partial derivatives of the regression function and joint density function of $ \bX $ exist and are continuous, then
\begin{equation} \label{eq:w2}
w(j, s^*; \bt') \geq \bigg(\frac{2p_j(s^*| \bt')}{|\overline F'_j(s^*; \bt')|\Delta(j, s^* ; \bt')}\bigg)^{2/3}. 
\end{equation}
%where
%$$
%\text{MDI}(X_j; \bt) = \sum_{\substack{\bt' \supset \bt \\ j_{\bt'}=j}}\Bigg(4\bigg|\frac{p_j(s^*; \bt')}{\overline F'_j(s^*; \bt')}\bigg|^2\Delta(j, s^* ; \bt')\Bigg)^{1/3},
%$$
%and the sum extends over all ancestor nodes $ \bt' $ of $ \bt $ such that $ j_{\bt'}=j $.
\end{theorem}

The main technical tools for proving \prettyref{thm:main}---\prettyref{thm:mainprob} and \prettyref{thm:second}---characterize $ \prob{X_j \leq s^* \mid \bX \in \bt'} $ in terms of the partial dependence function $ \overline F_j(s^*; \bt') $ and reduction in impurity $ \Delta(j, s^* ; \bt') $ at an optimal split. The form of the weights \prettyref{eq:w1} are derived from the first-order conditions (i.e., $\frac{\partial}{\partial s}\Delta(j, s; \bt')\mid_{s=s^*} = 0 $) of $ s^* $ as a global maximizer of $ \Delta(j, \cdot; \bt') $, whereas the lower bound \prettyref{eq:w2} incorporates \emph{both} first- and second-order optimality conditions (i.e., $\frac{\partial}{\partial s}\Delta(j, s; \bt')\mid_{s=s^*} = 0 $ and $\frac{\partial^2}{\partial s^2}\Delta(j, s; \bt')\mid_{s=s^*} < 0 $).

\begin{remark}
If the regression function is linear and the input is uniformly distributed, then $ \text{MDI}(X_j; \bt) $ with weights \prettyref{eq:w1} equals $ K_j(\bt) $, the number of times $X_j$ was selected among all ancestor nodes of terminal node $ \bt $.\footnote{Importance measures based on feature selection frequencies are implemented in standard R and Python packages for XGBoost \citep{chen2016xgboost}.}
\end{remark}

%In addition to the representation \prettyref{eq:imp}, we will see in later sections that $ \text{MDI}(X_j; \bt) $ has the lower bound
%\begin{equation}
%\text{MDI}(X_j; \bt) \geq \sum_{\substack{\bt' \supset \bt \\ j_{\bt'}=j}}4\bigg|\frac{p_j(s^*; \bt')}{\overline F'_j(s^*; \bt')}\bigg|^2\Delta(j, s^* ; \bt'),
%\end{equation}
%which is also a weighted sum (with weights that depend on the conditional densities and derivative of the partial dependence function) of the largest impurity decreases.

\begin{remark} \label{rmk:bias}
\prettyref{thm:main} implies that $ \text{diam}_{\calS}(\bt) $ converges to zero in $ \Prob_{\bX} $-probability when $ \text{MDI}(X_j; \bt) \rightarrow +\infty $ for each $ j \in \calS $. According to classic theory for partitioning-based prediction rules, shrinking terminal node diameters is a necessary condition for asymptotic consistency \citep{stone1977consistent}.
\end{remark}

In light of \prettyref{eq:nodesize}, it is tantalizing to interpret $ \text{MDI}(X_j; \bt) $ (and therefore $ \widehat{\text{MDI}}(X_j; T) $) as truly a measure of variable importance,
% attributed to the splitting variable.
%(albeit, in an asymptotic data setting).
%, in the same spirit as $ \widehat{\text{MDI}}(X_j; T) $---that is, a weighted sum of largest impurity decreases. %The terminal subnode size bound \prettyref{eq:nodesize} incorporates both the best improvement in impurity and the partial dependence function.
%function into a single quantity.
% importance measure attributed to the splitting variable. 
%In other words, \prettyref{eq:nodesize} may theoretically motivate MDI as a measure of variable influence in the sense that
since it governs the bias of the regression tree along a specific direction, i.e., the conditional probability content of a terminal subnode for a variable is small when the MDI for that variable is large.\footnote{The original motivation for MDI was based solely on simple heuristic arguments and so one is left guessing as to why it seems to work. Overall, little is known about the theoretical properties of variable importance measures; see \citep{kazemitabar2017variable, louppe2013understanding, li2019debiased} for recent attempts in this direction.}

As a consequence of \prettyref{thm:main}, we immediately see two important properties of the regression tree:
\begin{itemize}
\item Terminal subnodes with smaller $\mathbb{P}_{\bX} $-probability are along more ``important'' directions as delineated by MDI, and thus adapt to signal strength in each direction. That is, the tree needs to be split more often in order to create finer partitions along directions that are more relevant to the output.
%\item More ``important'' variables are associated with terminal subnodes having smaller $\mathbb{P}_{\bX} $-probability.
%\item The more ``important'' a variable is, the smaller the $ \mathbb{P}_{\bX} $-probability of the terminal subnode along that direction.
\item The adaptation to the importance of the variable is \emph{local} and depends on a particular terminal node of the tree. That is, each direction and location of the features requires a different level of granularity in the tree in order to detect and adapt to local changes in the regression surface. For example, regions of the input space where the response is ``flat'' do not need to be split as often and therefore a crude partition will suffice. On the other hand, complex functional dependencies may require a higher degree of fineness in the final model.
\end{itemize}
These observations are consistent with \citep[Section 4]{lin2006}, in that terminal nodes are on average narrower in directions with strong signals than in directions with weak signals. This is a very appealing property from a statistical perspective---trees with terminal subnodes that have large (resp. small) probability content in directions with weak (resp. strong) signals are less likely to overfit (resp. underfit) the data. 

\begin{remark}
An analogy can be drawn between the terminal subnode probability content $ \Prob_{\bX}[X_j \in [a_j, b_j]] $ and the bandwidth size in kernel regression, both of which control the bias of the predictor. Bandwidth selection for kernel regression \citep{yang1999multivariate} is often performed using a plug-in estimator of the best bandwidth that minimizes the AMISE. There, just like with \prettyref{eq:w2}, the best theoretical bandwidth also depends on the density function of the features and the smoothness of the regression function (albeit, via curvature---$ \int \frac{\partial^2}{\partial x^2_j}f(\bx)\frac{\partial^2}{\partial x^2_{j'}}f(\bx)\mathbb{P}_{\bX}(d\bx) $) \citep[Theorem 1]{yang1999multivariate} and \citep[Section 5.8]{wand1994kernel}. 
\end{remark}
%For example, one-dimensional kernel regression estimators have optimal bandwidth $ h^* $ of the form $ h^* = C(\frac{1}{n(\int (f''(x_1))^2\mathbb{P}_{X_1}(dx_1)})^{1/5} $, for some constant $ C $ that depends on the choice of kernel.

\begin{remark}
As will be discussed in \prettyref{sec:classification}, the statement of \prettyref{thm:main} also holds verbatim for binary classification trees with appropriate modifications for alternate (i.e., Gini) splitting rules. An illustrative example involving logistic regression will also be provided.
\end{remark}

\subsection{Empirical study}
In \prettyref{fig:data}, we showcase the aforementioned adaptive properties of trees on four representative datasets from the UC Irvine Machine Learning Repository, namely, the \emph{Airfoil Self-Noise} (\prettyref{fig:airfoil}) and \emph{Concrete Compressive Strength} (\prettyref{fig:concrete}) datasets for regression and the  \emph{Blood Transfusion Service Center} (\prettyref{fig:blood}), and \emph{HTRU2} (\prettyref{fig:star}) datasets for binary classification. 
We first standardized the input data so that each variable belongs to $ [0, 1] $, i.e., $ X' = (X-X_{(1)})/(X_{(n)}-X_{(1)}) $. For each dataset, we generated $ 1000 $ trees from bootstrap samples of the data. Each tree was generated by \texttt{rpart} in R with default settings, except for the complexity parameter $ \texttt{cp} $, which was set to $ 0.001 $. The black bars represent the median terminal subnode length $ \ell_j(\bt) = b_j(\bt)-a_j(\bt) $ for each tree, averaged over all $ 1000 $ trees. The white bars represent $ \widehat{\text{MDI}} $ \prettyref{eq:empMDI} for each variable. Both quantities are scaled so that the largest among them is $ 100 $ and the variables are ordered according to increasing $ \widehat{\text{MDI}} $. In agreement with \prettyref{thm:main}, the barplots reveal the inverse relationship between $ \widehat{\text{MDI}}(X_j) $ and the terminal subnode lengths $ \ell_j $.

\begin{figure} [t] 
\centering
\begin{subfigure}[t]{0.45\textwidth}
  \centering
  \includegraphics[width=1\linewidth]{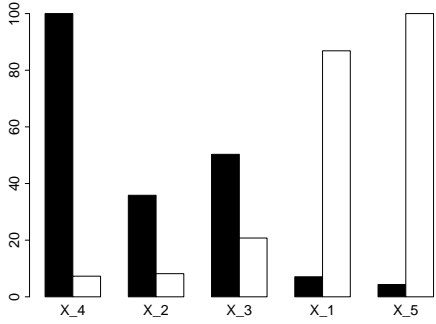}
  \caption{\emph{Airfoil Self-Noise}: $ n = 1503 $, $ d = 5 $.}
  \label{fig:airfoil}
\end{subfigure}
\hspace{1cm}
\begin{subfigure}[t]{0.45\textwidth}
  \centering
  \includegraphics[width=1\linewidth]{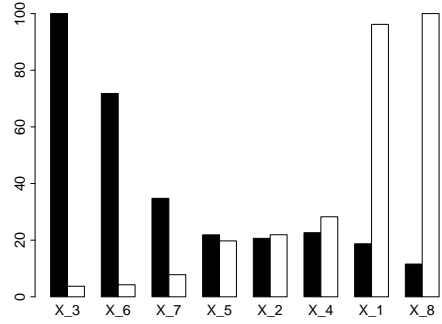}
  \caption{\emph{Concrete Compressive Strength}: $ n = 1030 $, $ d = 8 $.}
  \label{fig:concrete}
\end{subfigure}
\hspace{1cm}
\begin{subfigure}[t]{0.45\textwidth}
  \centering
  \includegraphics[width=1\linewidth]{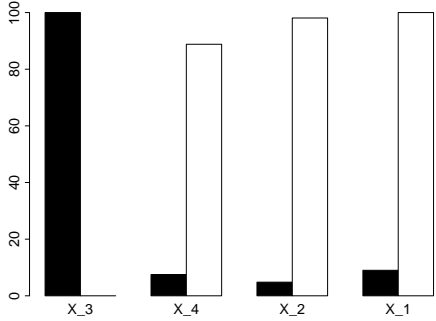}
  \caption{\emph{Blood Transfusion Service Center}: $ n = 748 $, $ d = 4 $.}
  \label{fig:blood}
\end{subfigure}
\hspace{1cm}
\begin{subfigure}[t]{0.45\textwidth}
  \centering
  \includegraphics[width=1\linewidth]{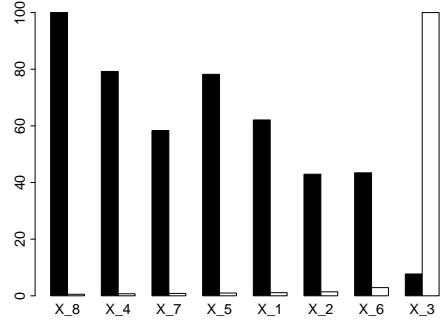}
  \caption{\emph{HTRU2}: $ n = 17898 $, $ d = 8 $.}
  \label{fig:star}
\end{subfigure}
\caption{Median subnode length of each tree, averaged over $ 1000 $ bootstrapped trees (black bars) and $\widehat{\text{MDI}}(X_j) $ \prettyref{eq:empMDI} (white bars). Both quantities are scaled so that the largest among them is $ 100 $ and the variables are ordered according to increasing $\widehat{\text{MDI}}(X_j) $.}
\label{fig:data}
\end{figure}

\subsection{Lower bounds on MDI} \label{sec:mdilower}

In light of \prettyref{thm:main} and \prettyref{rmk:bias}, it is natural to ask when $ \text{MDI}(X_j; \bt) $ diverges. We now provide some answers. First, let us mention that studying $ \text{MDI}(X_j; \bt) $ directly is hopeless since it is nearly impossible to give a closed form expression for each $ \Delta(j, s^* ; \bt') $. Nevertheless, $ \text{MDI}(X_j; \bt) $ can still be lower bounded by giving a lower bound on each weight $ w(j, s^*; \bt) $ and reduction in impurity $ \Delta(j,s^*; \bt) $. 

By definition of $ s^* $, one can lower bound each $ \Delta(j, s^* ; \bt') $ by $ \Delta(j, s'; \bt') $ for \emph{any} $ s' \in [a_j(\bt'),b_j(\bt')] $. Effectively, this means that to lower bound $ \text{MDI}(X_j; \bt) $, one can replace $ \Delta(j, s^* ; \bt') $ by $ \Delta(j, s'; \bt') $ for any choice $ s' $ in $ [a_j(\bt'), b_j(\bt')] $ or (per a Bayesian perspective) by the integrated decrease in impurity $ \int_{a_j(\bt')}^{b_j(\bt')}\Delta(j, s; \bt')\Pi(ds) $ with respect to a prior $ \Pi $ on the splits. (It is often convenient to choose $ s' $ to be the median of $ X_j \mid \bX \in \bt' $.) This observation is crucial to the forthcoming analysis, since it reduces the burden of finding $ s^* $ exactly to finding a suitable choice of $ s' $ or prior $ \Pi $ for which $ \Delta(j, \cdot; \bt') $ is tractable to analyze.

To lower bound the weights $ w(j, s^*; \bt') $ in \prettyref{thm:main}, one is confronted with obtaining a useful upper bound on either $ |\overline G_j(s^*; \bt')| $ for \prettyref{eq:w1} or $ |\overline F'_j(s^*; \bt')| $ for \prettyref{eq:w2}. We will see that it typically suffices to bound either by their supremum norm over splits in the parent subnode, and so no explicit knowledge of $ s^* $ is required. For the lower bound \prettyref{eq:w2}, one additionally needs to lower bound the conditional density of $ X_j \mid \bX \in \bt' $ at $ s^* $, or $ p_j(s^*| \bt') $. This too is a simple task if the joint density of $ \bX $ is uniformly bounded away from zero by a positive constant, in which case $ p_j(s^*| \bt') \geq \inf_s p_j(s| \bt') > 0 $. For example, if $ \bX $ is uniformly distributed, then $ p_j(s^*| \bt') = (b_j(\bt')-a_j(\bt'))^{-1} $.

Using these observations, we show in \prettyref{thm:lambdalower} that $ \text{MDI}(X_j; \bt) $ can be lower bounded by a positive constant multiple of the selection frequency of $ X_j $ in the tree. For brevity, we defer its proof until \prettyref{sec:proofs}. Before we state \prettyref{thm:lambdalower}, we first introduce some concepts. Central to the paper is a quantity which we call the \emph{node balancedness}. It measures the infinite sample proportion of data in the parent node that is contained in either daughter node from an optimal split.
\begin{definition}[Node balancedness] \label{def:nodebalance}
The balancedness of a node $ \bt $ is defined by
$$
\lambda_j(\bt) = 4P_j(\bt^*_L)P_j(\bt^*_R) = 1- |P_j(\bt^*_L)-P_j(\bt^*_R)|^2.
$$
\end{definition}
Another way of thinking about node balancedness is the following. Suppose we randomly generate a new $ \bX $ from $ \mathbb{P}_{\bX \mid \bX \in \bt} $ and classify $ Z = +1 $ if $ X_j \leq s^* $ or $ Z = -1 $ if $ X_j > s^* $. Then $ \lambda_j(\bt) $ is simply the variance of $ Z $.

The node balancedness is always one (perfectly balanced) when the split is performed at the median of the conditional distribution $ X_j \mid \bX \in \bt $. This particular situation occurs in the special case that the regression surface is linear and the input distribution is uniform.
%The quantity $ \lambda_j $ was defined for splits along a generic direction. When we want to stress a particular direction, say the $ j^{\Th} $, we write $ \lambda_j $.
In general, the quantity $ \lambda_j $ depends on the node $ \bt $---if $ \bt $ changes, so does $ s^* $. %Therefore $ \lambda_j $ is a local measure of the distance between an optimal split and its parent subnodes for a \emph{particular} node. 
We now introduce a more global measure, which depends only on the regression function.
\begin{definition}[Global balancedness] \label{def:edgegap}
%The node balancedness $ \Lambda \in [0, 1] $ is defined as the infimum of $ 4P(\bt^*_L)P(\bt^*_R) $ over all ancestor nodes $ \bt $, i.e.,
%$$
%\Lambda = \inf_{\bt \;\text{parent node}} 4P(\bt^*_L)P(\bt^*_R).
%$$
%where 
%$$ \bt \in \calT = \{ (a_1, b_1, a_2, b_2, \dots, a_d, b_d) \in [0, 1]^{2d} : a_{j'} < b_{j'}, \; j' = 1, 2, \dots, d \}. $$
The global balancedness $ \Lambda_j $ is defined as 
$$
%\Lambda = 1 - \sup_{\bt}\sqrt{\frac{|\overline G(s^*; \bt)|^2}{|\overline G(s^*; \bt)|^2+\Delta(s^* ; \bt)}}.
\Lambda_j = \inf_{\bt} \lambda_j(\bt),
$$
where the infimum runs over all parent nodes $ \bt $ of the best split left and right daughter nodes $ \bt^*_L $ and $ \bt^*_R $, respectively.
%where 
%$$ \bt \in \calT = \{ (a_1, b_1, a_2, b_2, \dots, a_d, b_d) \in [0, 1]^{2d} : a_{j'} < b_{j'}, \; j' = 1, 2, \dots, d \}. $$
\end{definition}

With these definitions in place, we are now ready to state \prettyref{thm:lambdalower}, which lower bounds $ \text{MDI}(X_j; \bt) $ in terms of the selection frequency for $ X_j $. For brevity, the proof is deferred until \prettyref{sec:lower}.
\begin{theorem} \label{thm:lambdalower}
Suppose
% the regression function is additive, $ f(\bx) = \sum_{j=1}^d f_j(x_j) $,
the $ j^{\Th}$ direction of $ f $ is not too ``flat'' in the sense that there exists a finite integer $ R \geq 1 $ such that for each $ x_j $ in $ [0, 1] $, there is a finite-order partial derivative, $ \tfrac{\partial^r}{\partial x^r_j}f(x_j, \bx_{\setminus j}) $ with $ 1 \leq r \leq R $, that is nonzero and continuous for all other input coordinates $ \bx_{\setminus j} $ in $ [0, 1]^{d-1}$. More formally, assume
%\begin{equation} \label{eq:flat}
%\sup_{x\in[0, 1]}\inf\{ r \geq 1 : f^{(r)}_j(x) \neq 0,\; f^{(r)}_j(\cdot)\; \text{continuous at}\; x \} < +\infty.
%\end{equation}
\begin{equation} \label{eq:flat}
%\sup_{\bx\in[0, 1]^d}\inf\{ \|\mathbf{r}\|_1 \geq 1 : \partial^{\mathbf{r}} f(\bx) \neq 0,\; \partial^{\mathbf{r}} f(\cdot)\; \text{continuous at}\; \bx \} < +\infty.
%& \sup_{x_j \in[0, 1]}\inf_{r\geq 1}\big\{ r : \tfrac{\partial^r}{\partial x^r_j}f(x_j, \bx_{\setminus j}) \neq 0\; \text{for all} \; \bx_{\setminus j} \in [0, 1]^{d-1}, \nonumber \\ & \qquad \tfrac{\partial^r}{\partial x^r_j}f(\cdot)\; \text{is continuous at}\; (x_j, \bx_{\setminus j})\; \text{for all} \; \bx_{\setminus j} \in [0, 1]^{d-1}\big\} < +\infty.
\sup_{x_j \in[0, 1]}\inf_{r\geq 1}\big\{ r : \tfrac{\partial^r}{\partial x^r_j}f(x_j, \bx_{\setminus j}) \; \text{is nonzero and continuous for all} \; \bx_{\setminus j} \in [0, 1]^{d-1}\big\}
\end{equation}
is finite.
If additionally the features of $ \bX $ are independent and have marginal densities that are continuous and never vanish, then the global balancedness $ \Lambda_j $ is strictly positive and
$$
\text{MDI}(X_j; \bt) \geq \Lambda_j K_j(\bt),
$$
where $ K_j(\bt) = \#\{ \bt' \supset \bt :  j_{\bt'} = j \} $ is the number of times $X_j$ was selected among all ancestor nodes of terminal node $ \bt $.
\end{theorem}

It will be shown in \prettyref{sec:ass} (see \prettyref{thm:radial}) that any linear combination of Gaussian radial basis functions in $\mathbb{R}^d $ with positive weights satisfies \prettyref{eq:flat}. Furthermore, \prettyref{eq:flat} also holds for any nonconstant, one-dimensional polynomial or partial sum of a Fourier series.

\begin{remark}
Taken together, \prettyref{thm:main} and \prettyref{thm:lambdalower} imply that $ \text{diam}_{\calS}(\bt) $ converges to zero (exponentially fast) in $ \Prob_{\bX} $-probability when $ K_j(\bt) \rightarrow +\infty $ for each $ j \in \calS $. The feature selection frequencies of important variables are typically large for deeply grown decision trees (in fact, $ K_j(\bt) $ is usually scales with the tree depth) and hence this condition is typically met. That is, the number of times $ \Delta(j, s^*, \bt') > \max_{j'\neq j}\Delta(j', s^*, \bt') $ at an ancestor node $ \bt' $ of $ \bt $ typically scales with the tree depth.
\end{remark}

\begin{remark}
Condition \prettyref{eq:flat} does not mean that all partial derivatives of $ f $ exist and are continuous. For example, the function $ f(\bx) = x_1 + (x_1-1/2)^2\text{sgn}(x_1-1/2) $ has discontinuous second derivative, yet still satisfies the condition.
\end{remark}

%For general multivariable regression functions, the situation is somewhat more complicated. A generalization of \prettyref{eq:flat} would be
%\begin{equation} \label{eq:flatmulti}
%\sup_{\bx\in[0, 1]^d}\inf\{ \|\mathbf{r}\|_1 \geq 1 : \partial^{\mathbf{r}} f(\bx) \neq 0,\; \partial^{\mathbf{r}} f(\cdot)\; \text{continuous at}\; \bx \} < +\infty.
%\sup_{\bx\in[0, 1]^d}\inf\Big\{ r \geq 1 : \frac{\partial^r}{\partial x^r_j}f(\bx) \neq 0,\; \frac{\partial^r}{\partial x^r_j}f(\cdot)\; \text{continuous at}\; \bx \Big\} < +\infty.
%\end{equation}
%Thus we define the family $ \calF $ to be the class of all regression functions for which \prettyref{eq:flatmulti} holds for all $ j = 1, 2, \dots, d $.

\subsection{Examples}
\prettyref{thm:lambdalower} does does not show how $ \Lambda_j $ depends on the structure of the regression function. It seems, at least for now, that such results are only available on a case-by-case basis and obtained with considerable effort. Here we give some example calculations of $ \Lambda_j $ for polynomial and trigonometric functions which decay inversely with the degree and periodicity, respectively. These theoretical results are accompanied by plots (see \prettyref{fig:poly} and \prettyref{fig:sine}) of $ \Delta(j, \cdot; \bt) $ together with sampling distributions of $ \hat s \in \argmax_s \widehat\Delta(j, s; \bt) $ from a sample size of $ n = 100 $ over $ 100 $ independent replications. Note that here a smaller sample size was purposely chosen to mimic a situation where the split is performed in a deep node and thus likely to contain only a small number of observations. As evidenced by the plots, the optimal splits tend to be closer to the parent subnode edges (in this case $ 0 $ and $ 1 $) with larger degree and periodicity. This phenomenon is manifested in the lower bounds on $ \Lambda_j $ in \prettyref{ex:poly} and \prettyref{ex:sine}---some of the terminal nodes will be large if splits from ancestor nodes are close to their parent subnode edges. In conjunction with \prettyref{thm:main} and the inverse relationship between the terminal node size and $ \text{MDI}(X_j; \bt) $ (being a weighted sum of $ \Delta(j, s^* ; \bt') $), the plots \prettyref{fig:poly} and \prettyref{fig:sine} also reveal that the splits tend to be near the edges when the reduction in impurity is small. In future sections, we will theoretically confirm this (see \prettyref{thm:mainprob} and \prettyref{thm:second}) and show, more generally, that splits occur near the edges of the parent subnode whenever $ \Delta(j, s^* ; \bt') $ is small.\footnote{Note that \prettyref{thm:mainprob} and \prettyref{thm:second}, used to prove this phenomenon, do not need \prettyref{ass:indep}.} This phenomenon has also been dubbed ``end-cut preference'' in the literature \citep{ishwaran2015effect}, \citep[Section 11.8]{breiman1984}. In \prettyref{sec:alternate}, we will study a penalized variant of $ \Delta(j, s^* ; \bt') $ in order to mitigate this effect.

For each of the following three examples, we assume that $ \bX $ is uniformly distributed on $ [0, 1]^d $.
%The following examples highlight the dependence of $ \Lambda_j $ on the structure of the regression function.
\begin{example} \label{ex:poly}
Suppose $ f(\bx) = \sum_{j=1}^d \beta_j x_j^{k_j} $ for nonzero constants $ \{ \beta_j \} $ and integer $ k_j \geq 0 $. Suppose $ j \in \calS $ so that $ k_j \geq 1 $. Then
$$
\Lambda_j \geq \left(\frac{1}{k_j(k_j+1)}\right)^{2/3}.
$$
It is possible to show that, more generally, if $ f(\bx) = \sum_{j=1}^d \beta_j (x_j-\alpha_j)^{k_j} $, for constants $ \{ \alpha_j \} $, then $ \Lambda_j \geq C 4^{-k_j/3}k_j^{-4/3} $ for some universal constant $ C > 0 $.
\begin{figure} [t] 
\centering
\begin{subfigure}[t]{0.45\textwidth}
  \centering
  \includegraphics[width=1\linewidth]{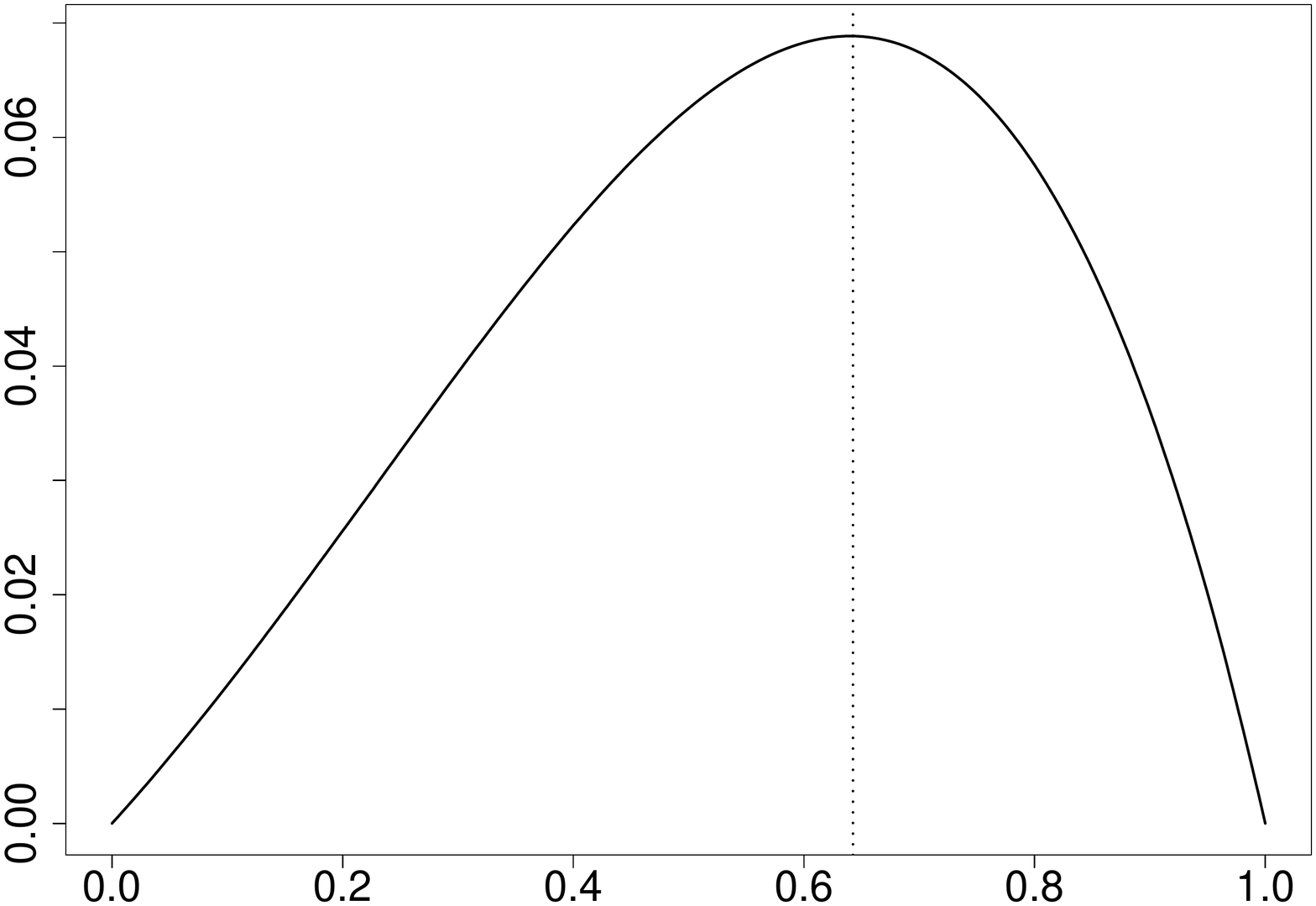}
  \caption{Plot of $ s \mapsto \Delta(1, s ; [0, 1]) $ for $ f(\bx) = x^{2}_1 $.}
\end{subfigure}
\hspace{1cm}
\begin{subfigure}[t]{0.45\textwidth}
  \centering
  \includegraphics[width=1\linewidth]{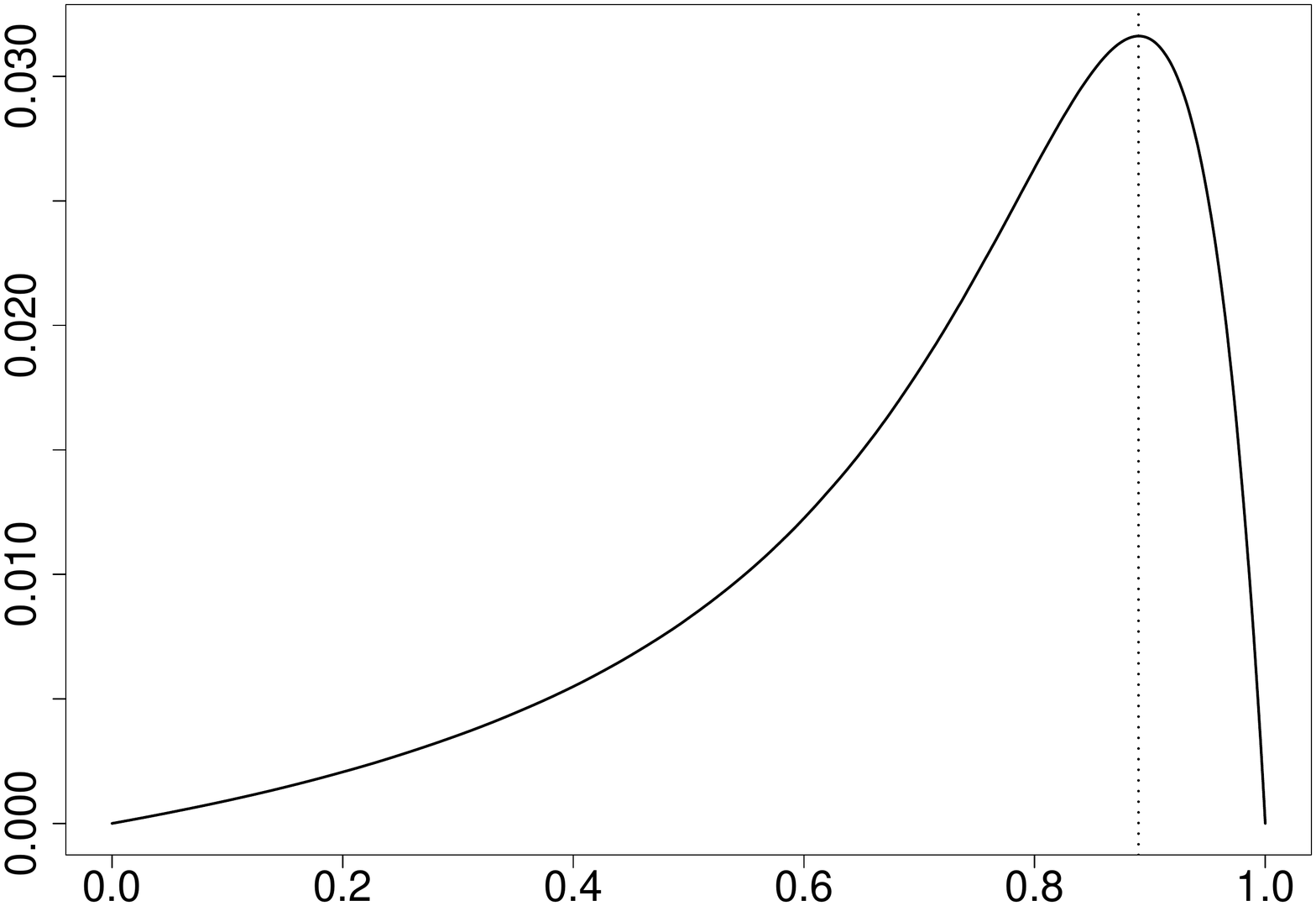}
  \caption{Plot of $ s \mapsto \Delta(1, s ; [0, 1]) $ for $ f(\bx) = x^{10}_1 $.}
\end{subfigure}
\hspace{1cm}
\begin{subfigure}[t]{0.45\textwidth}
  \centering
  \includegraphics[width=1\linewidth]{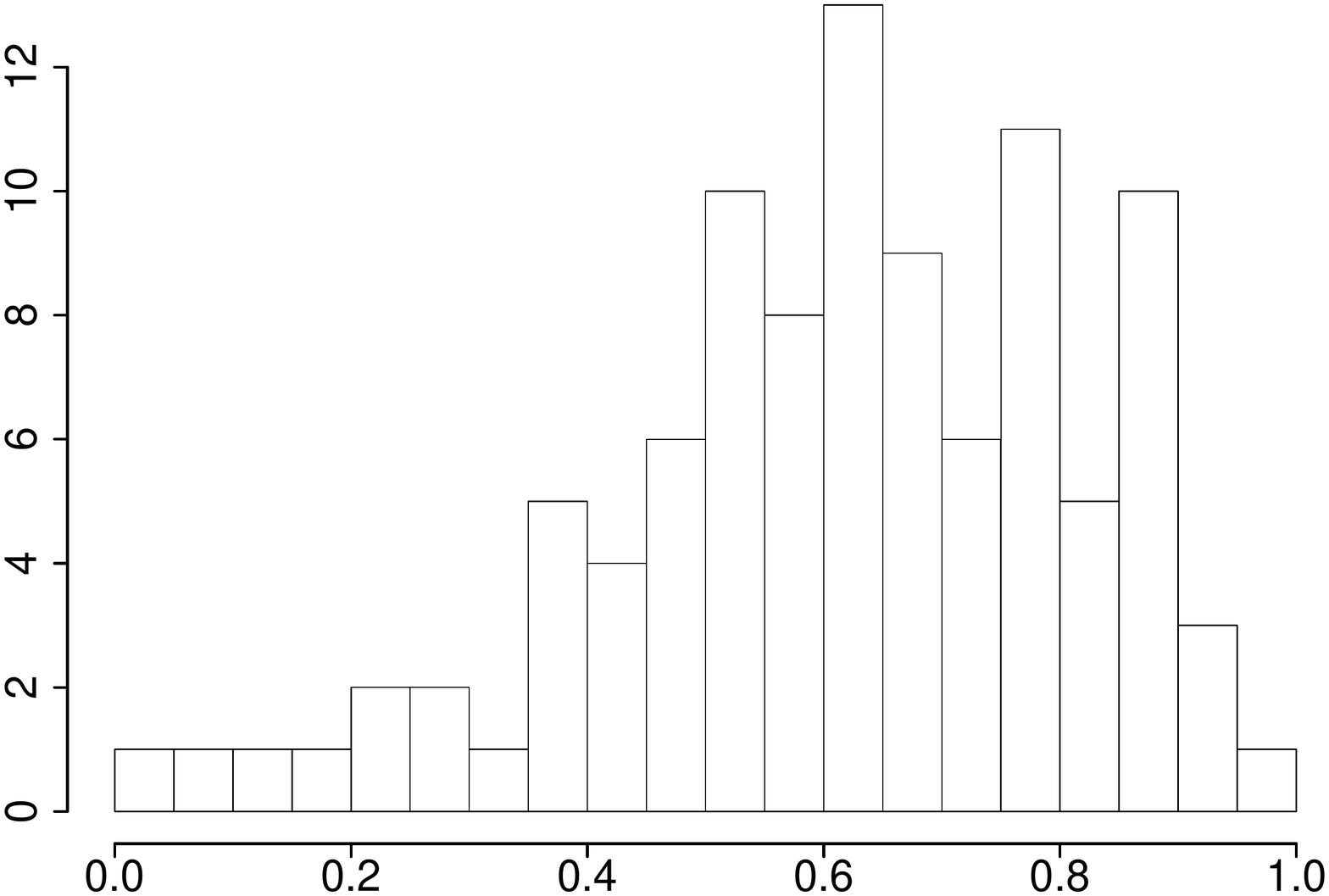}
  \caption{Histogram of $ \hat s $ for $ Y = X^{2} + \varepsilon $ ($ X \sim \Unif(0, 1) $, $ \varepsilon \sim N(0, 1) $) with $ n = 100 $ from $ 100 $ replications.}
  \end{subfigure}
\hspace{1cm}
\begin{subfigure}[t]{0.45\textwidth}
  \centering
  \includegraphics[width=1\linewidth]{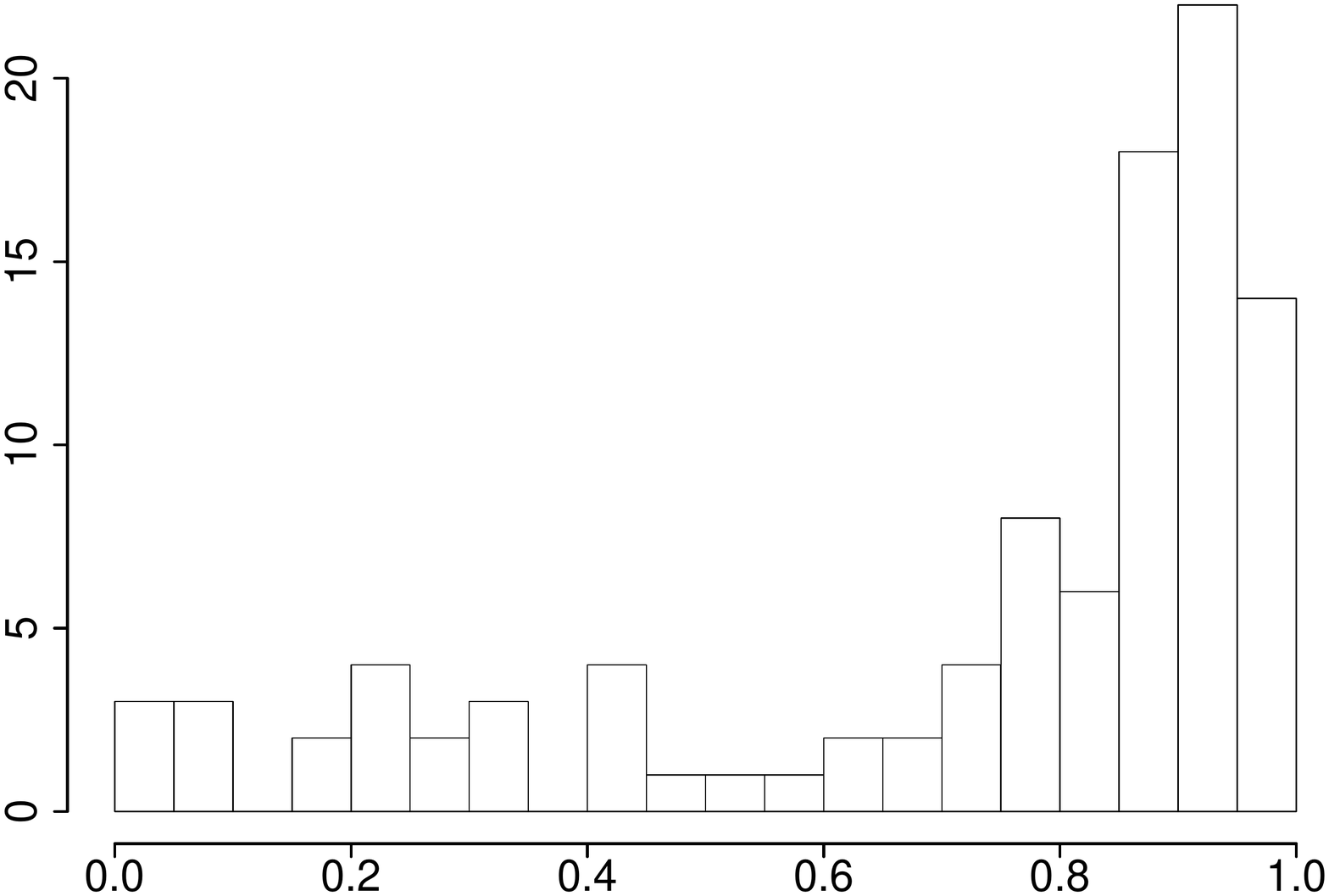}
  \caption{Histogram of $ \hat s $ for $ Y = X^{10} + \varepsilon $ ($ X \sim \Unif(0, 1) $, $ \varepsilon \sim N(0, 1) $) with $ n = 100 $ from $ 100 $ replications.}
\end{subfigure}
\caption{Plots of $ s \mapsto \Delta(1, s ; [0, 1]) $ and corresponding maxima (dotted vertical lines) for \prettyref{ex:poly}. Histograms show sampling distribution of $ \hat s $ for $ n = 100 $ from $ 100 $ replications.}
\label{fig:poly}
\end{figure}
\end{example}
\begin{example} \label{ex:sine}
Suppose $ f(\bx) = \sum_{j=1}^d \beta_j \sin(2\pi m_j x_j) $ for nonzero constants $ \{ \beta_j \} $ and integer $ m_j \geq 0 $. Suppose $ j \in \calS $ so that $ m_j \geq 1 $. There exists a universal constant $ C > 0 $ such that
$$
\Lambda_j \geq Cm_j^{-4/3}.
$$
\begin{figure} [t] 
\centering
\begin{subfigure}[t]{0.45\textwidth}
  \centering
  \includegraphics[width=1\linewidth]{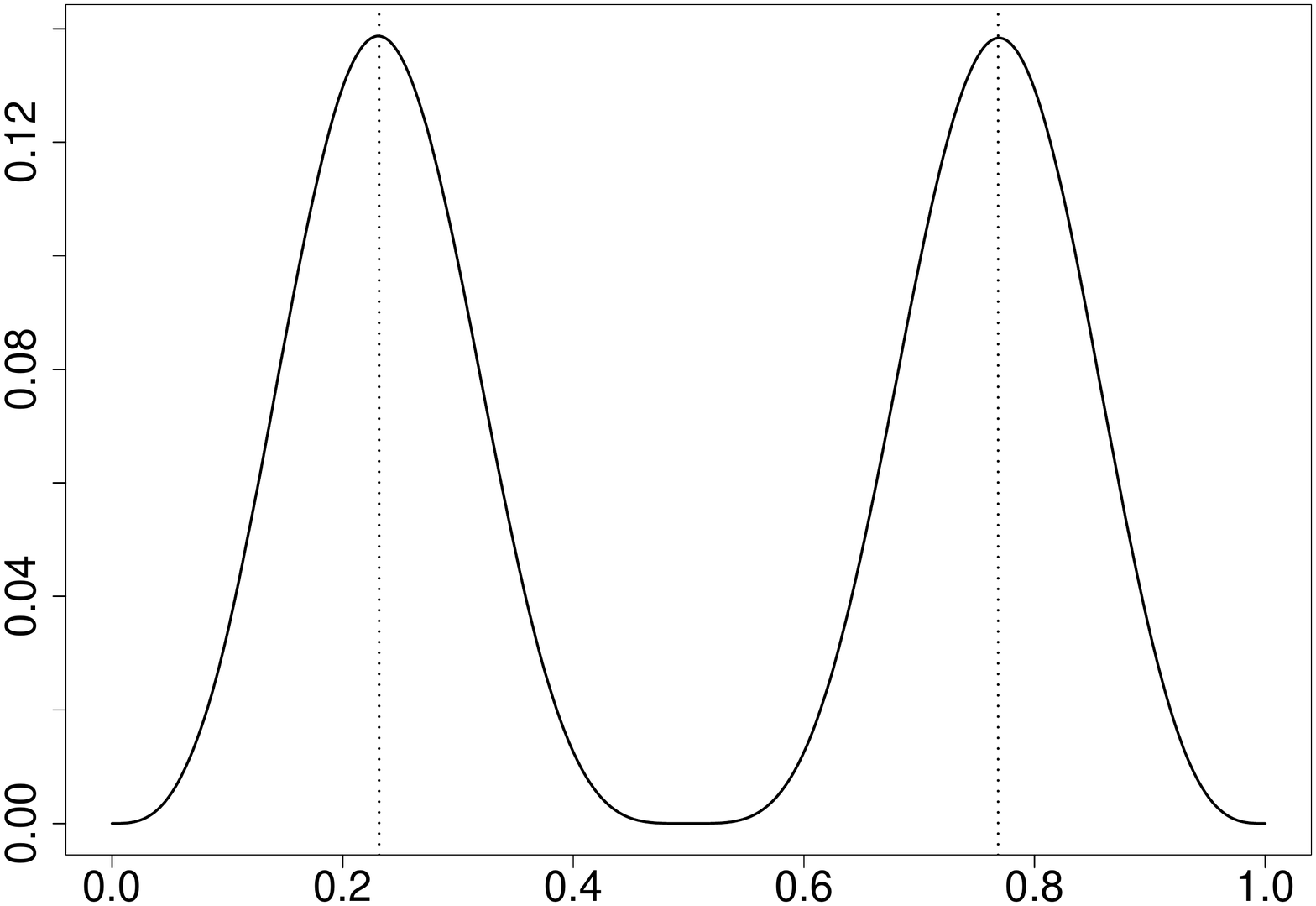}
  \caption{Plot of $ s \mapsto \Delta(1, s; [0, 1]) $ for $ f(\bx) = \sin(4\pi x_1) $.}
\end{subfigure}
\hspace{1cm}
\begin{subfigure}[t]{0.45\textwidth}
  \centering
  \includegraphics[width=1\linewidth]{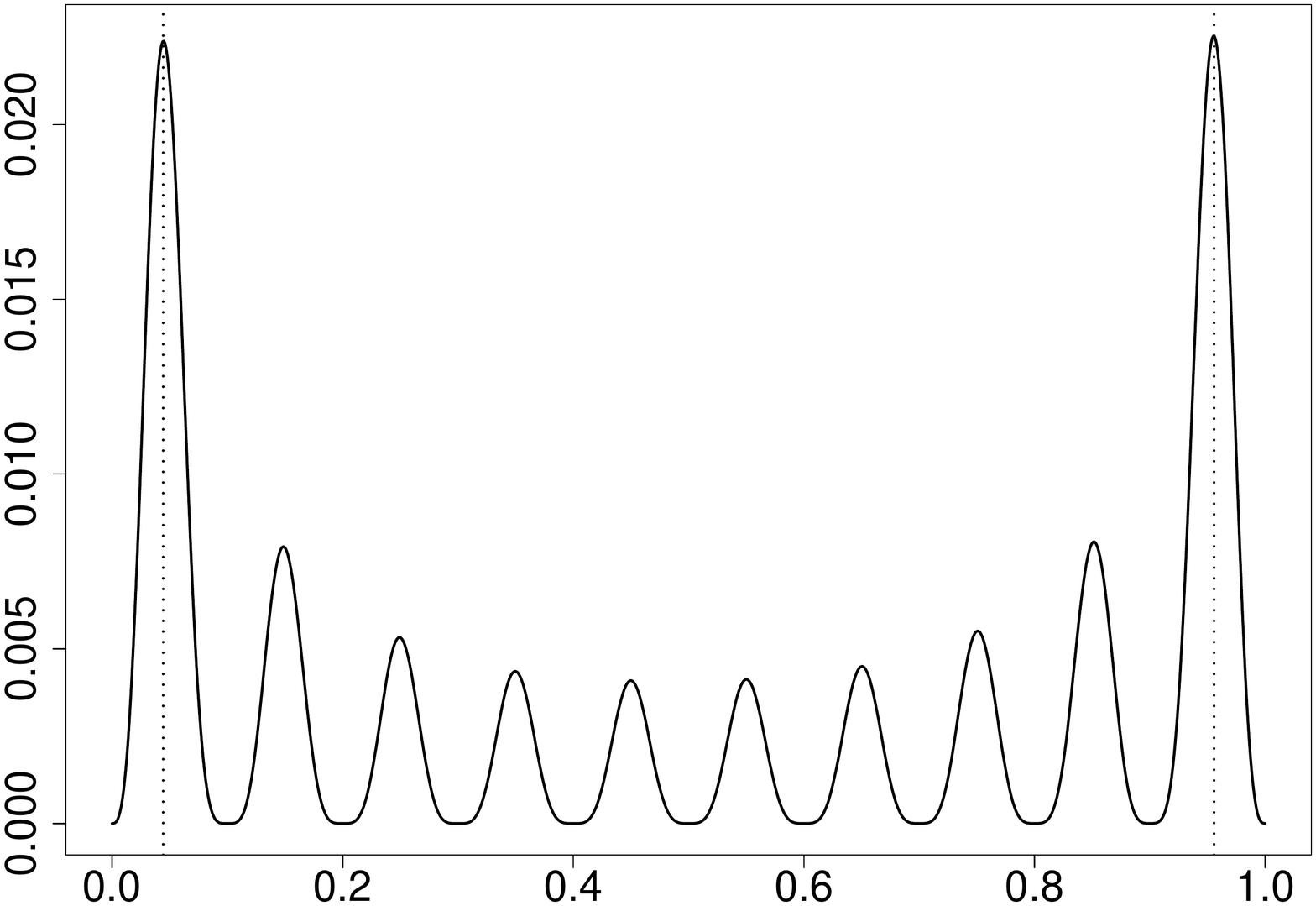}
  \caption{Plot of $ s \mapsto \Delta(1, s; [0, 1]) $ for $ f(\bx) = \sin(20\pi x_1) $.}
\end{subfigure}
\hspace{1cm}
\begin{subfigure}[t]{0.45\textwidth}
  \centering
  \includegraphics[width=1\linewidth]{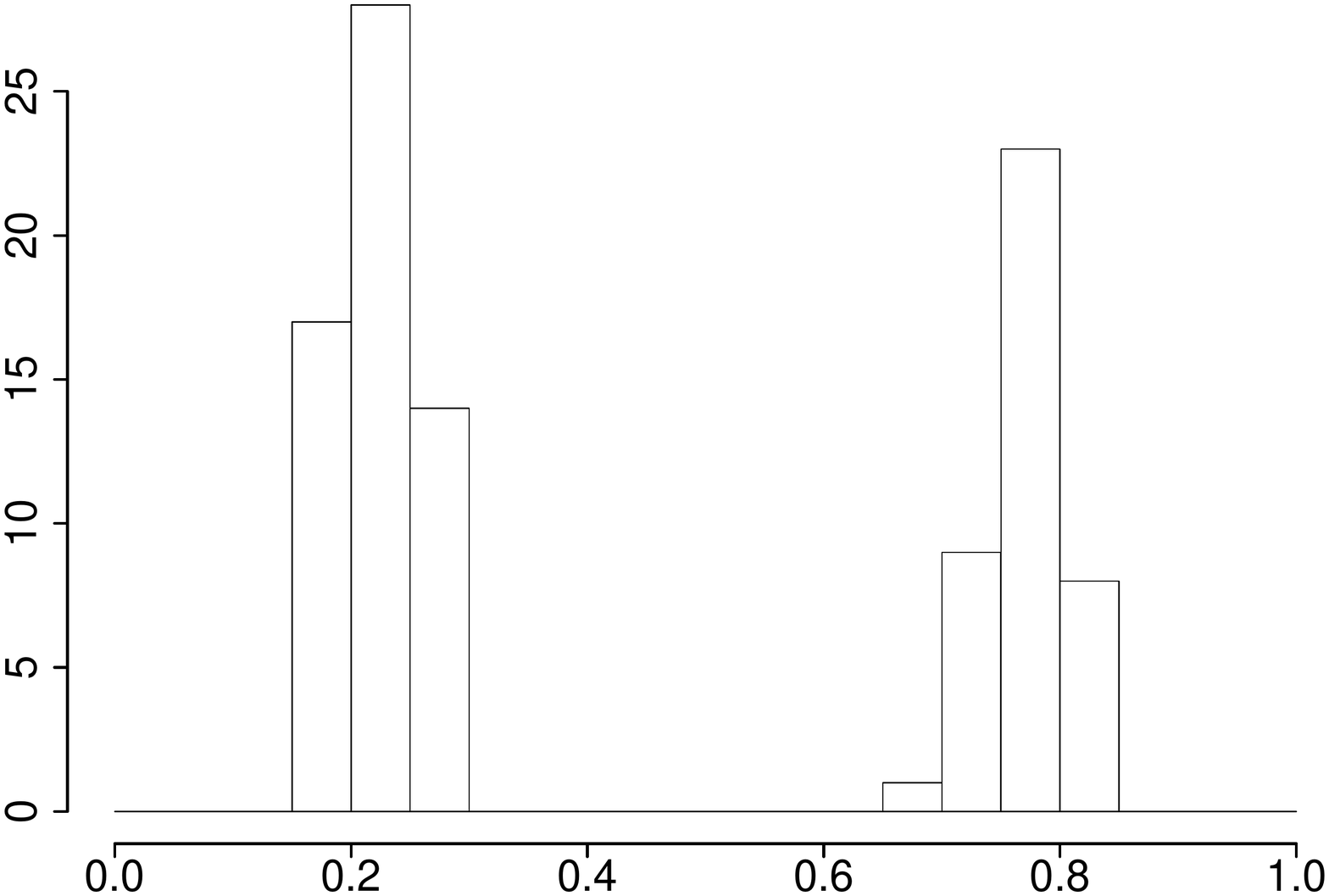}
  \caption{Histogram of $ \hat s $ for $ Y = \sin(4\pi X) + \varepsilon $ ($ X \sim \Unif(0, 1) $, $ \varepsilon \sim N(0, 1) $) with $ n = 100 $ from $ 100 $ replications.}
  \end{subfigure}
\hspace{1cm}
\begin{subfigure}[t]{0.45\textwidth}
  \centering
  \includegraphics[width=1\linewidth]{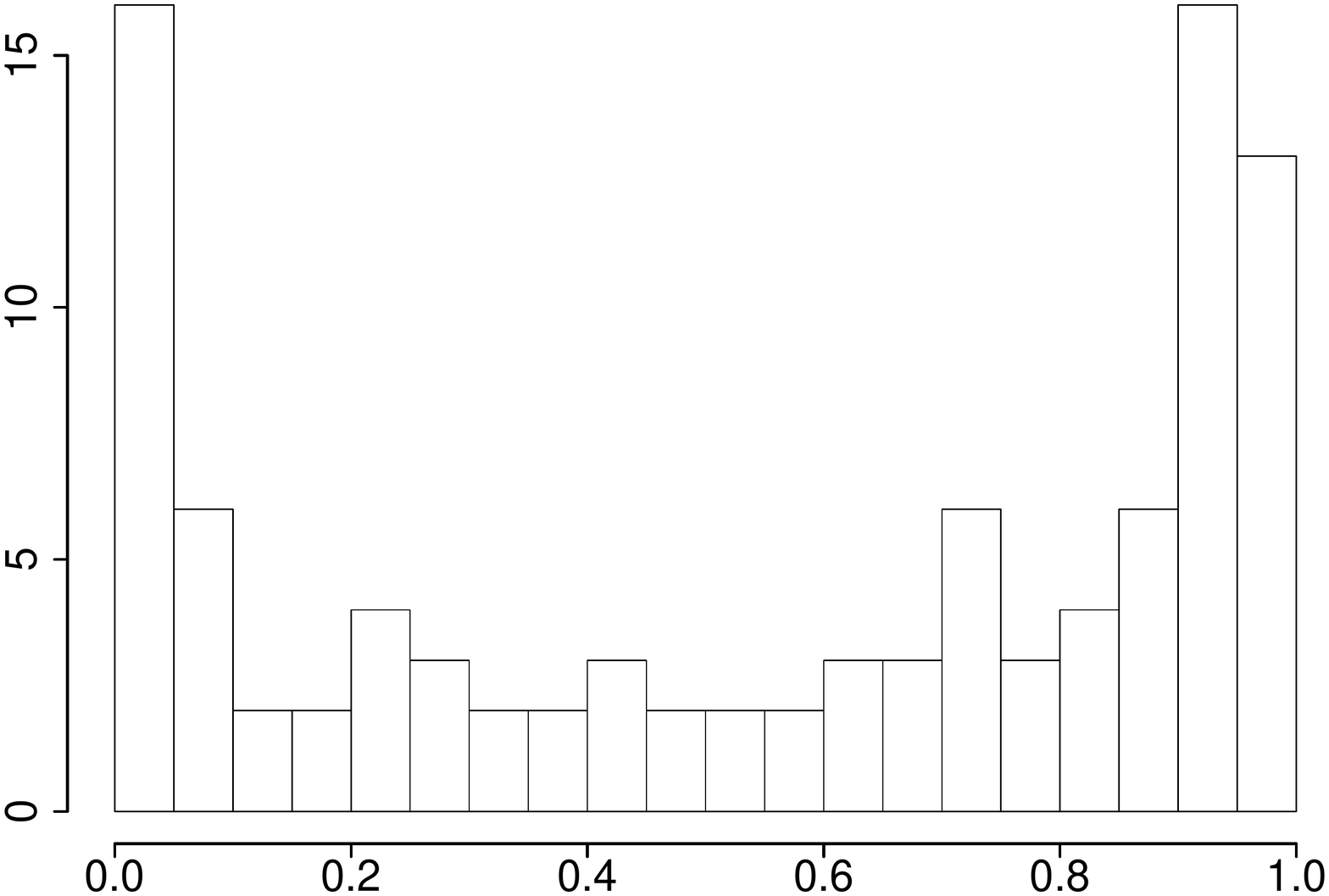}
  \caption{Histogram of $ \hat s $ for $ Y = \sin(20\pi X) + \varepsilon $ ($ X \sim \Unif(0, 1) $, $ \varepsilon \sim N(0, 1) $) with $ n = 100 $ from $ 100 $ replications.}
\end{subfigure}
\caption{Plots of $ s \mapsto \Delta(1, s; \bt) $ and corresponding maxima (dotted vertical lines) for \prettyref{ex:sine}. Histograms show sampling distribution of $ \hat s $ for $ n = 100 $ from $ 100 $ replications.}
\label{fig:sine}
\end{figure}
\end{example}
\prettyref{ex:sine} implies that to ensure each terminal subnode has small probability content, $ K_j(\bt) $ (or the tree depth) should be larger for functions that have larger frequencies in their frequency domain. This is not merely a coincidence. In fact, it can be shown more generally that if $ \bX $ is uniformly distributed, then $ \Delta(j, s^*; \bt) \geq \pi^{-2}\sum_{k\neq 0}|c_k|^2k^{-2} $, where $ \{c_k\} $ are the Fourier coefficients of the conditional partial dependence function $ \overline F_j(x_j; \bt) \sim \sum_k c_{k}e^{2\pi \mathrm{i} k (x_j-a_j(\bt))/(b_j(\bt)-a_j(\bt))} $. This can easily be established by lower bounding $ \Delta(j, s^*; \bt) $ by $ \int_{a_j(\bt)}^{b_j(\bt)}\Delta(j, s, \bt)\Pi(ds) $, where $ \Pi $ is the uniform prior on $ [a_j(\bt), b_j(\bt)] $, and using Parseval's identity. That is,
\begin{align*}
\Delta(j, s^*; \bt) & \geq \int_{a_j(\bt)}^{b_j(\bt)}\frac{(\int_{a_j(\bt)}^{s}\overline G_j(s'; \bt)ds')^2}{(s-a_j(\bt))(b_j(\bt)-s)}\Pi(ds) \\
& \geq \frac{4}{(b_j(\bt)-a_j(\bt))^3}\int_{a_j(\bt)}^{b_j(\bt)}\bigg(\int_{a_j(\bt)}^s\overline G_j(s'; \bt)ds'\bigg)^2ds \\
& = \frac{4}{b_j(\bt)-a_j(\bt)}\int_{a_j(\bt)}^{b_j(\bt)}\bigg|\sum_{k\neq 0} \frac{c_{k}}{2\pi \mathrm{i} k}(e^{2\pi \mathrm{i} k(s-a_j(\bt))/(b_j(\bt)-a_j(\bt))}-1)\bigg|^2ds \\
& = 4\bigg(\sum_{k\neq 0}\frac{|c_{k}|^2}{4\pi^2 k^2}+\bigg|\sum_{k\neq 0}\frac{c_{k}}{2 \pi k}\bigg|^2\bigg)
 \geq \pi^{-2}\sum_{k \neq 0}|c_{k}|^2k^{-2}.
\end{align*}

Our final example is a quintessential non-additive regression function known as ``Friedman \#1'' from \citep[Section 4.3]{friedman1991multivariate}. This function was used by Friedman to illustrate the efficacy of Multivariate Adaptive Regression Splines (MARS).
\begin{example} \label{ex:friedman}
Suppose 
\begin{equation} \label{eq:friedman}
f(\bx) = 10\sin(\pi x_1 x_2) + 20(x_3 - 1/2)^2 + 10x_4 + 5 x_5.
\end{equation}
Then each $ \Lambda_j $ is bounded from below in \prettyref{tab:tab2}.
\begin{table}
\centering
    \begin{tabular}{c|c}
\text{Variable} & $ \Lambda_j \geq $ \\
\hline
$ X_1 $ & $ 1/(2\pi^2) $ \\
\hline
$ X_2 $ & $ 1/(2\pi^2) $ \\
\hline
$ X_3 $ & $ (1/12)^{2/3} $ \\
\hline
$ X_4 $ & $ (1/4)^{1/3} $ \\
\hline
$ X_5 $ & $ (1/4)^{1/3} $ \\
\hline
%$ \vdots $ & $ \vdots $ \\
%$ x^k $ & $ \sqrt{\frac{Ck\log k}{1+Ck \log k}} $ \\
    \end{tabular}
    \caption{Lower bounds on $ \Lambda_j $ for $ j = 1, 2, 3, 4, 5 $ in \prettyref{ex:friedman}. For $ j = 1, 2 $, we use \prettyref{lmm:twodim}, for $ j = 3 $, we use \prettyref{lmm:square}, and for $ j = 4, 5 $, we use \prettyref{ex:poly} with $ k_j = 1 $.}
\label{tab:tab2}
\end{table}
\end{example}

\section{Assumptions on the regression function} \label{sec:ass}

Note that variable ``irrelevance'' is not the same as conditional independence, i.e., $ Y \perp  X_j \mid \bX_{\setminus j}, \; \bX \in \bt $, as some variable $ X_j $ can be irrelevant in the sense that $ \Delta(j, \cdot; \bt) = 0 $, yet it still influences the distribution of output values,  i.e., $ Y\not\perp  X_j \mid \bX_{\setminus j}, \; \bX \in \bt $.\footnote{See \citep{louppe2013understanding} for more details along these lines.} Therefore, in order to ensure that $ \text{MDI}(X_j; \bt) $ is large for all nodes $ \bt $ in direction with a truly strong signal, i.e., $ j \in \calS $, we need a condition like \prettyref{eq:flat} so that the regression function is marginally not too ``flat'' and hence $ \Delta(j, s^* ; \bt) > 0 $.\footnote{Conversely, $ \Delta(j', s^*, \bt) $ for $ j' \in \calS^c $ may be on the same order as $ \Delta(j, s^*, \bt) $ for $ j \in \calS $ due to spurious correlation between $ X_j $ and $ X_{j'} $. Thus, the CART algorithm may have a selection bias and spend unnecessary time splitting on $ X_{j'} $, when conditional on $ X_j $, this variable plays no role in determining the output, i.e., $ Y \perp X_{j'} \mid X_j, \; \bX \in \bt $. 
%For example, let $ f(\bx) = \beta_1 x_1 + \beta_2 x_2 $ with $ \beta_2 = 0 $ and consider a joint distribution on $ X_1 $ and $ X_2 $, where the marginals are uniform. Then $ \Delta(2, s^*, [0, 1]^2) \rightarrow \Delta(1, s^*, [0, 1]^2) = \beta^2_1/16 $ as the correlation between $ X_1 $ and $ X_2 $ approaches $ \pm 1 $. 
This situation is less troublesome from an approximation error perspective, but becomes highly relevant for variable selection and interpretability.} Condition \prettyref{eq:flat} is sufficient (but not necessary\footnote{For example, in one dimension, there are examples of regression functions whose $ r^{\Th} $ order derivatives in \prettyref{eq:flat} have jump or removable discontinuities, yet $ \Lambda_j > 0 $.}) for the existence of a positive $ \Lambda_j $ which depends only on the regression function.
In the additive case, i.e., when $ f(\bx) = \sum_{j=1}^d f_j(x_j) $, condition \prettyref{eq:flat} is stronger than saying that each $ f_j $ is continuous and nonconstant on each subnode. For example, flat functions, i.e., $ f(\bx) = e^{-1/x^2_1}\indc{x_1 \neq 0} $ with $ \frac{\partial^r}{\partial x^r_1}f(\bx)\mid_{x_1=0} = 0 $ for all $ r $, or functions with derivatives having essential discontinuities, i.e., $ f(\bx) = x^2_1\sin(1/x_1)\indc{x_1\neq 0} $ may be both continuous and nonconstant on each subnode, yet violate \prettyref{eq:flat}. In particular, for both functions, $ \Delta(1, s^*; \bt) > 0 $ for all nodes---however, $ \Lambda_1 $ must equal zero. 

%These functions along with successive splits from $ \widehat\Delta(\cdot, \bt) $ are plotted in \prettyref{fig:gap}. We used $ n = 10^6 $ with no additive noise and generated the recursive partition (approximately $ 30 $ splits) via \texttt{rpart} in R (with default settings, except for $ \texttt{cp} = 5 \times 10^{-5} $). Notice the large gaps in the partition near zero. Since both functions are nonconstant on any subnode, these large gaps may produce a large bias. 
In the multivariable, non-additive case, the function $ f(\bx) = x_1+x_2-2x_1x_2 $ does not satisfy \prettyref{eq:flat} because $ f(x_1, 1/2) $ and $ f(1/2, x_2) $ are both equal to $ 1/2 $ and therefore are constant. In fact, if $ \bX $ is uniformly distributed and $ \bt = [0, 1]^2 $, then integrating out either direction yields a constant function on $ [0, 1] $ and hence $ \Delta(j, s^* ; \bt) = 0 $ for $ j = 1, 2 $.
%However, $ f(\bx) = x_1+x_2+2x_1x_2 $ does satisfy \prettyref{eq:flat} since $ \frac{\partial}{\partial x_1}f(x_1, x_2) $ and $ \frac{\partial}{\partial x_2}f(x_1, x_2) $ are both at least $ 1 $ for all input choices. 
%\begin{assumption} \label{ass:ass3}
%There is no variable $ x $ and no interval $ [a, b] $ for which $ \expect{Y \mid \bX = \bx} $ is constant in $ x $ on $ [a, b] $.
%\end{assumption}
%Roughly, this means that if we split along a noisy variable, the variation of the function in that subnode is small.
%The contrapositive statement of \prettyref{ass:partial} can be stated as follows: if the regression function depends on a coordinate $ x $ (i.e., is nonconstant in $ x $), then the conditional expectation of $ Y $ given that $ \bX \in \bt $ and $ X = x $ also depends on $ x $.
%and \prettyref{ass:ass3}. 
More generally, if $ \Delta(j, \cdot; \bt) = 0 $ for $ j = 1, 2, \dots, d $,
%if the input distribution $ \bX $ is uniform on $ [0, 1]^d $ and
%\begin{equation} \label{eq:int}
%\overline F_j(x_j ; \bt) = C_j, \quad \text{for all}\; x_j \in  [a_j(\bt), b_j(\bt)], \quad j = 1, 2, \dots, d,
%\end{equation} 
then any split along any variable in $ \bt $ results in a zero decrease in impurity \citep[Technical Lemma 1]{scornet2015supp}, despite the possibility that the regression function is nonconstant on $ \bt $.
%If the input distribution is uniform on $ [0, 1]^d $ and the one-dimensional partial integrals of the regression function over a subnode of $ [0, 1] $ are constant [i.e., when all but one variable is integrated out, $ \int_{A_{\,-j}} f(x)dx_{\,-j} = C_j $ for $ A_{\,-j} \subset [0, 1]^{d-1} $ and all $ x \in A_j \subset [0, 1] $ and all $ j = 1, 2, \dots, d $], then any split along any variable results in a zero decrease in the population conditional variance \citep[Technical Lemma 1]{scornet2015supp}, despite the fact that the regression function may be nonconstant on its entire domain. 
Such a situation may lead one to erroneously classify certain features as ``weak'' when they may not be so. Practically speaking, this means that the CART algorithm may ignore certain variables and therefore fail to create a fine enough partition of $ [0, 1]^d $, which may introduce a large amount of bias. Consider again the example $ f(\bx) = x_1+x_2-2x_1x_2 $ with uniformly distributed predictors and suppose we wish to split at node $ \bt = [0, 1]^2 $.
% take $ \bt = [0, 1]^d $ and the function $ f $ to be a Lipschitz probability density function, i.e., $ f(\bx) = \partial_{\bx}\overline F(\bx) $, with distribution function
%\begin{equation} \label{eq:bad}
%\overline F(\bx) = \frac{x_1x_2\cdots x_d}{1+(1-x_1)(1-x_2)\cdots(1-x_d)}, \quad \bx \in [0, 1]^d,
%\end{equation}
%or
%\begin{equation} \label{eq:bad}
%\overline F(\bx) = x_1x_2\cdots x_d(1-(1-x_1)(1-x_2)\cdots(1-x_d)), \quad \bx \in [0, 1]^d,
%\end{equation}
%and $ F = 0 $ otherwise. [Note that $ f $ is a simple polynomial!] 
%Every $ S $-dimensional [$ S < d $] marginal distribution of the density $ f $ is uniform (i.e., constant) on $ [0, 1]^S $ and hence 
Any further splitting via the CART protocol will result in zero impurity reduction. 
%[In fact, the density function of any non-uniform copula has this property.] 
Thus, one may be tempted to assume the function is constant on the node when in fact $ f $ ``strongly'' depends on both variables. 
%In fact, even if $ S $ [$ S < d $] variables are split along simultaneously, the $ S $-dimensional partial integrals are still constant on any node of $ [0, 1]^S $. 
Hence, we require some sort of ``self-consistency'' property of the regression function, namely, if $ \Delta(j, \cdot; \bt) $ is zero for all splits $ s \in [a_j(\bt),b_j(\bt)] $ for $j = 1, 2, \dots, d $, then $ f $ \emph{is} constant on $ \bt $ and therefore the bias of the tree on that node is zero. With this self-consistency assumption, even though the diameters of the nodes may not converge to zero, for the purpose of controlling the approximation error, one can safely ignore the regression function entirely on that node, regardless of whether the algorithm actually performs any further splitting.

Despite the aforementioned difficulties, we show in the sequel that any linear combination of Gaussian radial basis functions with positive weights satisfies \prettyref{eq:flat}. To this end, consider the function class
\begin{align*}
\calF & = \Big\{\sum_{k=1}^K w_k \exp\{(\bx-\bmu_k)^{\top}\bSigma_k(\bx-\bmu_k)\}: w_k \geq 0, \bmu_k \in \mathbb{R}^d, \bSigma_k\; \text{diagonal}\}.
\end{align*}
We allow for the possibility that $ \bSigma_k $ contains diagonal entries that are equal to zero, with the interpretation that the corresponding term is independent (constant) in that coordinate. It is known that $ \calF $ is a dense subclass of all positive continuous functions on $ [0, 1]^d $ \citep{park1991universal}. We have the following theorem.

%\begin{theorem}[\citep{scornet2015}]
%Suppose the features of $ \bX $ are independent and each density is bounded away from zero. Let $ f \in \calF'_0 $. If along any direction $ X $,
%$$
%\expect{Y \mid \bX \in \bt, \; X = s} = \expect{Y \mid \bX \in \bt}, \quad s \in [a,b],
%$$
%then $ f $ is constant on $ \bt $.
%\end{theorem}

\begin{theorem} \label{thm:radial}
%If the features of $ \bX $ are independent and $ f \in \calF $, then \prettyref{ass:partial} holds.
%Suppose the features of $ \bX $ are independent and $ f \in \calF $. 
If $ j \in \calS $, then any $ f \in \calF $ satisfies \prettyref{eq:flat}.
%If $ \overline F_j(x_j ; \bt) $ is constant in $ x_j $ on a subnode, then $ j \in \calS^c $. Equivalently, if $ \Delta(j, s; \bt) = 0 $ for all splits $ s $ on a subnode, then $ j \in \calS^c $.
\end{theorem}
%\begin{remark}
%One can also reach the same conclusion as \prettyref{thm:radial} for general weights $ w_k \in \mathbb{R} $ by assuming that each component of $ \bmu_k $ and $ \bD_k $ are distinct for all $ k $.
%\end{remark}
\begin{proof}
%Regression functions in $ \calF $ satisfy the same technical condition as additive regression functions. 
Without loss of generality, suppose each weight of combination $ w_k $ is strictly positive. Then $ f(\bx) $ as a function of $ x_j $ has the form of a one-dimensional combination of Gaussian radial basis functions with positive weights, i.e., $ \sum_{k=1}^K w'_k \exp\{h_k(x_j-\mu_k)^2\}, $ where $ w'_k > 0 $, $ h_k $, and $ \mu_k $ belong to $ \mathbb{R} $. It can further be assumed without loss of generality that each Gaussian function $ \exp\{h_k(x_j-\mu_k)^2\} $ in the sum is distinct. Suppose, contrary to hypothesis, that $ f $ does not satisfy \prettyref{eq:flat}. Then, since $ f $ has continuous partial derivatives of all orders, there exists $ \bx' \in [0, 1]^d $ such that $ \frac{\partial^r}{\partial x^r_j}f(\bx') = 0 $ for all $ r \geq 1 $. However, $ f $ as a function of $ x_j $ is analytic and so $ f(x_j, \bx'_{\setminus j}) $ is constant in a neighborhood of $ x'_j $, say $ (a, b) $.
%Suppose that $ \sum_{k=1}^K w'_k \psi(h_k(x_j-\mu_k)^2/2) $ is constant for all $ x_j \in [a(\bt),b(\bt)] $ and $ \bx_{\setminus j} \in [0, 1]^{d-1} $, which would correspond to $ \Delta(j, s; \bt) = 0 $ for all $ s \in [a_j(\bt),b_j(\bt)] $.
%a zero decrease in the infinite sample conditional variance of $ Y $ in the node $ \bt $. 
This implies that the partial derivative of $ f $ with respect to $ x_j $ at $ (x_j, \bx'_{\setminus j}) $, namely $2\sum_{k=1}^Kh_kw'_k(x_j-\mu_k)\exp\{h_k(x_j-\mu_k)^2\} $, is equal to zero for all $ x_j \in (a, b) $. 
%Suppose that it is equal to zero for all $ x_j \in [a_j(\bt),b_j(\bt)] $ and $ \bx_{\setminus j} \in [0, 1]^{d-1} $, which would correspond to $ \Delta(j, s; \bt) = 0 $ for all $ s \in [a_j(\bt),b_j(\bt)] $.
Thus, one is confronted with the question about the multiplicity of ways to represent the zero function of an exponential polynomial. To answer this, consider the following lemma, which is a special case of a much more general result in complex analysis \citep[Theorem 1.6]{lang1987introduction}, \citep{green1972holomorphic} and can be deduced using induction and differentiation. For the sake of completeness, we include its proof in \prettyref{app:proofs}.

\begin{lemma}[Linear independence of exponential polynomials] \label{lmm:exppoly}
Suppose $ P_1, \dots, P_K $ are distinct (real or complex) polynomials without constant terms and $ R_1, \dots, R_K $ are (real or complex) polynomials. If
$ \sum_{k=1}^KR_ke^{P_k} = 0 $ on an open subset (of the reals or complex plane), then $ R_1 = \cdots = R_K = 0 $.
\end{lemma}

Note that $ 2\sum_{k=1}^Kh_kw'_k(x_j-\mu_k)\exp\{h_k(x_j-\mu_k)^2\} $ is an exponential polynomial of the form $ \sum_{k=1}^KR_k(x_j)e^{P_k(x_j)} $ and hence by \prettyref{lmm:exppoly} above, $ R_k(x_j) = 2h_kw'_k(x_j-\mu_k) = 0 $, implying that $ h_k = 0 $ for all $ k $ (since $ w'_k > 0 $). Thus $ f $ is constant in $ x_j $ on $ [0, 1] $ for all $ \bx_{\setminus j} \in [0, 1]^{d-1} $, which is a contradiction to the assumption that $ j \in \calS $. 
%A similar conclusion can be made for general linear combinations [i.e., with possibly negative weights] of Gaussian radial basis functions, except that $ \bmu_k = (\mu_{k1}, \dots, \mu_{kd}) $ and $ \bsigma_k = (\sigma_{k1}, \dots, \sigma_{kd}) $ are required to have unique coordinates $ (\mu_{kj}, \sigma_{kj}) $ across $ k $.
\end{proof}

\begin{remark}
A similar argument to \prettyref{thm:radial} can be used to show that any nonconstant polynomial or partial sum of a Fourier series in $ \mathbb{R} $ also satisfies \prettyref{eq:flat}. Hence, there is a dense collection of additive regression functions $ f(\bx) = \sum_{j=1}^d f_j(x_j) $ that satisfies \prettyref{eq:flat}.
\end{remark}

Individual decision trees are not competitive predictors, since their high variability and tendency to overfit makes them generalize poorly to new data. Random forests, on the other hand, are an archetypal example of variance reduction via ensemble averaging, where many weak predictors (such as decision trees) are combined to form a stronger predictor. Next, we will use \prettyref{thm:main} and \prettyref{thm:lambdalower} to show asymptotic consistency of Breiman's random forests grown with the infinite sample CART criterion.

\section{Application to random forests} \label{sec:rf}

Random forests are ubiquitous among ensemble averaging algorithms because of their ability to reduce overfitting, handle high-dimensional sparse settings, and efficient implementation. 
%As a method that grows many base tree predictors and then combines them, 
%They are related to kernel regression \citep{breiman2000, geurts2006, scornet2016}, adaptive nearest neighbors \citep{lin2006}, and AdaBoost \citep{wyner2017}. 
Due to these attractive features, they have been widely adopted and applied to various prediction and classification problems, such as those encountered in bioinformatics and computer vision.
%One of the most successful ensemble learners is random forests, a method introduced by \citep{breiman2001}. 
The base learner for a random forest is a binary tree constructed using the methodology of CART. Naturally, some of our analysis for decision trees can be carried over to random forests. We explore this connection in this section.

%Classification and Regression Tree (CART) \citep{breiman1984}---a recursive procedure in which binary splits recursively partition the tree into homogeneous or near-homogeneous terminal nodes (the ends of the tree). A good binary split partitions data from the parent tree-node into two daughter nodes so that the ensuing homogeneity of the daughter nodes is improved from the parent node. 
Random forests grow an ensemble of $\texttt{ntree}$ regression trees.\footnote{In what follows, we use typewriter fonts for the variable names in the R package \texttt{randomForest}.} Each tree is grown independently using a bootstrap sample of the original data (also known as a bagged decision tree). As with traditional decision trees, terminal nodes of the tree consist of the predicted values which are then aggregated by averaging to obtain the random forest predictor.
%For example, in classification, each tree casts a vote for the class and the majority vote determines the predicted class label. 
Unlike CART decision trees, random forest trees are grown nondeterministically with two levels of randomization. In addition to the randomization introduced by growing the tree using a bootstrap sample, a second layer of randomization is injected with a random feature selection mechanism. Here, instead of splitting a tree node using all $d$ features, the random forest algorithm selects, at each node of each tree, a random subset of $ \texttt{mtry} $ potential variables that are used to further refine the tree node by splitting. The number of potential variables \texttt{mtry} is often much smaller than $d$; for regression, the default value is $ \lfloor d/3 \rfloor $. This two-level randomization is designed to decorrelate trees and therefore reduce variance. To reduce bias, random forest trees are grown deeply%Indeed, Breiman's original proposal called for splitting to purity in classification problems.
---in fact, each tree is grown as deeply as possible with the stipulation that each terminal node contains at least $ \texttt{nodesize} $ observations.

More concretely, a random forest is a predictor that is built from an ensemble of randomized base regression trees $\{\widehat Y(\bx; \Theta_m, \calD_n) \}_{1 \leq m \leq \texttt{ntree}}$. The sequence $ \{\Theta_m\}_{1 \leq m\leq \texttt{ntree}} $ consists of i.i.d. realizations of a random variable $\Theta$, which governs the probabilistic mechanism that builds each tree. These individual random trees are aggregated to form the final output
\begin{equation} \label{eq:finitetree}
\widehat Y(\bx; \Theta_1, \dots, \Theta_{\texttt{ntree}}, \calD_n) \triangleq \frac{1}{\texttt{ntree}}\sum_{m=1}^{\texttt{ntree}} \widehat Y (\bx; \Theta_m, \calD_n).
\end{equation}
When $ \texttt{ntree} $ is large, the law of large numbers justifies using
\begin{equation*}
\widehat Y(\bx) = \widehat Y(\bx; \calD_n) \triangleq \Expect_{\Theta}\left[\widehat Y(\bx;\Theta,  \calD_n)\right],
\end{equation*}
in lieu of \prettyref{eq:finitetree}, where $\Expect_{\Theta}$ denotes expectation with respect to $ \Theta $, conditionally on $\calD_n$. 
We shall henceforth work with these infinite sample versions (i.e., infinite number of trees) of their empirical counterparts (i.e., finite number of trees).

Let us now briefly describe the random feature mechanism of random forests in greater detail. To ensure that candidate strong (resp. weak) coordinates have high (resp. low) selection probabilities, at each step, we randomly select without replacement a subset $ \calM \subset \{1,\dots,d\} $ of cardinality $ \texttt{mtry} $ and then select the variable in $ \calM $ and corresponding split $ s^* $ that most decreases impurity within the current node. That is, for each coordinate in $ \calM $, calculate a split $ s^* $ that maximizes $ \Delta(\cdot, \bt) $ and store the corresponding maximum value $ \Delta(j, s^* ; \bt) $. Finally, select one variable $ X_{j^*} $ at random among the corresponding largest elements of $ \{\Delta(j, s^* ; \bt)\}_{j\in \calM} $ to further split along within the current node. For a more detailed discussion of the algorithm, see \citep{scornet2015}.
As is argued in \citep[Section 3]{biau2012}, this random feature selection procedure will produce selection probabilities $ \mathbb{P}_{\Theta}[j_{\bt}(\Theta) = j] $
% \sum_{k=1}^{S} \frac{1}{k}\frac{\binom{S-1}{k-1}\binom{d-S}{\texttt{mtry}-k}}{\binom{d}{\texttt{mtry}}} = \frac{1}{S}\left(1-\frac{\binom{d-S}{\texttt{mtry}}}{\binom{d}{\texttt{mtry}}}\right) $ and $ \frac{1}{d-S}\frac{\binom{d-S}{\texttt{mtry}}}{\binom{d}{\texttt{mtry}}} $.
that concentrate around $ 1/S $ for $ j \in \calS $ and zero otherwise. Hence each ``strong'' variable has roughly an equal chance of being selected among all ``strong'' variables.

Researchers have spent a great deal of effort in understanding theoretical properties of various streamlined versions of Breiman's original algorithm \citep{genuer2010, genuer2012, arlot2014, biau2008, denil2014, biau2012, scornet2016asymptotics}. See \citep{biau2016} for a comprehensive overview of current theoretical and practical understanding. Unlike Breiman's CART algorithm, these stylized versions are typically analyzed under the assumption that the probabilistic mechanism $ \Theta $ that governs the construction of each tree \emph{does not depend} on the pair $ (\bX, Y) $ (i.e., the splits to not depend on the data distribution), largely with the intent of reducing the complexity of their theoretical analysis. Such models are referred to as ``purely random forests'' \citep{genuer2012}. In one variant known as a ``centered random forest'' (proposed by Breiman himself in a technical report \citep{breiman2004} and later studied by \citep{biau2012}), the splits are performed at the midpoint of each subnode and hence corresponds to a special case of the present paper, where the input distribution is uniform and the regression surface is linear.

Inspired by the results of \prettyref{thm:main} and \citep[Theorem 4.1]{scornet2016asymptotics}, we now show asymptotic consistency for Breiman's random forests with splits determined by the infinite sample CART sum of squares error criterion. Consistency of random forests with CART (albeit, with the finite sample splitting criterion) was previously only known when the regression function has an additive structure \citep{scornet2015}. While this important work provides insight into the complicated and subtle mechanisms of random forests, it still does not adequately explain its potential as a general nonparametric method. For example, additive models are not flexible enough to allow for interactions among covariates (which limits their flexibility for multi-dimensional statistical modeling), and there are already other highly effective training algorithms such as backfitting \citep{breiman1985estimating}. 

We follow the terminology of Scornet \citep{scornet2016asymptotics} and call a random forest ``totally nonadaptive'' if it is built independently of the training set $ \calD_n $. Let $ \texttt{maxnodes} $ denote the maximum number of terminal nodes in the tree built with randomness $ \Theta $.
%\footnote{This parameter is related to \texttt{maxnodes}, the maximum number of terminal nodes, in the R package \texttt{randomForest}.} 
Due to space constraints, we defer the proof of the next result, \prettyref{thm:forest}, until \prettyref{app:mainproofs}. For the statement of \prettyref{thm:forest}, we let $ \text{MDI}(X_j; \bt) $ denote the conditional mean decrease in impurity \prettyref{eq:imp} with weights from \prettyref{thm:main}.
% Thus, we find that
%$$
%\text{MDI}(X_j; \bt) \gtrsim \frac{\Lambda_j k_n}{S},
%\frac{1}{S}\sum_{\bt' \supset \bt}\mathbb{E}_{\Theta}\left[\frac{\Delta(s^*_j; \bt')}{|\overline G(s^*_j; \bt')|^2+\Delta(s^*_j; \bt')} \mid j_{\bt'}(\Theta) =j \right],
%$$
%which is implicitly independent of the ambient dimension $ d $. 

%The theoretical CART random forest is related to the theoretical $ \Lambda/2 $ quantile forest, and hence much of the techniques for its analysis are inherited from quantile forests.
\begin{theorem} \label{thm:forest}
Consider a totally nonadaptive forest predictor $ \widehat Y(\bx) = \widehat Y(\bx; \calD_n) $, where each tree is grown with the infinite sample CART sum of squares error criterion \prettyref{eq:pop}.
%according to \prettyref{algo:cartpop}. 
Suppose that
\begin{enumerate}[(a)]
\item $ f(\bx) $ is continuous on $ [0, 1]^d $;
%\item $ k_n \rightarrow +\infty $;
%\item $ n/2^{k_n} \rightarrow + \infty $ and $ k_n \rightarrow +\infty $; 
\item $ n/\texttt{maxnodes} \rightarrow + \infty $ with $ \mathbb{P}_{\Theta} $-probability one;
and
\item $ \min_{\bt}\text{MDI}(X_j; \bt) \rightarrow + \infty $ with $ \mathbb{P}_{\Theta} $-probability one for all $ j \in \calS $.
\end{enumerate}
Then $ \mathbb{E}_{\calD_n}\left[{\int |f(\bx) - \widehat Y(\bx)|^2 \mathbb{P}_{\bX}(d\bx)}\right] \rightarrow 0 $ as $ n \rightarrow + \infty $.
\end{theorem}

\begin{remark}
If additionally the regression function satisfies the assumptions of \prettyref{thm:lambdalower}, then $ \min_{\bt}\text{MDI}(X_j; \bt) \geq \Lambda_j \min_{\bt}K_j(\bt) $, where $ \Lambda_j > 0 $. By the previous discussion, $ \mathbb{P}_{\Theta}[j_{\bt'}(\Theta) = j] \approx \frac{1}{S} > 0 $ for each nonterminal node $ \bt' $ and hence $ \min_{\bt}K_j(\bt) $ goes to infinity with the tree depth $ \mathbb{P}_{\Theta} $-almost surely.
\end{remark}

\begin{remark}
When coupled with a study of the variance of the forest, our theory for the bias of individual trees (\prettyref{thm:main} and \prettyref{thm:lambdalower}) enables one to determine the quality of convergence of $ \widehat Y $ as a function of the parameters of the random forest, e.g., sample size, dimension, sparsity level, and depth to which the individual trees are grown. The analysis also reveals a \emph{local} bias-variance tradeoff, which highlights the local adaptivity of random forests. However, we shall choose not to pursue this analysis in the present paper and leave it for future work.
\end{remark}
\section{Finite sample analysis} \label{sec:finite}

From the perspective of studying the theoretical properties of decision trees, it is desirable if the splits do not separate a small fraction of data points from the rest of the sample. That is, the node counts $ N(\bt) $ should be large enough to contain enough data points so that local estimation valid. On the other hand, the node sizes should be small enough to identify local changes in the regression surface. This is true, for example, if the splitting criterion encourages splits that are performed away from the parent subnode edges. Perhaps the earliest mention of such a condition is \cite[Section 12.2]{breiman1984}, where, en route to establishing asymptotic consistency, it is assumed that the proportion of data points (from a learning sample of size $ n $) in each terminal node is at least $ k_n n^{-1}\log n $ for some sequence $ k_n $ tending to infinity and that the diameters of the terminal nodes converge to zero in probability. A similar assumption is explicitly made in the analysis of \citep{meinshausen2006quantile} and \citep{wager2015} to ensure that terminal node diameters of the forest tend to zero as the sample size tends to infinity, which as mentioned earlier, is a necessary condition to prove the consistency of partitioning estimates \cite{stone1977consistent}, \citep[Chapter 4]{gyorfi2002distribution}. This property is also satisfied by $ q $ quantile forests \citep{scornet2016asymptotics}, where each split contains at least a fraction $ q \in (0, 1) $ of the observations falling into the parent node. Furthermore, in a Bayesian setting, comparable regularity conditions for partitions have been assumed for theoretical analysis of Bayesian random forests (e.g., BART \citep{chipman2010bart}) \citep[Definition 3.1]{rockova2017posterior}, \citep[Definition 2.4]{van2017bayesian}, \citep[Definition 6.1]{rockova2018theory} (e.g., in so-called median or $ k $-$d $ tree partitions \citep{bentley1975multidimensional}, each split roughly halves the number of data points inside the node).

From the perspective of adaptive estimation, forcing the recursive partitions to artificially separate a fixed fraction of the data points at each step may be undesirable. Indeed, \citep{ishwaran2015effect} argues that CART posses a desirable trait of splitting near the edges along noisy variables. This perspective, together with the fact that one does not know a priori which variables are important, leads one to the conclude that it is undesirable to sacrifice the data-dependent nature of the split criterion in order to ensure that a technical condition is met. There are currently no results stating that splits in CART are performed away from the parent subnode edges. However, our results show that the standard assumptions for ``valid partitions'' (for example, see \citep[Assumption 3]{meinshausen2006quantile}, \citep[Definition 1]{wager2015}, or \citep[Definition 4]{wager2018estimation}) are satisfied for CART.

In the next subsection, we show how one can go from bounds on the conditional probability content of an optimally split subnode to bounds on the distance of an optimal split to its parent subnode edges.

\subsection{From probabilities to distances} \label{sec:distance}

Let $ Q(p) $ be the quantile function of the probability measure on $ [a(\bt), b(\bt)] $ with distribution function $ \prob{X \leq s \mid \bX \in \bt } $, i.e., $ Q(p) = \inf\{ s\in [a(\bt),b(\bt)] : p \leq \prob{X \leq s \mid \bX \in \bt}\} $, $ p \in [0, 1] $. Suppose for the moment that $ \bX \sim \Unif([0, 1]^d) $. Then $ \prob{X \leq s \mid \bX \in \bt} = \frac{s-a(\bt)}{b(\bt)-a(\bt)} $ and hence $ Q(p) = a(\bt) + p(b(\bt)-a(\bt)) $. From this, one can deduce that $ s^* $ lies between $ a(\bt) + \frac{b(\bt)-a(\bt)}{2}\Psi $ and $ b(\bt) - \frac{b(\bt)-a(\bt)}{2}\Psi $, where $ \Psi = 1 - \sqrt{1-\Lambda} $. This allows us to go from probability estimates to distance estimates. Inspired by this observation, we impose a general condition on the form of the quantile function $ Q(p) $ so that similar conclusions can be made for other input distributions. The assumption can also be interpreted as a regularity condition on $ Q(p) $. Roughly, it says that $ Q(p) $ cannot be too ``flat'' when $ p $ is near zero or one.
\begin{assumption} \label{ass:Q}
There exist two increasing bijections $ q_1 : [0, 1] \to [0, 1] $ and $ q_2 : [0, 1] \to [0, 1] $ such that for all nodes $ \bt $,
$$
q_1(p) \leq \frac{Q(p)-a(\bt)}{b(\bt)-a(\bt)} \leq q_2(p), \quad p \in [0, 1].
$$
\end{assumption}
Note that $ q_1 $ and $ q_2 $ are both necessarily continuous. This device allows us to bound the distance of $ s^* $ to the edges $ a(\bt) $ and $ b(\bt) $ in terms of the global balancedness $ \Lambda $. To see this, note that by \prettyref{def:edgegap}, $ 4P(\bt^*_L)P(\bt^*_R) \geq \Lambda $ and hence both $ P(\bt^*_L) $ and $ P(\bt^*_R) $ are between $ \frac{1}{2}(1 \pm \sqrt{1-\Lambda}) $. Let $ \Psi =  1 - \sqrt{1-\Lambda} \in [0, 1] $ so that, by \prettyref{eq:s}, 
\begin{equation} \label{eq:quantile}
Q\left(\Psi/2\right) \leq s^* \leq Q\left( 1- \Psi/2 \right).
%\lim_{p \downarrow 1-\frac{\Lambda}{2}}Q\left(p\right).
\end{equation}

%\begin{corollary} \label{cor:opt}
%Both $ P(\bt^*_L) $ and $ P(\bt^*_R) $ are between
%\begin{equation} \label{eq:main}
%\left|P(\bt^*_L)-\frac{1}{2}\right| \leq \frac{1-\Lambda}{2} \quad and \quad \left|P(\bt^*_R)-\frac{1}{2}\right| \leq \frac{1-\Lambda}{2}
%\end{equation}
%and hence
%\begin{equation} \label{eq:main}
%\frac{1}{2}(1 \pm \sqrt{1-\Lambda}).
%\text{and} \label{eq:main} \\ 
%& \hspace{-2cm} 4P(\bt^*_L)P(\bt^*_R) \geq \frac{\lambda^2}{1+\lambda^2}. \label{eq:balance}
%\end{equation}
%Furthermore, let $ Q(p) $ be the quantile function of the probability measure on $ [a(\bt), b(\bt)] $ with distribution function $ \prob{X \leq s \mid \bX \in \bt } $, i.e., $ Q(p) = \inf\{ s\in [a(\bt),b(\bt)] : p \leq \prob{X \leq s \mid \bX \in \bt}\} $, $ p \in [0, 1] $. Let $ \Psi =  1 - \sqrt{1-\Lambda} \in [0, 1] $. Then,
%\begin{equation} \label{eq:quantile}
%Q\left(\Psi/2\right) \leq s^* \leq Q\left( 1- \Psi/2 \right).
%\lim_{p \downarrow 1-\frac{\Lambda}{2}}Q\left(p\right).
%\end{equation}
%\end{corollary}

%\begin{proof}
%The inequalities in \prettyref{eq:main} are direct consequences of \prettyref{def:edgegap}, since $ 4P(\bt^*_L)P(\bt^*_R) \geq \Lambda $. The bounds on $ s^* $ in \prettyref{eq:quantile} follow from \prettyref{eq:main} and \prettyref{eq:s}.
%\end{proof}

%\begin{remark}
%As $ \Lambda $ decreases, the bounds \prettyref{eq:main} on $ P(\bt^*_L) $ and $ P(\bt^*_R) $ become less and less useful. At the very extreme, when $ \Lambda \rightarrow 0 $, \prettyref{eq:quantile} implies that $ a(\bt) \leq s^* \leq b(\bt) $, which does not contain any useful information.
%\end{remark}
Finally, using \prettyref{eq:quantile}, we have that
\begin{equation*}
(b(\bt)-a(\bt))q_1(\Psi/2) + a(\bt) \leq s^* \leq b(\bt) - (b(\bt)-a(\bt))(1-q_2(1-\Psi/2)),
\end{equation*}
and, by rearranging terms,
\begin{equation} \label{eq:length}
\left| s^* - \frac{a(\bt)+b(\bt)}{2} \right| \leq \frac{b(\bt)-a(\bt)}{2}(1-\Gamma),
\end{equation}
where $ \Gamma = 2\min\{q_1(\Psi/2), 1-q_2(1-\Psi/2) \} $. These bounds show how far (in terms of distance) any optimal split is from the parent subnode edges. The inequality in \prettyref{eq:length} bounds the distance between any optimal split and the midpoint of the parent subnode.

We now provide some examples of joint distributions that satisfy \prettyref{ass:Q}. The proofs of each are given in  \prettyref{app:proofs}. 

\begin{enumerate}[(a)]
%\begin{equation*}
%\left|s^*-\frac{a+b}{2}\right| \leq \frac{b-a}{2}(1-\Lambda).
%a + (b-a)\frac{\Lambda}{2} \leq s^* \leq b - (b-a)\frac{\Lambda}{2}.
%\end{equation*}
\item {\it Independent joint density bounded from above and below.} If each $ X $ is i.i.d. $ \Unif(0, 1) \sim \Beta(1, 1) $, then $ Q(p) = a(\bt) + p(b(\bt)-a(\bt)) $ and $ q_1(p) = q_2(p) = p $.
\item {\it Independent joint density unbounded from below.} If each $ X $ is i.i.d. $ \Beta(2, 1) $, then $ Q(p) = \sqrt{a^2(\bt) + p(b^2(\bt)-a^2(\bt))} $, $ q_1(p) = p $, and $ q_2(p) = \sqrt{p} $.
%\item {\it Dependent joint density unbounded from below.} Suppose $ d = 2 $ and let $ p_{\bX}(x_1, x_2) = x_1+x_2 $ and $ \bt = [a_1, b_1] \times [a_2,b_2] $. Then $ \prob{X_1\leq s \mid \bX \in \bt} = \frac{(s-a_1)(s+a_1+b_2+a_2)}{(b_1-a_1)(b_1+a_1+b_2+a_2)} $, $ q_1(p) = p $, and $ q_2(p) = \sqrt{p} $.
\item {\it Independent joint density unbounded from above.} If each $ X $ is i.i.d. $ \Beta(1/2, 1) $, then $ Q(p) = (\sqrt{a(\bt)} + p(\sqrt{b(\bt)}-\sqrt{a(\bt)}) )^2 $, $ q_1(p) = p^2 $, and $ q_2(p) = p $.
%\item {\it Dependent joint density unbounded from above.} Suppose $ d = 2 $ and let $ p_{\bX}(x_1, x_2) = \frac{1}{4}(x^{-1/2}_1+x^{-1/2}_2) $ and $ \bt = [a_1, b_1] \times [a_2,b_2] $. Then $ \prob{X_1\leq s \mid \bX \in \bt} = \frac{(\sqrt{s}-\sqrt{a_1})\left(\sqrt{s}+\sqrt{a_1}+\sqrt{b_2}+\sqrt{a_2}\right)}{(\sqrt{b_1}-\sqrt{a_1})\left(\sqrt{b_1}+\sqrt{a_1}+\sqrt{b_2}+\sqrt{a_2}\right)} $, $ q_1(p) = p^2 $, and $ q_2(p) = p $.
\end{enumerate}

%\begin{assumption}
%For each node $ \bt $ and split $ s $,
%\begin{align}
%& \prob{X \leq s \mid \bX_{\calS} \in t_{\calS}} \geq C\prob{X \leq s \mid \bX \in \bt}\;\text{and} \\
%&  \prob{X > s \mid \bX_{\calS} \in t_{\calS}} \geq C\prob{X > s \mid \bX \in \bt},
%\end{align}
%for a universal constant $ C>0 $.
%\end{assumption}
%Note that
%$$
%\Delta(s^* ; \bt) - \Delta(\hat s ; \bt) \geq 0, \quad \widehat\Delta(s^* ; \bt) - \Delta(\hat s ; \bt) \geq 0
%$$
%Next, note that
%\begin{align*}
%\Delta(s^* ; \bt) - \Delta(\hat s ; \bt) & = [\Delta(s^* ; \bt)-\widehat\Delta(\hat s ; \bt)] + [\widehat\Delta(\hat s ; \bt) - \Delta(\hat s ; \bt)] \\
%& \leq [\Delta(s^* ; \bt)-\widehat\Delta(s^* ; \bt)] + [\widehat\Delta(\hat s ; \bt) - \Delta(\hat s ; \bt)] \\
%& \leq  2\sup_{s\in[a,b]}|\Delta(s ; \bt)-\widehat\Delta(s ; \bt)|
%\end{align*}
%Thus, it follows that
%$$
%|\Delta(s^* ; \bt) - \Delta(\hat s ; \bt)| \leq 2\sup_{s\in[a,b]}|\Delta(s ; \bt)-\widehat\Delta(s ; \bt)|.
%$$
%Next, by a Taylor expansion of $ \Delta(s ; \bt) $ at $ s = s^* $, we have
%$$
%\Delta(s ; \bt) = \Delta(s^* ; \bt) + (s-s^*)\Delta'(s^*; \bt) + \frac{(s-s^*)^2}{2}\Delta''(s_0; \bt),
%$$
%for some $ s_0 $ between $ s $ and $ s^* $. Since $ \Delta'(s^*; \bt) = 0 $, we have,
%$$
%|\Delta(s^* ; \bt) - \Delta(\hat s ; \bt)| = \frac{(\hat s-s^*)^2}{2}|\Delta''(s_0; \bt)|.
%$$
%It thus follows that
%$$
%(\hat s-s^*)^2 \leq \frac{4\sup_{s\in[a,b]}|\Delta(s ; \bt)-\widehat\Delta(s ; \bt)|}{|\Delta''(s_0; \bt)|}.
%$$
\subsection{Number of observations in nodes}
In this section, 
%we describe how one can reach similar conclusions by splitting according to the empirical CART criterion. In particular, 
we give lower bound on (a) the number of observations contained in a daughter node from an optimally split parent node using the finite sample criterion $ \widehat\Delta(\cdot; \bt) $ (denoted by $ N(\hat \bt_L) $ and $ N(\hat \bt_R) $) and (b) the distance of the optimal split to the edges of its parent node.\footnote{The terminal node counts $ N(\hat \bt_L) $ and $ N(\hat \bt_R) $ are also called \texttt{nodesize} in the R package \texttt{randomForest}.} Throughout this section, we implicitly assume that \prettyref{ass:density} and \prettyref{ass:Q} hold for each of the covariates. For simplicity, we assume that $ \calS $ is equal to the full set of $ d $ variables, although under \prettyref{ass:indep}, one can also develop lower bounds for the number of observations that land in nodes along informative directions in $ \calS $, i.e.,
$ N_{\calS} = \sum_{i=1}^n \indc{\bX_{i\calS} \in \bt_{\calS}} $, where $ \bx_{\calS} = (x_j : j\in \calS) $ and $ \bt_{\calS} = \{ \bx_{\calS} : \bx \in \bt \} $.
% (with analogous definitions for $ N_{L\calS} $ and $ N_{R\calS} $). 
This is a quantity of interest because, from the perspective of estimation, one could still have consistency even if $ N_{\calS^c} = 0 $.

In what follows, we define $ \Psi = 1-\sqrt{1-\Lambda} $, $ \Gamma = 2\min\{q_1(\Psi/2), 1-q_2(1-\Psi/2) \} $, $ p_L = q^{-1}_2(\Gamma^2/2) $, $ p_R = 1-q^{-1}_1(1-\Gamma^2/2) $, and $ p = \min\{p_L, p_R \} $. For example, when $ \bX \sim \Unif([0, 1]^d) $, we have $ \Gamma = \Psi $ and $ p_L = p_R = \Psi^2/2 $. 

For the next set of results, we let $ \mathscr{S}^* = \argmax_s \Delta(s ; \bt) $ and $ \hat s \in \argmax_s \widehat\Delta(s ; \bt) $. Due to space constraints, we prove both theorems in \prettyref{app:proofs}.
%This implies that
%$$
%\left|s^* - \frac{a+b}{2} \right| \leq \frac{b-a}{2}(1-q(\Lambda))
%$$

\begin{theorem} \label{thm:sample}
Suppose $ N(\bt) $ is large enough so that, given $ N(\bt) $ and $ \bt $, with probability at least $ 1-\delta $,
\begin{equation} \label{eq:close}
%\left|\hat s- s^*\right| \leq \frac{b-a}{2}\Gamma(1-\Gamma).
\text{dist}(\hat s,\; \mathscr{S}^*) \leq \frac{b(\bt)-a(\bt)}{2}\Gamma(1-\Gamma).
\end{equation}
Then with probability at least $ 1-\delta $,
$$
\frac{b(\bt) - a(\bt)}{2}\Gamma^2 + a(\bt) \leq \hat s \leq b(\bt) -  \frac{ b(\bt) - a(\bt)}{2}\Gamma^2.
$$
If $ \bt $ is independent of the training data $ \calD_n $, then, given $ N(\bt) $ and $ \bt $, with probability at least $ 1-\delta -2\exp\{-Np^2/2\} $,
%$$ \widehat p (\hat \bt_L) \geq \frac{p_L}{2} \quad \text{and} \quad \widehat p (\hat \bt_R) \geq \frac{p_R}{2}. $$
$$N(\hat \bt_L) \geq N(\bt) \frac{p_L}{2} \quad \text{and} \quad N(\hat \bt_R) \geq N(\bt) \frac{p_R}{2}. $$
\end{theorem}

\begin{remark}
The assumption \prettyref{eq:close} can be recast by saying that the distance between $ s^* $ and $ \hat s $ is less than a constant multiple, namely $ \frac{1}{2}\Gamma(1-\Gamma) $, of the length of the parent subnode. (Note that by definition, $ \text{dist}(\hat s,\; \mathscr{S}^*) \leq b(\bt)-a(\bt) $.)
%$ |\hat s - s^*| \leq b-a $.] 
Such an assumption is not unrealistic since if $ \Delta(\cdot ; \bt) $ has a unique global maximum, i.e., $ s^* $ is unique, \citep[Theorem 2]{ishwaran2015effect} shows $ \text{dist}(\hat s,\; \mathscr{S}^*) $ converges weakly to zero. In fact, with similar assumptions, one can characterize the rate of convergence \citep[Theorem 3.2.5]{wellner1996weak}. Indeed, \citep{banerjee2007confidence, buhlmann2002analyzing} show cube root asymptotics (i.e., $ n^{1/3}(\hat s - s^*) $ converges in distribution) of split points for one-level decision trees (e.g., decision stumps) using the CART sum of squares error criterion. The work of \citep[Section 3.4.2]{buhlmann2002analyzing} also extends these rates to multi-level decision trees, which is particularly relevant to our setting.
\end{remark}

%The previous theorem simplifies if $ \bX $ is uniform, since $ q_1(p) = q_2(p) = p $. We state it next as a corollary.
%\begin{corollary}
%Suppose $ \bX $ is uniformly distributed on $ [0, 1]^d $ and that $ N$ is large enough so that, given $ N$ and $ \bt $, with probability at least $ 1-\delta $,
%\begin{equation}
%\left|\hat s- s^*\right| \leq \frac{b-a}{2}\Lambda(1-\Lambda) \sim \frac{b-a}{4\Lambda^2}, \quad \Lambda \rightarrow + \infty.
%\text{dist}(\hat s,\; \mathscr{S}^*) \leq \frac{b-a}{2}\Lambda(1-\Lambda) \sim \frac{b-a}{4\Lambda^2}, \quad \Lambda \rightarrow + \infty.
%\end{equation}
%If $ \bt $ is independent of the input data $ \{\bX_i \}_{i=1}^n $, then, given $ N$ and $ \bt $, with probability at least $ 1-\delta -2\exp\{-N\Lambda^2/4\} $,
%$$ \widehat p(\hat \bt_L) \geq \frac{\Lambda^2}{4} \sim \frac{1}{16\Lambda^4} \quad \text{and} \quad \widehat p(\hat \bt_R) \geq \frac{\Lambda^2}{4} \sim \frac{1}{16\Lambda^4}, \quad \Lambda \rightarrow + \infty.  $$
%\end{corollary}

In practice, $ a(\bt) $, $ b(\bt) $, and $ \bt $ all depend on the data $ \calD_n $. We now state a refinement of \prettyref{thm:sample} that allows for data-dependent splits. Essentially it says that if the optimal empirical and population split points are sufficiently close to each other and the fraction of data points contained in the parent node is at least $ \alpha $, then the fraction of data points contained in each daughter node is at least $ \alpha \beta $, where $ \beta $ is strictly positive and depends only on $ q_1 $, $ q_2 $, and $ \Lambda $. In fact, the next theorem shows that one can take $ \beta = \min\{p_L, p_R\}/2 $ with high probability. This ensures that the daughter nodes contain a sufficiently large number of training samples, so that subsequent empirical splits will be close to their infinite sample counterparts, and so on and so forth.
%Here, we let $ s^* $ denote the optimizer of $ s \mapsto \Delta(s ; \bt) $ and $ \hat s $ denote the optimizer of $ s \mapsto \widehat \Delta(s ; \bt) $.

\begin{theorem} \label{thm:sampledata}
Let $ \alpha > 0 $ and suppose $ n $ is large enough so that with probability at least $ 1-\delta $,
\begin{equation}
%\left|\hat s- s^*\right| \leq \frac{\widehat b- \widehat a}{2}\Gamma(1-\Gamma) \quad \text{and} \quad N\geq n\alpha.
\text{dist}(\hat s,\; \mathscr{S}^*) \leq \frac{ b(\bt)- a(\bt)}{2}\Gamma(1-\Gamma) \quad \text{and} \quad N(\bt) \geq n\alpha.
\end{equation}
Then with probability at least $ 1-\delta $,
$$
\frac{b(\bt) - a(\bt)}{2}\Gamma^2 + a(\bt) \leq \hat s \leq b(\bt) -  \frac{ b(\bt) - a(\bt)}{2}\Gamma^2,
$$
and with probability at least $ 1-\delta-16(n^{2d}+1)\exp\{-n \alpha^2 p^2/512 \} $,
%$$ \widehat p (\hat \bt_L) \geq \frac{p_L}{2} \quad \text{and} \quad \widehat p (\hat \bt_R) \geq \frac{p_R}{2}. $$
$$ N(\hat \bt_L) \geq N(\bt) \frac{p_L}{2} \quad \text{and} \quad N(\hat \bt_R) \geq N(\bt) \frac{p_R}{2}. $$
\end{theorem}

The previous theorem can be used inductively to show that if the tree is grown to a depth of $ K $, then with high probability, each terminal subnode length is at most $ (1-\Gamma^2/2)^K $. An obvious line of future work would be to make this statement more rigorous and use it to furnish a convergence result for ensembles of decision trees with empirical splits.

\section{Extension to classification trees} \label{sec:classification}

Our results have focused on regression trees, although nearly identical bounds can be developed for classification trees. In the binary classification context, i.e., $ Y \in \{ -1, +1 \} $, the variance impurity \prettyref{eq:impurity} equals $ \Delta(\bt) = 4\times\frac{\#\{Y_i = +1 : \bX_i \in \bt\}}{N(\bt)} \times \frac{\#\{Y_i = -1 : \bX_i \in \bt\}}{N(\bt)} $, which is also known as \emph{Gini} impurity. Here, instead of averages, the tree outputs the majority vote at a terminal node $ \bt $, i.e., $ \widehat Y = +1 $ if $ \#\{Y_i = +1 : \bX_i \in \bt\} \geq \#\{Y_i = -1 : \bX_i \in \bt\} $ and $ \widehat Y = -1 $ otherwise.

The infinite sample Gini impurity is $$ \Delta(\bt) = 4\prob{Y=+1 \mid \bX \in \bt}\prob{Y = -1 \mid \bX \in \bt}. $$ The analog to the conditional partial dependence function \prettyref{eq:partialdef} is $$ \overline F(x; \bt) = 2\prob{Y=+1 \mid \bX \in \bt,\; X = x} $$ and its mean-centered version is $$ \overline G(x; \bt) = 2\prob{Y=+1 \mid \bX \in \bt,\; X = x}- 2\prob{Y=+1 \mid \bX \in \bt}. $$ In agreement with \prettyref{eq:simple0}, $ \Delta(s; \bt) $ also has the representation 
$$
4P(\bt_L)P(\bt_R)[\prob{Y=+1 \mid \bX \in \bt, \; X \leq s} - \prob{Y = +1 \mid \bX \in \bt, \; X > s}]^2. 
$$
Thus, adapting the proofs of \prettyref{thm:mainprob} and \prettyref{thm:second}---whose original proofs also rely on such a representation for $ \Delta(\cdot; \bt) $---it can easily be shown that \prettyref{thm:main} and its consequences hold verbatim with these modified definitions of $ \overline F(\cdot; \bt) $, $ \overline G(\cdot; \bt) $, and $ \Delta(\cdot; \bt) $. In particular, \prettyref{thm:lambdalower} holds if $ \prob{Y=+1\mid \bX = \bx} $ satisfies \prettyref{eq:flat}. As a concrete example (c.f., \prettyref{ex:poly} and \prettyref{ex:sine}), consider the following logistic regression model. We give the proof in \prettyref{app:proofs}.
\begin{example}\label{ex:logistic}
Let $ \prob{Y = +1 \mid \bX = \bx} = (1+e^{-\beta_0-\langle \bx,\;\bbeta\rangle })^{-1} $ with intercept coefficient $ \beta_0 $ and effects coefficients $ \bbeta = (\beta_1, \beta_2, \dots, \beta_d) $ and $ \bX \sim \Unif([0,1]^d) $. Suppose $ j \in \calS $ so that $ \beta_j \neq 0 $. Then,
$$
\Lambda_j \geq \min\{ |\beta_j|^{-4/3}, (1/8)^{2/3} \},
$$
and hence
$$
\text{MDI}(X_j; \bt) \geq K_j(\bt)\min\{ |\beta_j|^{-4/3}, (1/8)^{2/3} \}.
$$
\end{example}
Curiously, this lower bound does not depend on any of the parameters other than $ \beta_j $. The importance of $ X_j $ decreases as we become increasingly more certain that either $ Y = + 1 $ or $ Y = - 1 $, i.e., as the parameter $ \beta_j $ or $ -\beta_j $ grows, respectively.

In the next section, we provide proofs of the main results from \prettyref{sec:results}---\prettyref{thm:main} and \prettyref{thm:lambdalower}.

\section{Proofs and additional results} \label{sec:proofs}
%One can distill condition \prettyref{eq:flat} into the following condition on the partial dependence function.
%\begin{assumption} \label{ass:partial}
%For each node $ \bt $, if $ \overline F_j(x_j ; \bt) $ is constant in $ x_j $ on a subnode, then $ j \in \calS^c $.
%\end{assumption}

%\begin{assumption} \label{ass:ass3}
%If $ \expect{Y \mid \bX = \bx} $ is constant in $ x_j $ on any subnode, then $ j \in \calS^c $.
%\end{assumption}

%\begin{lemma} \label{lmm:delta}
%Suppose \prettyref{ass:density} and $ \Delta(s^* ; \bt) > 0 $ hold. If $ \Delta(j, s; \bt) = 0 $ for all splits $ s $ in a subnode, then $ j \in \calS^c $.
%\end{lemma}

Throughout this section, for notational clarity and brevity, we omit dependence on the input variable index $ j $ for all quantities and assume that we are splitting on a generic coordinate $ X $. We also sometimes omit dependence on $ \bt $, and substitute $ a(\bt) $ and $ b(\bt) $ with $ a $ and $ b $, respectively.

The first theorem, from which most of the results in the paper are derived, gives a clean expression for the optimally split daughter node probabilities in terms of the partial dependence function and largest reduction in impurity.
\begin{theorem} \label{thm:mainprob} Suppose \prettyref{ass:density} holds and $ \Delta(s^* ; \bt) > 0 $. Then
\begin{equation} \label{eq:universal}
\prob{X \leq s^* \mid \bX \in \bt} = \frac{1}{2}\left(1 \pm \sqrt{\frac{|\overline G(s^*; \bt)|^2}{|\overline G(s^*; \bt)|^2+\Delta(s^* ; \bt)}} \right),
\end{equation}
and consequently,
\begin{equation} \label{eq:lambdauniversal}
\lambda(\bt) = \frac{\Delta(s^* ; \bt)}{|\overline G(s^*; \bt)|^2+\Delta(s^* ; \bt)}.
\end{equation}
\end{theorem}

\begin{proof}
Recall from \prettyref{eq:simple0} that one can write
\begin{equation} \label{eq:simple}
%\Delta(s ; \bt) = \frac{P(\bt_L)}{P(\bt_R)}\left[\expect{Y \mid \bX \in \bt, \; X \leq s} - \expect{Y \mid \bX \in \bt}\right]^2.
\Delta(s ; \bt) = P(\bt_L)P(\bt_R)\left[\expect{Y \mid \bX \in \bt, \; X \leq s} - \expect{Y \mid \bX \in \bt, \; X > s}\right]^2.
\end{equation}
%If $ \bX $ is uniform on $ [0, 1]^d $ and $ \bt_L = [a,s] $ and $ \bt_R = (s, b] $,
%and $ f(\bx) = \sum_{j=1}^d f_j(x) $ is an additive regression function
%\prettyref{eq:simple} admits an even simpler form:
%\begin{equation*}
%\Delta(s ; \bt) = \frac{(s-a)(b-s)}{(b-a)^2}\left[\expect{Y \mid \bX \in \bt, \; X \leq s} - \expect{Y \mid \bX \in \bt, \; X > s}\right]^2.
%\end{equation*}
Next, define
$$
\Xi(s; \bt) =  P(\bt_L)P(\bt_R)\left[\expect{Y \mid \bX \in \bt, \; X \leq s} - \expect{Y \mid \bX \in \bt, \; X > s}\right],
$$
so that
\begin{equation} \label{eq:objective}
\Delta(s ; \bt) = \frac{[\Xi(s; \bt)]^2}{P(\bt_L)P(\bt_R)}.
\end{equation}
An easy calculation shows that
\begin{align} \label{eq:derivative}
\Xi'(s) & = \frac{\partial}{\partial s}\Xi(s; \bt) = p(\bt_L)[\expect{Y \mid \bX\in \bt, \; X = s} - \expect{Y \mid \bX \in \bt}] \nonumber \\ & = p(\bt_L)\overline G(s; \bt).
\end{align}

% is the density function of the probability measure $ \mathbb{P}_{X \mid \bX \in \bt}(ds) $ and $ \overline G(s; \bt) = \expect{Y \mid \bX\in \bt, \; X = s} - \expect{Y \mid \bX \in \bt} $. 

%We define $ \overline F(s; \bt) = \expect{Y \mid \bX\in \bt, \; X_j = s} $ as the conditional expectation of $ Y $ with respect to the probability measure $ \mathbb{P}_{\bX_{\setminus j} \mid \bX \in \bt, \; X_j = s_j}(d\bs_{\setminus j}) = \frac{\mathbb{P}_{\bX\mid \bX \in \bt}(d\bs_{\setminus j})}{\frac{d}{ds_j}\prob{X \leq s_j \mid \bX \in \bt}} $
%$ \prob{\bX \leq \bs \mid \bX \in \bt, \, X = s} = \frac{\frac{\partial}{\partial s}\prob{\bX \leq \bs, \; X \leq s}}{\frac{\partial}{\partial s}\prob{\bX \in \bt, \;X \leq s}}  $ 
%defined on the $ d-1$ dimensional subspace of all coordinates except for $ X_j $. This is also the conditional distribution function of $ \bX $ given $ \bX \in \bt $ and $ X_j = s $. We henceforth refer to it as the \emph{partial dependence function} within node $ \bt $.
Taking the derivative of $ \Delta(s ; \bt) $ with respect to $ s $, we find that
\begin{equation} \label{eq:firstderivative}
\frac{\partial}{\partial s}\Delta(s ; \bt) = \frac{\Xi(s; \bt)p(\bt_L)[2P(\bt_L)P(\bt_R)\overline G(s; \bt) - \Xi(s; \bt)(1-2P(\bt_L))]}{[P(\bt_L)P(\bt_R)]^2}.
\end{equation}
Suppose $ s^* $ is a global maximizer of \prettyref{eq:objective} (in general, it need not be unique). Then a necessary condition (first-order optimality condition) is that the derivative of $ \Delta(s ; \bt) $ is zero at $ s^* $. That is, from \prettyref{eq:firstderivative}, $ s^* $ satisfies
\begin{equation} \label{eq:derivativezero}
\Xi(s^*; \bt)p(\bt^*_L)[2P(\bt^*_L)P(\bt^*_R)\overline G(s^*; \bt) - \Xi(s^*; \bt)(1-2P(\bt^*_L))] = 0.
\end{equation}
If $ p(\bt^*_L) > 0 $ and $ \Delta(s^* ; \bt) > 0 $, it follows from rearranging \prettyref{eq:derivativezero} that
\begin{equation} \label{eq:sol}
P(\bt^*_L) = \frac{1}{2} - \frac{\text{sgn}(\Xi(s^*; \bt))\overline G(s^*; \bt)}{\sqrt{\Delta(s^* ; \bt)}}\sqrt{P(\bt^*_L)P(\bt^*_R)}.
\end{equation}

%The reason for the form in which \prettyref{eq:sol} is written will be made apparent shortly. 
This expresses $ P(\bt^*_L)$ as a fixed point of the mapping 
$$ p \mapsto \frac{1}{2} - \frac{\text{sgn}(\Xi\circ Q)(p))(\overline G\circ Q)(p)}{\sqrt{(\Delta \circ Q)(p)}}\sqrt{p(1-p)}, \quad 0 < p < 1, $$ 
where $ \circ $ denotes the composition operator and $ Q $ denotes the quantile function of the probability measure with distribution function $ \prob{X \leq s \mid \bX \in \bt} $, i.e., 
$$ Q(p) = \inf\{ s\in [a,b] : p \leq \prob{X \leq s \mid \bX \in \bt}\}. $$
%Determination of $ s^* $ through analytical means is challenging since the map can be highly non-convex. Nevertheless, we will show that it is possible to derive useful and meaningful bounds on $ s^* $ that quantify its distance from the edges $ a(\bt) $ and $ b(\bt) $.

The solution to \prettyref{eq:sol} is obtained by solving a simple quadratic equation of the form $ p = 1/2 \pm c\sqrt{p(1-p)} $, which proves the identity \prettyref{eq:universal}. The identity for the node balancedness \prettyref{eq:lambdauniversal} follows immediately from \prettyref{def:nodebalance} and \prettyref{eq:universal}.
%\begin{align*}
%P(\bt^*_L) & = \begin{cases} \frac{1}{2}\left(1 - \sqrt{\frac{|\overline G(s^*; \bt)|^2}{|\overline G(s^*; \bt)|^2+\Delta(s^* ; \bt)}} \right), \quad \text{if} \; s^* < \text{median}(X \mid \bX \in \bt) & \\ \frac{1}{2}\left(1 + \sqrt{\frac{|\overline G(s^*; \bt)|^2}{|\overline G(s^*; \bt)|^2+\Delta(s^* ; \bt)}} \right),  \quad \text{if} \; s^* \geq \text{median}(X \mid \bX \in \bt). \end{cases}
%\end{align*}
\end{proof}

\begin{remark}
One can make connections between the representation in \prettyref{thm:mainprob} and other quantities defined in the literature. For example, \citep[Section 2.8]{ishwaran2015effect} define the (empirical) edge-cut preference statistic of a split $ s $ as 
$$
%\text{ecp}(s) = \frac{1}{2} - \min\{ P(\bt_L), P(\bt_R) \}  = \frac{|P(\bt_L) - P(\bt_R)|}{2}.
\text{ecp}(s) = \frac{1}{2} - \frac{\min\{ N(\bt) \widehat P(\bt_L)-1, N(\bt) \widehat P(\bt_R)-1 \}}{N(\bt) -1}  = \frac{N(\bt) }{N(\bt) -1}\frac{|\widehat P(\bt_L) - \widehat P(\bt_R)|}{2}.
$$
The population version is $ \frac{|P(\bt_L) - P(\bt_R)|}{2} $, which according to \prettyref{thm:mainprob}, is equal to $ \frac{1}{2}\sqrt{1-\lambda(\bt)} = \frac{1}{2}\sqrt{\frac{|\overline G(s^*; \bt)|^2}{|\overline G(s^*; \bt)|^2+\Delta(s^* ; \bt)}} $ at the optimal split $ s^* $.
\end{remark}

The expression in \prettyref{thm:mainprob} reveals that the optimal split is a perturbation of the median of the conditional distribution $ X \mid \bX \in \bt $, where the gap is governed by the largest decrease in impurity, namely, $ \Delta(s^* ; \bt) $, and the mean-centered partial dependence function, namely, $ \overline G(s^*; \bt) $. At the very extreme, the reduction in weighted variance is smallest when there is no signal in the splitting direction---$ \widehat\Delta(j, s^* ; \bt) \gg \widehat\Delta(j', s^*; \bt) \approx 0 $ for $ j \in \calS $ and $ j' \in \calS^c $. Thus, by \prettyref{thm:mainprob}, splits along directions that contain a strong signal (as opposed to noisy or directions or directions with weak signals) tend to be further away from the parent node subedges \citep{ishwaran2015effect}. In fact, this has been empirically observed for some time \citep[Section 11.8]{breiman1984}, i.e., squared error impurity tends to favor end-cut splits---that is, splits in which the proportion of data contained in an optimally split node is close to zero or one. 
%As we shall see, these two forces work in tandem to determine how close the optimal splits are to the edges, which in turn, govern the bias of the tree. 
%Such a dependence on the smoothness of the regression function can be compared to kernel regression, where the (theoretical) optimal bandwidth depends on the curvature of the regression function \citep{yang1999multivariate}.
The perturbation from the median of $ X \mid \bX \in \bt $ is zero and hence $ P(\bt^*_L) = 1/2 $ when $ \overline G(s^*; \bt) = 0 $, or equivalently, when $ \overline F(s^*; \bt) = \expect{Y \mid \bX \in \bt, \; X = s^*} = \expect{Y \mid \bX \in \bt} $. This is true in the special case that the regression function is linear and the input distribution is uniform, since in this case it can be shown that $ s^* = (a(\bt)+b(\bt))/2 $. Next, we state a more general result for other regression functions.

\begin{example}
Suppose $ \Delta(s^* ; \bt) > 0 $. If $ \bX $ is uniform, then 
%\begin{align*}
%s^* & = \begin{cases} \frac{a+b}{2}-\frac{b-a}{2}\sqrt{\frac{|\overline G(s^*; \bt)|^2}{|\overline G(s^*; \bt)|^2+\Delta(s^* ; \bt)}}, \quad \text{if} \; s^* < (a+b)/2 & \\ \frac{a+b}{2}+\frac{b-a}{2}\sqrt{\frac{|\overline G(s^*; \bt)|^2}{|\overline G(s^*; \bt)|^2+\Delta(s^* ; \bt)}},  \quad \text{if} \; s^* \geq (a+b)/2. \end{cases}
%\end{align*}
$$
s^* = \frac{a(\bt)+b(\bt)}{2}\pm\frac{b(\bt)-a(\bt)}{2}\sqrt{\frac{|\overline G(s^*; \bt)|^2}{|\overline G(s^*; \bt)|^2+\Delta(s^* ; \bt)}}.
$$
\end{example}

We also have the following corollary, which expresses the $ \Prob_{\bX} $-probability of any terminal node $ \bt $ in terms of the largest decrease in impurity and the conditional partial dependence function. Its proof can be deduced from a simple induction argument and \prettyref{thm:mainprob}.

\begin{corollary} \label{cor:forestsplit}
Consider a decision tree with splits determined by optimizing the infinite sample CART objective \prettyref{eq:infinite}. Suppose \prettyref{ass:density} holds and $ \Delta(s^*; \bt') > 0 $ for all nodes $ \bt' $.
%, i.e., $ T $ is a decision tree with splits determined by the infinite sample splitting criterion.
%The $ \Prob_{\bX} $-probability of any terminal node $ \bt $ is at most
%after splitting the $ j^{\Th} $ variable ($ j \in \calS $) $ K_j $ times is between $ \prod_{j=1}^d\left( \Lambda_j/4 \right)^{K_j} $ and $ \prod_{j=1}^d\left( 1- \Lambda_j/4 \right)^{K_j} $.
%The $ \Prob_{\bX} $-probability of any node is at most 
%$$ \exp\left\{ -\frac{1}{12}\sum_{j\in\calS}\text{MDI}(X_j; \bt) \right\}. $$
%or
%$$ \frac{1}{M}\sum_{m=1}^M\exp\left\{ -\frac{1}{4}\sum_{j=1}^d\text{MDI}_m(X_j; \bt)/\omega^2_j(G;[0,1]) \right\}, $$
%where $ \text{MDI}(X_j; \bt) = \sum_{\bt' \supset \bt}\frac{\Delta(s^*_{\bt'}; \bt')}{\omega^2(F_j; \bt')}\indc{j_{\bt'}=j} $ (the sum extends over all ancestor nodes $ \bt' $ of the terminal node $ \bt $) is a measure of the variable importance of $ X_j $ within terminal node $ \bt $.
% $ \text{MDI}(X_j) = \frac{1}{M}\sum_{m=1}^{M}\sum_{t_m:j_{t_m}=j}\Delta(s^*_{t_m}, t_m) $
%or
%$ \text{MDI}_m(X_j) = \sum_{t_m:j_{t_m}=j}\Delta(s^*_{t_m}, t_m) $
%Furthermore, if the features of $ \bX $ are independent, then the $ \Prob_{\bX} $-probability of any subnode along the $ j^{\Th} $ coordinate is between $ (\Lambda_j/4)^{K_j} $ and $ (1-\Lambda_j/4)^{K_j} $. Thus, the diameter of the largest node converges to zero in $ \Prob_{\bX} $-probability when $ K_j\Lambda_j \sim K_j/(4\Lambda_j^2) \rightarrow +\infty $, for each $ j = 1, 2, \dots, d $.
Then the $ \Prob_{\bX} $-probability of any terminal node $ \bt $ is
$$
\Prob_{\bX}[\bX \in \bt] = \prod_{\bt' \supset \bt}\frac{1}{2}\left(1+ \eta_{\bt'}\sqrt{\frac{|\overline G_j(s^*; \bt')|^2}{|\overline G_j(s^*; \bt')|^2+\Delta(j, s^* ; \bt')}}\right),
$$
where the product extends over all ancestor nodes $ \bt' $ of $ \bt $. The value of $ \eta_{\bt'} \in \{-1, +1\} $ is given in \prettyref{tab:eta}.
\begin{table}
\centering
\begin{tabular}{ l | c r }
  $ \eta_{\bt'} $ & daughter node $ \bt' $ & $ s^* < \text{median}(X_j \mid \bX \in \bt') $ \\
\hline
  $ +1 $ & right & yes \\
\hline
  $ +1 $ & left & no \\
\hline
  $ -1 $ & left & yes \\
\hline
  $ -1 $ & right & no \\
\hline
\end{tabular}
\caption{Table showing the values of $ \eta_{\bt'} \in \{-1,+1\} $ from \prettyref{cor:forestsplit}.}
\label{tab:eta}
\end{table}
\end{corollary}

For purposes of obtaining useful upper and lower bounds on $ p(\bt^*_L) $ (and therefore also $ p(\bt^*_R) $), we see from \prettyref{eq:universal} that it suffices to lower bound $ \Delta(s^*, \bt) $ and upper bound $ |\overline G(s^*; \bt)|^2 + \Delta(s^* ; \bt) $. The next lemma shows that $ |\overline G(s^*; \bt)|^2 + \Delta(s^* ; \bt) $ is at most the oscillation of the partial dependence function over the node.
\begin{lemma} \label{lmm:variation} Suppose \prettyref{ass:density} and $ \Delta(s^* ; \bt) > 0 $. Then,
\begin{equation} \label{eq:GDineq}
|\overline G(s^*; \bt)|^2 + \Delta(s^* ; \bt) \leq \sup_{s\in[a(\bt),b(\bt)]}|\overline G(s; \bt)|^2 \leq \omega^2(\overline F(\cdot; \bt); [a,b]). %\leq \text{TV(\overline F(\cdot; \bt); [a,b])}.
\end{equation}
Furthermore, if each first-order partial derivative of the regression function and joint density $ p_{\bX} $ exist and are continuous, then
\begin{equation} \label{eq:TV}
|\overline G(s^*; \bt)|^2 + \Delta(s^* ; \bt) \leq \int_{a(\bt)}^{b(\bt)} |\overline F'(s; \bt)|ds.
\end{equation}
\end{lemma}
\begin{proof}
It can be shown that $ \Delta(s ; \bt) \leq 4P(\bt_L)P(\bt_R)\sup_{s\in[a(\bt),b(\bt)]}|\overline G(s; \bt)|^2 $. By \prettyref{thm:mainprob}, $ 4P(\bt^*_L)P(\bt^*_R) = \frac{\Delta(s^* ; \bt)}{|\overline G(s^*; \bt)|^2 + \Delta(s^* ; \bt)} $. Thus, $ \frac{\Delta(s^* ; \bt)}{\sup_{s\in[a(\bt),b(\bt)]}|\overline G(s; \bt)|^2} \leq \frac{\Delta(s^* ; \bt)}{|\overline G(s^*; \bt)|^2 + \Delta(s^* ; \bt)} $, which is equivalent to the first claimed inequality in \prettyref{eq:GDineq}.
%If $ \bX $ has independent features,
Next, we show that $ \sup_{s\in[a(\bt),b(\bt)]}|\overline G(s; \bt)| \leq \omega(\overline F(\cdot; \bt); [a(\bt),b(\bt)]) $. Since $$ \int_{a(\bt)}^{b(\bt)} \overline F(s; \bt)p(\bt_L)ds = \expect{Y \mid \bX \in \bt}, $$ by the generalized mean value theorem for definite integrals, there exists $ s' \in [a(\bt),b(\bt)] $ such that $ \overline G(s'; \bt) = 0 $. Hence $ \sup_{s\in[a(\bt),b(\bt)]}|\overline G(s; \bt)| $ can also be bounded by the oscillation of the partial dependence function $ \overline F(s; \bt) = \expect{Y \mid \bX\in \bt, \; X = s} $ on $ [a(\bt),b(\bt)] $ since
\begin{align*}
\sup_{s\in[a(\bt),b(\bt)]}|\overline G(s; \bt)| & = \sup_{s\in[a(\bt),b(\bt)]}|\overline G(s; \bt)-\overline G(s'; \bt)| \\ & \leq \sup_{s,s'\in[a(\bt),b(\bt)]}|\overline F(s; \bt) - \overline F(s'; \bt)| \\ & = \omega(\overline F(\cdot; \bt); [a(\bt),b(\bt)]).
%\frac{1}{2}TV(G; [a,b]).
\end{align*}
This proves the inequalities in \prettyref{eq:GDineq}.

To show \prettyref{eq:TV}, note that when $ \overline F(\cdot; \bt) $ is smooth, its oscillation is bounded by its total variation $ \text{TV}(\overline F(\cdot; \bt); [a(\bt),b(\bt)]) = \int_{a(\bt)}^{b(\bt)} |\overline F'(s; \bt)|ds $. This bound is occasionally useful and will be used in the proof of \prettyref{ex:poly}, \prettyref{ex:sine}, and \prettyref{ex:friedman}.
\end{proof}

Combining \prettyref{thm:mainprob} with \prettyref{lmm:variation}, we have the following bound on the conditional daughter node probabilities.

\begin{theorem} \label{thm:first}
Suppose \prettyref{ass:density} and $ \Delta(s^* ; \bt) > 0 $. Then both $ P(\bt^*_L) $ and $ P(\bt^*_R) $ are between
$$
\frac{1}{2}\Bigg(1\pm\sqrt{1-\frac{\Delta(s^* ; \bt)}{\omega^2(\overline F(\cdot; \bt); [a(\bt),b(\bt)])}}\Bigg),
$$
and consequently,
\begin{equation} \label{eq:lambdaosc}
\lambda(\bt) \geq \frac{\Delta(s^* ; \bt)}{\omega^2(\overline F(\cdot; \bt); [a(\bt),b(\bt)])}.
\end{equation}
\end{theorem}

This implies that $ P(\bt_L) $ and $ P(\bt_R) $ tend to be more extreme (i.e., closer to zero or one) if the oscillation of the partial dependence function $ \overline F(\cdot; \bt) $ is large. Indeed, we have seen from  \prettyref{ex:sine} that the optimal split point for a sinusoidal waveform gets closer and closer to its parent subnode edges as the periodicity increases.

Note that to obtain \prettyref{thm:mainprob} and its consequences in \prettyref{thm:first}, we only used the first-order optimality condition for $ s^* $. Next, we show that by additionally incorporating second-order conditions, an alternate (and sometimes better) bound can be obtained. First, we state a lemma.
%\begin{figure} [t] 
%\centering
%  \includegraphics[width=0.5\linewidth]{constant_sample.pdf}
%\caption{Plot of $ \widehat\Delta(\cdot , [0, 1]^2) $ (with $ \Delta(; [0, 1]^2) = 0 $ shown in the dotted line for comparison) for $ f(\bx) = x_1+x_2-2x_1x_2 $, $ n = 100 $, and $ \varepsilon \sim N(0, 1) $. Note that the optimal split $ s^* $ is near the endpoints.
% The decrease in empirical weighted variance is small along either direction, even though $ \calS = \{1, 2\} $.
%}
%\label{fig:examples}
%\end{figure}

\begin{lemma}\label{lmm:secondorder}
%Suppose $ p(\bt^*_L) \neq 0 $. Then,
Suppose \prettyref{ass:density} holds and each first-order partial derivative of the regression function and joint density $ p_{\bX} $ exist and are continuous. Then,
\begin{equation} \label{eq:secondorder}
%|\overline G(s^*; \bt)|^2 + \Delta(s^* ; \bt) + \frac{\partial}{\partial s}|\overline G(s; \bt)|^2 \mid_{s=s^*} \frac{P(\bt^*_L)P(\bt^*_R)}{p(\bt^*_L)(1-2P(\bt^*_L))} < 0.
%1 + \frac{\Delta(s^* ; \bt)}{|\overline G(s^*; \bt)|^2}  + \frac{ \frac{\partial}{\partial s}\log(|\overline G(s; \bt)|^2) \mid_{s=s^*}}{\frac{\partial}{\partial s}\log(P(\bt^*_L)P(\bt^*_R)) \mid_{s=s^*}} < 0.
|\overline G(s^*; \bt)|^2 + \Delta(s^* ; \bt) < \frac{|\overline F'(s^*; \bt)|}{p(\bt^*_L)}\sqrt{P(\bt^*_L)P(\bt^*_R)\Delta(s^* ; \bt)}.
\end{equation}
\end{lemma}

\begin{proof}
If the first-order partial derivative of the joint density $ p_{\bX} $ exists and is continuous, then the density of the conditional probability measure $ \mathbb{P}_{\bX_{\setminus j} \mid \bX \in \bt, \; X_j = x_j}(d\bx_{\setminus j}) $, namely $ \bx_{\setminus j} \mapsto p_{\bX}(x_j, \bx_{\setminus j})/\int_{\bt_{\setminus j}}p_{\bX}(x_j, \bx'_{\setminus j})d\bx'_{\setminus j} $, is continuously differentiable in $ x_j $. If, in addition, the first-order partial derivative of the regression function exists and is continuous, then by Leibniz's integral rule, $ \overline F'(\cdot; \bt) $ exists and is therefore well-defined.
We will show that if $ \frac{\partial}{\partial s}\Delta(s ; \bt) \mid_{s=s^*} = 0 $, then
\begin{equation} \label{eq:secondderivative}
\frac{\partial^2}{\partial s^2}\Delta(s ; \bt) \mid_{s=s^*} = \frac{2(p(\bt^*_L))^2}{P(\bt^*_L)P(\bt^*_R)}\Big(|\overline G(s^*; \bt)|^2 + \Delta(s^* ; \bt)+\frac{\overline F'(s^*; \bt)\Xi(s^*; \bt)}{p(\bt^*_L)}\Big).
\end{equation}
The conclusion \prettyref{eq:secondorder} then follows from the fact that $ \frac{\partial^2}{\partial s^2}\Delta(s ; \bt) \mid_{s=s^*} < 0 $ since $ s^* $ is a global maximizer.
Let us now show \prettyref{eq:secondderivative}. We use the expression \prettyref{eq:firstderivative} as a starting point. Since $ \frac{\partial}{\partial s}\Delta(s ; \bt) \mid_{s=s^*} = 0 $, the second derivative at $ s = s^* $ is equal to
\begin{equation} \label{eq:deriv1}
\frac{\Xi(s; \bt)p(\bt_L)}{[P(\bt_L)P(\bt_R)]^2} \left[\frac{\partial}{\partial s}[2P(\bt_L)P(\bt_R)\overline G(s; \bt) - \Xi(s; \bt)(1-2P(\bt_L))]\mid_{s=s^*} \right].
\end{equation}
Next, we compute the derivative in \prettyref{eq:deriv1} and find that
\begin{align}
& \frac{\partial}{\partial s}[2P(\bt_L)P(\bt_R)\overline G(s; \bt) - \Xi(s; \bt)(1-2P(\bt_L))] = \nonumber  \\ & \qquad 2p(\bt_L)(1-2P(\bt_L))\overline G(s; \bt) + 2P(\bt_L)P(\bt_R)\overline F'(s; \bt) \nonumber \\ & \qquad - \Xi'(s)(1-2P(\bt_L)) + 2\Xi(s; \bt)p(\bt_L). \label{eq:deriv2}
\end{align}
Next, recall that $ \Xi'(s) = p(\bt_L)\overline G(s; \bt) $, so that the expression in \prettyref{eq:deriv2} is equal to
\begin{equation} \label{eq:deriv3}
p(\bt_L)(1-2P(\bt_L))\overline G(s; \bt) + 2P(\bt_L)P(\bt_R)\overline F'(s; \bt) + 2\Xi(s; \bt)p(\bt_L).
\end{equation}
Next, we multiply \prettyref{eq:deriv3} by $ \frac{\Xi(s; \bt)}{2P(\bt_L)P(\bt_R)p(\bt_L)} $ so that \prettyref{eq:deriv1} is equal to
$$
\frac{2(p(\bt^*_L))^2}{P(\bt^*_L)P(\bt^*_R)}\left[\frac{(1-2P(\bt_L))\Xi(s; \bt)\overline G(s; \bt)}{2P(\bt_L)P(\bt_R)} + \frac{\Xi(s; \bt)\overline F'(s; \bt)}{p(\bt_L)} + \frac{[\Xi(s; \bt)]^2}{P(\bt_L)P(\bt_R)}\right].
$$
Finally, observe that by the first-order condition \prettyref{eq:firstderivative}, $ \frac{(1-2P(\bt^*_L))\Xi(s^*; \bt)\overline G(s^*; \bt)}{2P(\bt^*_L)P(\bt^*_R)} = |\overline G(s^*; \bt)|^2 $, and by definition, $ \frac{[\Xi(s; \bt)]^2}{P(\bt_L)P(\bt_R)} = \Delta(s ; \bt) $.
\end{proof}

\begin{theorem} \label{thm:second}
Suppose \prettyref{ass:density} holds, $ \Delta(s^* ; \bt) > 0 $, and each first-order partial derivative of the regression function and joint density $ p_{\bX} $ exist and are continuous. Then both $ P(\bt^*_L) $ and $ P(\bt^*_R) $ are between
\begin{equation} \label{eq:probderiv}
\frac{1}{2}\Bigg(1\pm\sqrt{1-\Bigg(\frac{4(p(\bt^*_L))^2\Delta(s^* ; \bt)}{|\overline F'(s^*; \bt)|^2}\Bigg)^{1/3}}\Bigg),
\end{equation}
and consequently,
\begin{equation} \label{eq:lambdaloweruniversal}
\lambda(\bt) \geq \Bigg(\frac{4(p(\bt^*_L))^2\Delta(s^* ; \bt)}{|\overline F'(s^*; \bt)|^2}\Bigg)^{1/3}.
\end{equation}
\end{theorem}
\begin{proof}
%If $ \bX $ is uniformly distributed, then one can show that
%$$
%P(\bt^*_L) \geq \left(\frac{\Delta(s^* ; \bt)}{16(b-a)^2|\overline F'(s^*; \bt)|^2}\right)^{1/3} \quad \text{and} \quad P(\bt^*_R) \geq \left(\frac{\Delta(s^* ; \bt)}{16(b-a)^2|\overline F'(s^*; \bt)|^2}\right)^{1/3}.
%$$
First note that $ \overline F'(s^*; \bt) \neq 0 $, since otherwise, by \prettyref{eq:secondderivative} and the second-order optimality condition, $ \Delta(s^* ; \bt) = 0 $. Using the solutions to the first-order condition \prettyref{eq:sol}, we have that
$$
4P(\bt^*_L)P(\bt^*_R) = \frac{\Delta(s^* ; \bt)}{|\overline G(s^*; \bt)|^2+\Delta(s^* ; \bt)},
$$
and furthermore, by the second-order condition \prettyref{eq:secondorder}
$$
4P(\bt^*_L)P(\bt^*_R) \geq \frac{2p(\bt^*_L)\sqrt{\Delta(s^* ; \bt)}}{|\overline F'(s^*; \bt)|\sqrt{4P(\bt^*_L)P(\bt^*_R)}}.
$$
Rearranging yields
$$
\lambda(\bt) = 4P(\bt^*_L)P(\bt^*_R) \geq \left(\frac{4(p(\bt^*_L))^2\Delta(s^* ; \bt)}{|\overline F'(s^*; \bt)|^2}\right)^{1/3}.
%= \left(\frac{4\Delta(s^* ; \bt)}{(b-a)^2|\overline F'(s^*; \bt)|^2}\right)^{1/3}.
$$
Finally, it is a simple exercise to show that if $ 4p(1-p) \geq \lambda $, then $ \frac{1}{2}(1-\sqrt{1-\lambda}) \leq p \leq \frac{1}{2}(1+\sqrt{1-\lambda}) $.
\end{proof}

\begin{remark}
If the regression surface is linear and the distribution of $ \bX $ is uniform, then \prettyref{eq:probderiv} is approximately equal to $ 0.2 $. Compare this with the true value of $ 0.5 $ for both daughter node conditional probabilities.
\end{remark}

%The value of $ \eta_{\bt'} \in \{-1, +1\} $ depends on the following four cases: $ \eta_{\bt'} = 1 $ if (1) $ s^*_j $ is less than $ \text{median}(X \mid \bX \in \bt') $ and $ \bt' $ is the right daughter node of its parent node or (2) $ s^*_j $ is greater than $ \text{median}(X \mid \bX \in \bt') $ and $ \bt' $ is the left daughter node of its parent node; $ \eta_{\bt'} = -1 $ if (1) $ s^*_j $ is less than $ \text{median}(X \mid \bX \in \bt') $ and $ \bt' $ is the left daughter node of its parent node or (2) $ s^*_j $ is greater than $ \text{median}(X \mid \bX \in \bt') $ and $ \bt' $ is the right daughter node of its parent node.

We are now in a position to give the proof of \prettyref{thm:main}.
\begin{proof}[Proof of \prettyref{thm:main}]
Let $ \bt' $ be the parent node of $ \bt $. Suppose we split along coordinate $ X $. Then $ \prob{X \in [a(\bt), b(\bt)]} $ is at most
%\begin{align*}
%\prob{\X \in t_{k+1}} & \leq \max\{\prob{\bX \in \bt_k, \; X_j \leq s^*}, \prob{\bX \in \bt_k, \; X_j > s^*} \} \\ & = \prob{\bX \in \bt_k}\max\{\prob{X_j \leq s^* \mid \bX \in \bt_k}, \prob{X_j > s^* \mid\bX \in \bt_k} \} \\ & = \prob{\bX \in \bt_k}\max\{P(\bt^*_L), P(\bt^*_R)\} \\ & \leq \prob{\bX \in \bt_k}\left(1-\frac{1}{2}\left(1-\sqrt{\frac{1}{1+\lambda^2_j(\bt_k)}}\right)\right) \\
%& \leq \prob{\bX \in \bt_k}\exp\left\{-\frac{1}{12}\lambda^2_j(\bt_k)\right\},
%\end{align*}
\begin{align*}
& \max\{\prob{X \in \bt', \; X \leq s^*}, \prob{X \in \bt', \; X > s^*} \} \\ & = \prob{X \in [a(\bt'), b(\bt')]}\max\{\prob{X \leq s^* \mid X \in [a(\bt'), b(\bt')]}, \\ & \qquad \prob{X > s^* \mid X \in [a(\bt'), b(\bt')]} \} \\ 
%& \leq \prob{X \in I_k}(1-C\min\{P(\bt^*_L), P(\bt^*_R)\}) \\ 
& \leq \prob{X \in [a(\bt'), b(\bt')]}\exp\{ -\eta P(s^*|\bt')(1-P(s^*| \bt')) \} \\ 
%& \leq \prob{X \in I_k}\left(1-\frac{C}{2}\left(1-\sqrt{\frac{1}{1+\lambda^2(\bt_k)}}\right)\right) \\
& = \prob{X \in [a(\bt'), b(\bt')]}\exp\big\{-\frac{\eta }{4}[\Delta(s^*; \bt')+|\overline G(s^*; \bt')|^2]^{-1}\Delta(s^*; \bt')\big\},
\end{align*}
where the first inequality follows from \prettyref{ass:indep}, the penultimate inequality follows from $ \max\{ p, 1-p \} \leq e^{-p(1-p)} $ for $ p \in [0, 1] $, and the final inequality follows from \prettyref{thm:mainprob}.
%where the last line follows from $ 1-z \leq e^{-z} $ for $ z \geq 0 $ and $ \lambda_j(\bt_k) \leq \sqrt{\frac{1+\sqrt{5}}{2}} $.
%and
%\begin{align*}
%\prob{\bX \in t_{k+1}} & \leq \max\{\prob{\bX \in \bt_k, \; X \leq s^*}, \prob{\bX \in \bt_k, \; X > s^*} \} \\ & = \prob{\bX \in \bt_k}\max\{\prob{X \leq s^* \mid \bX \in \bt_k}, \prob{X > s^* \mid\bX \in \bt_k} \} \\ & = \prob{\bX \in \bt_k}\max\{P(\bt^*_L), P(\bt^*_R)\} \\ & \leq \prob{\bX \in \bt_k}(1-\Lambda_j/4).
%\end{align*}
By induction and using $ \prob{X \in [0, 1]} = 1 $, we have $ \prob{a_j(\bt) \leq X_j \leq b_j(\bt)} \leq \exp\big\{-\frac{\eta }{4}\sum_{\substack{\bt' \supset \bt \\ j_{\bt'}=j}}[\Delta(j, s^* ; \bt')+|\overline G_j(s^*; \bt')|^2]^{-1}\Delta(j, s^* ; \bt')\big\} $, which proves \prettyref{eq:nodesize} with weights \prettyref{eq:w1}. The lower bound on the weights \prettyref{eq:w2} is a direct consequence of \prettyref{thm:second}.
\end{proof}
\begin{remark}
Note that $ \Delta(s^*; \bt') > 0 $ is not needed for \prettyref{thm:main}. This is because the inequality \prettyref{eq:nodesize} remains true even if, for some $ \bt' $, $ \Delta(s^*; \bt') = 0 $ (with the convention that $ 0/0 = 0 $).
\end{remark}

\subsection{Alternative splitting rules} \label{sec:alternate}

To mitigate the effect of end-cut splits, as discussed in \prettyref{sec:results}, one can subtract a positive penalty $ \text{pen}(s; \bt) $ from $ \Delta(s ; \bt) $ and instead solve
$$
s^* \in \argmax_{s\in[a(\bt),b(\bt)]}\{ \Delta(s ; \bt) - \text{pen}(s; \bt) \}.
$$
Intuitively, $ \text{pen}(s; \bt) $ should be large when $ s $ is close to the edges and small when $ s $ is far from the edges. The penalty should also be proportional to $ \Delta(s ; \bt) $ so that some influence from original objection function is retained. One natural choice of penalty that meets these criteria is 
$$ \text{pen}(s; \bt) = (1-(4P(\bt_L)P(\bt_R))^{\alpha})\Delta(s ; \bt), \quad \alpha \geq 0. $$ 
Of course, in practice one would use
$$
\hat s \in \argmax_{s\in[a(\bt),b(\bt)]}\{ \widehat\Delta(s ; \bt) - \widehat{\text{pen}}(s; \bt) \},
$$
where $ \widehat{\text{pen}}(s; \bt) = (1-(4\widehat P(\bt_L)\widehat P(\bt_R))^{\alpha})\widehat\Delta(s ; \bt) $.

This regularization procedure is not new---\citep[Section 11.8]{breiman1984} proposed that, to avoid end-cut splits, one should instead maximize $ \widehat\Delta(s ; \bt) $ multiplied by some power of $ \widehat P(\bt_L)\widehat P(\bt_R) $, say $ (\widehat P(\bt_L)\widehat P(\bt_R))^{\alpha} $ for $ \alpha \neq 1 $. Thus, $ (\widehat P(\bt_L)\widehat P(\bt_R))^{\alpha} $ acts as a multiplicative regularizer that modulates the effect of edge-cut preference in CART. The essential challenge is to choose the regularization parameter large enough so that splits away from the parent subnode edges are encouraged, but not in such a way that the homogeneity of the node (as measured by the variance in $ (\bX, Y) $) is unimproved. Good values of $ \alpha $ can be determined by any number of means, including cross-validation on a hold-out set of the data.

Denote the objective function by $ \Delta_{\alpha}(s; \bt) = (4P(\bt_L)P(\bt_R))^{\alpha}\Delta(s ; \bt) $ and its maximizer by $ s^*_{\alpha} $ (similarly define $ \hat s_{\alpha} $ as the empirical maximizer). In \prettyref{fig:sine_edge} (c.f., \prettyref{fig:sine}), we illustrate the effect of regularization in determining the optimal splits for the sinusoidal waveform encountered previously in \prettyref{ex:sine}.

\begin{figure} [t] 
\centering
\begin{subfigure}[t]{0.45\textwidth}
  \centering
  \includegraphics[width=1\linewidth]{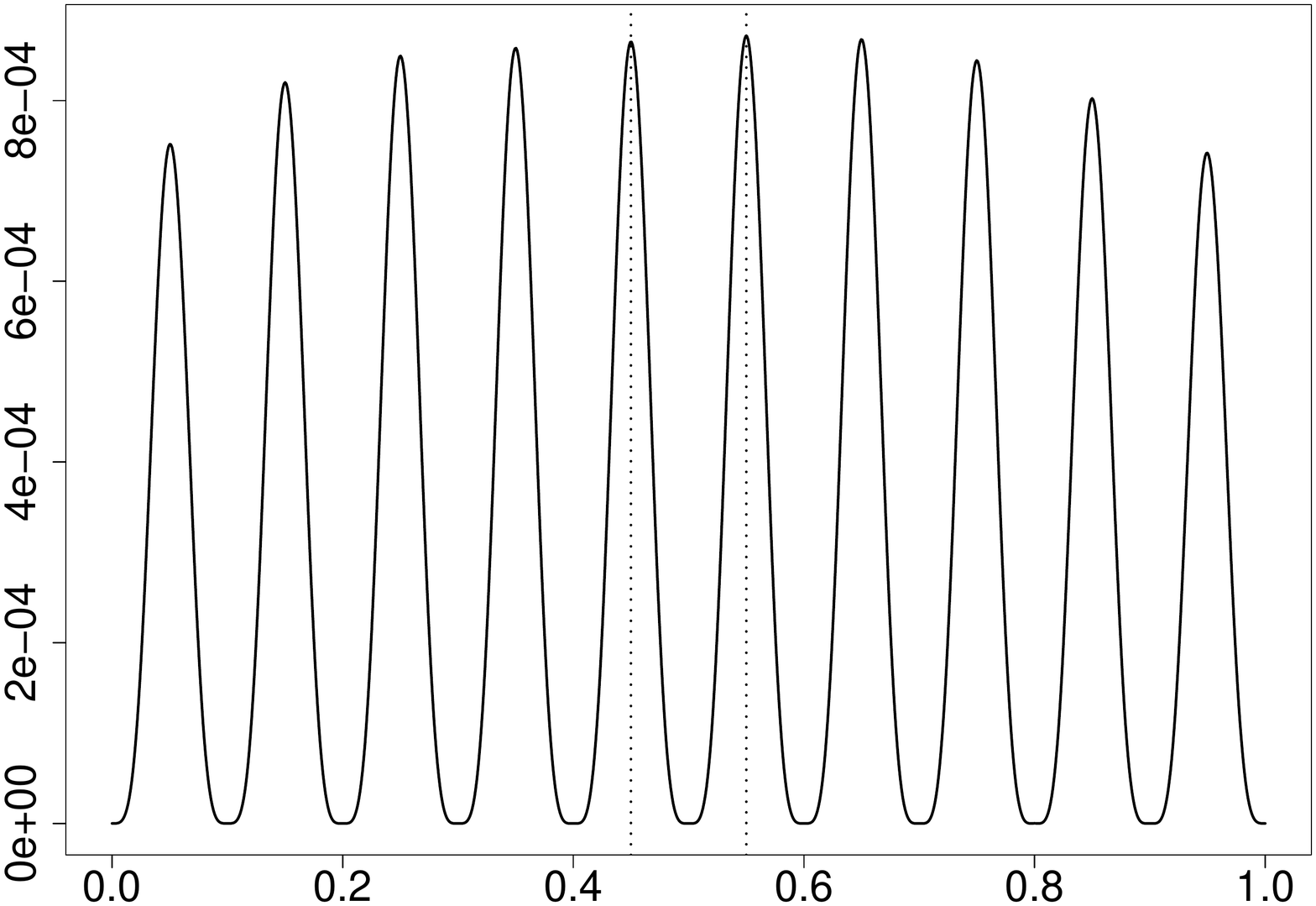}
  \caption{Plot of $ s \mapsto \Delta_{1.1}(1, s; [0, 1]) $ for $ f(\bx) = \sin(20\pi x_1) $.}
\end{subfigure}
\hspace{1cm}
\begin{subfigure}[t]{0.45\textwidth}
  \centering
  \includegraphics[width=1\linewidth]{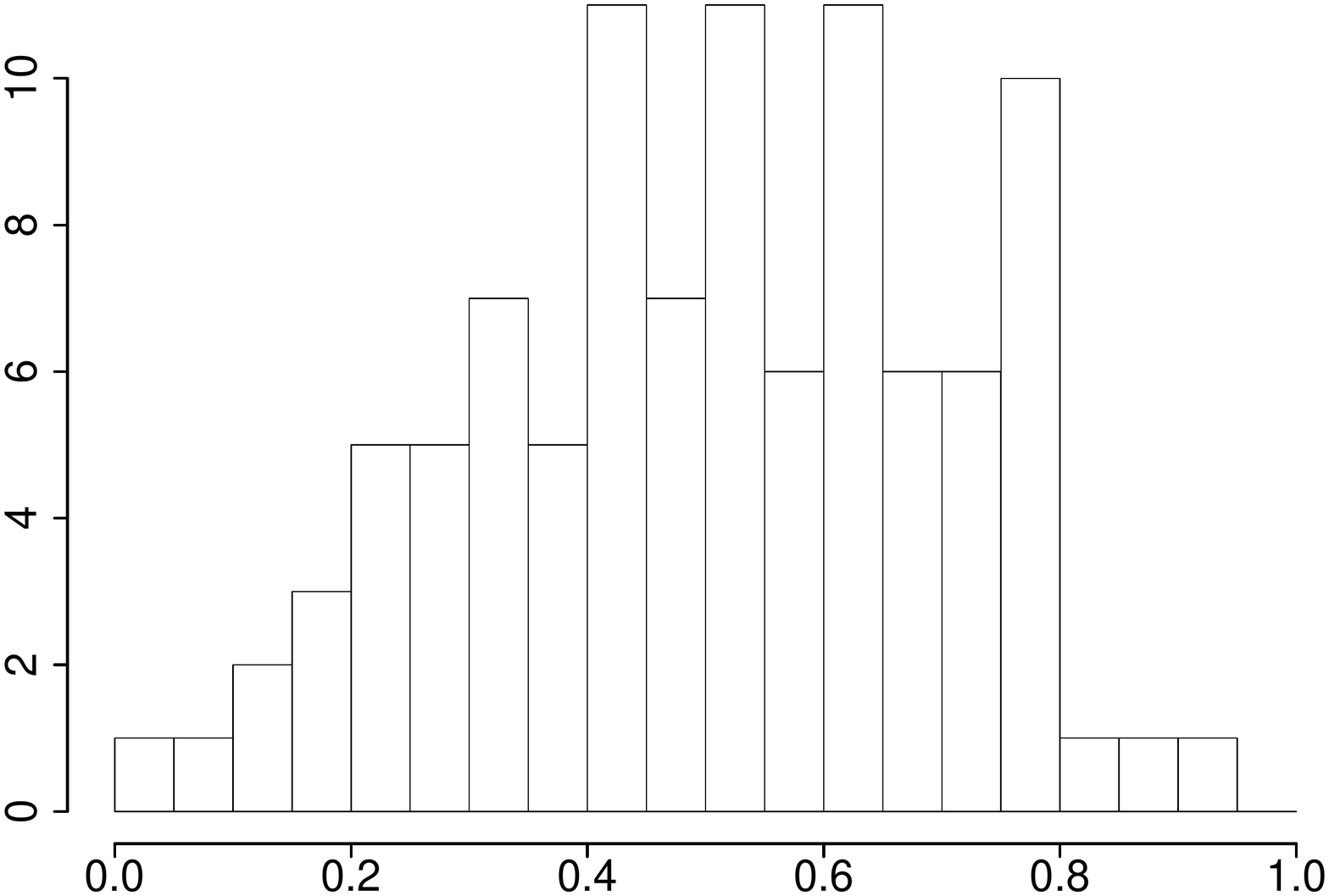}
  \caption{Histogram of $ \hat s_{1.1} $ for $ Y = \sin(20\pi X) + \varepsilon $ ($ X \sim \Unif(0, 1) $, $ \varepsilon \sim N(0, 1) $) with $ n = 100 $ from $ 100 $ replications.}
\end{subfigure}
\caption{Plot of $ s \mapsto \Delta_{1.1}(1, s; \bt) $ and corresponding maxima (dotted vertical lines). Histogram shows sampling distribution of $ \hat s_{1.1} $ for $ n = 100 $ from $ 100 $ replications.}
\label{fig:sine_edge}
\end{figure}

By a similar argument to establishing \prettyref{thm:mainprob}, the optimal $ P(\bt^*_L) $ satisfies
$$
P(\bt^*_L) = \frac{1}{2} - \frac{\text{sgn}(\Xi(s^*_{\alpha}))\overline G(s^*_{\alpha}; \bt)}{2(1-\alpha)\sqrt{\Delta_{\alpha}(s^*_{\alpha}; \bt)}}(4P(\bt^*_L)P(\bt^*_R))^{\frac{\alpha+1}{2}}.
$$ 
With $ v = |1-\alpha|^2\frac{\Delta_{\alpha}(s^*_{\alpha}; \bt)}{|\overline G(s^*_{\alpha}; \bt)|^2} $, we have
%$$
%\frac{1}{2}(1 - \lambda^{-1} (4P(\bt^*_L)P(\bt^*_R))^{\frac{\alpha+1}{2}}) \leq P(\bt^*_L) \leq \frac{1}{2}(1 +  \lambda^{-1} (4P(\bt^*_L)P(\bt^*_R))^{\frac{\alpha+1}{2}}),
%$$ 
%which implies 
$ 4P(\bt^*_L)P(\bt^*_R) = 1- v^{-1}(4P(\bt^*_L)P(\bt^*_R))^{\alpha+1} $ and hence $ 4P(\bt^*_L)P(\bt^*_R) \geq \frac{v^{\frac{1}{\alpha+1}}}{1+v^{\frac{1}{\alpha+1}}} $. Let us now obtain a further lower bound on $ \frac{v^{\frac{1}{\alpha+1}}}{1+v^{\frac{1}{\alpha+1}}} $. To this end, note that by concavity of $ x \to x^{1/(\alpha+1)} $, we have
\begin{align*}
& |\overline G(s^*_{\alpha}; \bt)|^{\frac{2}{\alpha+1}} +((1-\alpha)^2\Delta_{\alpha}(s^*_{\alpha}; \bt))^{\frac{1}{\alpha+1}} \\ & \leq
2^{\frac{\alpha}{1+\alpha}}\left(|\overline G(s^*_{\alpha}; \bt)|^2 + (1-\alpha)^2\Delta_{\alpha}(s^*_{\alpha}; \bt) \right)^{\frac{1}{\alpha+1}}. 
\end{align*}
This means that
\begin{equation} \label{eq:alphalower}
4P(\bt^*_L)P(\bt^*_R) \geq \bigg(\frac{2^{-\alpha}(1-\alpha)^2\Delta_{\alpha}(s^*_{\alpha}; \bt)}{|\overline G(s^*_{\alpha}; \bt)|^2 +(1-\alpha)^2\Delta_{\alpha}(s^*_{\alpha}; \bt)}\bigg)^{\frac{1}{\alpha+1}}.
\end{equation}

Solving this for $ P(\bt^*_L) $ and $ P(\bt^*_R) $ yields the following theorem, which is a direct analog to \prettyref{thm:mainprob}. Using this, we also give a lower bound for $ \Lambda_{\alpha} $, the global balancedness (see \prettyref{def:edgegap}) for $ \Delta_{\alpha}(\cdot; \bt) $.
\begin{theorem} Suppose \prettyref{ass:density} holds and $ \Delta(s^* ; \bt) > 0 $. Then both $ P(\bt^*_L) $ and $ P(\bt^*_R) $ are between
%$$
%\frac{1}{2}\left(1 \pm \sqrt{\frac{1}{1+v^{\frac{1}{\alpha+1}}}} \right),
%$$
%$$
%\frac{1}{2}\left(1 \pm \sqrt{\frac{|\overline G(s^*_{\alpha}; \bt)|^{\frac{2}{\alpha+1}}}{|\overline G(s^*_{\alpha}; \bt)|^{\frac{2}{\alpha+1}} +((1-\alpha)^2\Delta_{\alpha}(s^*_{\alpha}; \bt))^{\frac{1}{\alpha+1}}}} \right).
%$$
$$
\frac{1}{2}\Bigg(1 \pm \sqrt{1-\bigg(\frac{2^{-\alpha}(1-\alpha)^2\Delta_{\alpha}(s^*_{\alpha}; \bt))}{|\overline G(s^*_{\alpha}; \bt)|^2 +(1-\alpha)^2\Delta_{\alpha}(s^*_{\alpha}; \bt)}\bigg)^{\frac{1}{\alpha+1}}} \Bigg),
$$
and consequently
\begin{equation} \label{eq:lambdaalphalower}
\lambda_{\alpha}(\bt) \geq \bigg(\frac{2^{-\alpha}(1-\alpha)^2\Delta_{\alpha}(s^*_{\alpha}; \bt))}{|\overline G(s^*_{\alpha}; \bt)|^2 +(1-\alpha)^2\Delta_{\alpha}(s^*_{\alpha}; \bt)}\bigg)^{\frac{1}{\alpha+1}}.
\end{equation}
%where $ v = |1-\alpha|^2\frac{\Delta_{\alpha}(s^*_{\alpha}; \bt)}{|\overline G(s^*_{\alpha}; \bt)|^2} $.
\end{theorem}
It is often possible to show that $ \frac{(1-\alpha)^2\Delta_{\alpha}(s^*_{\alpha}; \bt)}{|\overline G(s^*_{\alpha}; \bt)|^2 +(1-\alpha)^2\Delta_{\alpha}(s^*_{\alpha}; \bt)} \asymp \frac{\Delta(s^* ; \bt)}{|\overline G(s^*; \bt)|^2 +\Delta(s^* ; \bt)} $ and hence, by virtue of \prettyref{eq:alphalower}, \prettyref{eq:lambdauniversal}, and \prettyref{eq:lambdaalphalower}, that $ \Lambda_{\alpha}  \asymp \Lambda^{1/(1+\alpha)} $, where $ \Lambda = \Lambda_0 $ is the global balancedness for the unpenalized criterion. The next theorem (c.f., \prettyref{thm:second}) shows that this is improvable to $ \Lambda_{\alpha}  \asymp \Lambda^{1/(3+\alpha)} $ when $ \alpha \in [0, 1) $. Using this, it can easily be shown via a modification of the proofs of \prettyref{ex:poly} and \prettyref{ex:sine} that $ \Lambda_{\alpha} = \Omega(k^{-4/(3+\alpha)}) $ and $ \Lambda_{\alpha} = \Omega(m^{-4/(3+\alpha)}) $, respectively. These quantities are larger than their counterparts using the unpenalized $ \Delta(\cdot; \bt) $ and hence the regularization does indeed encourage splits that are farther away from the parent subnode edges.
%Take, for example, an additive polynomial model $ f(\bx) = \sum_{j=1}^d x^{k_j}_j $ with integer degrees $ k_j $. By \prettyref{thm:poly}, $ \Lambda \geq \frac{1}{2\max_j k_j} $. However, one can show that with the $ \Delta_{\alpha} $ splitting rule, there exists a constant $ C = C_{\alpha} $ such that $ \Lambda \geq \frac{C_{\alpha}}{\max_j k_j} $. Hence the difference between the global balancednesss induced by $ \Delta $ and $ \Delta_{\alpha} $ is $ (\max_j k_j)^{-2} $ and $ (\max_j k_j)^{-2/(\alpha+1)} $, respectively, which could be significant when $ \max_j k_j $ is large. Similarly, if $ f(\bx) = \sum_{j=1}^d \sin(2\pi m_j x_j) $, then this difference could be between $ (\max_j m_j)^{-2} $ and $ (\max_j m_j)^{-2/(\alpha+1)} $. 

\begin{theorem}
Suppose \prettyref{ass:density} holds, $ \Delta(s^* ; \bt) > 0 $, and each first-order partial derivative of the regression function and joint density $ p_{\bX} $ exist and are continuous. Let $ \alpha \in [0, 1) $. Then both $ P(\bt^*_L) $ and $ P(\bt^*_R) $ are between
$$
\frac{1}{2}\Bigg(1\pm\sqrt{1-\Bigg(\frac{4^{1-\alpha}(1-\alpha)^2(p(\bt_L))^2\Delta_{\alpha}(s^*_{\alpha}; \bt)}{|\overline F'(s^*_{\alpha}; \bt)|^2}\Bigg)^{\frac{1}{3+\alpha}}}\Bigg),
$$
and consequently,
$$
\lambda_{\alpha}(\bt) \geq \Bigg(\frac{4^{1-\alpha}(1-\alpha)^2(p(\bt_L))^2\Delta_{\alpha}(s^*_{\alpha}; \bt)}{|\overline F'(s^*_{\alpha}; \bt)|^2}\Bigg)^{\frac{1}{3+\alpha}}.
$$
\end{theorem}

\begin{proof}
We have the chain of inequalities
\begin{align}
%& |\overline G(s^*_{\alpha}; \bt)|^{\frac{2}{\alpha+1}} +((1-\alpha)^2\Delta_{\alpha}(s^*_{\alpha}; \bt))^{\frac{1}{\alpha+1}} \\ & \leq
& 2^{\frac{\alpha}{1+\alpha}}\left(|\overline G(s^*_{\alpha}; \bt)|^2 + (1-\alpha)^2\Delta_{\alpha}(s^*_{\alpha}; \bt) \right)^{\frac{1}{\alpha+1}} \nonumber \\
& \leq 2^{\frac{\alpha}{1+\alpha}}\left((1+\alpha)|\overline G(s^*_{\alpha}; \bt)|^2 + (1-\alpha)^2\Delta(s^*_{\alpha}; \bt) \right)^{\frac{1}{\alpha+1}} \nonumber \\
& \leq \Big(\frac{2^{\alpha}(1-\alpha)|\overline F'(s^*_{\alpha}; \bt)||\Xi(s^*_{\alpha})|}{p(\bt^*_L)}\Big)^{\frac{1}{\alpha+1}} \nonumber \\
& = \Big(\frac{(1-\alpha)|\overline F'(s^*_{\alpha}; \bt)|\sqrt{(4P(\bt^*_L)P(\bt^*_R))^{1-\alpha}\Delta_{\alpha}(s^*_{\alpha}; \bt)}}{2^{1-\alpha}p(\bt^*_L)}\Big)^{\frac{1}{\alpha+1}}. \label{eq:secondalphalower}
\end{align}
The first inequality is by concavity of $ x\mapsto x^{1/(\alpha+1)} $ and the second inequality is due to the fact that $ \Delta_{\alpha}(s^*_{\alpha}; \bt) \leq \Delta(s^*_{\alpha}; \bt) $. The third inequality comes from the second-order derivative condition (c.f., \prettyref{eq:secondderivative}), i.e., given $ \frac{\partial}{\partial s}\Delta_{\alpha}(s; \bt) \mid_{s=s^*_{\alpha}} = 0 $, the second derivative $ \frac{\partial^2}{\partial s^2}\Delta_{\alpha}(s; \bt) \mid_{s=s^*_{\alpha}}  $ equals
$$
\frac{8(p(\bt^*_L))^2}{(4P(\bt^*_L)P(\bt^*_R))^{1-\alpha}}\Big(\frac{1+\alpha}{1-\alpha}|\overline G(s^*_{\alpha}; \bt)|^2 + (1-\alpha)\Delta(s^*_{\alpha}; \bt) + \frac{\overline F'(s^*_{\alpha}; \bt)\Xi(s^*_{\alpha})}{p(\bt^*_L)}\Big).
$$
Finally, combining \prettyref{eq:secondalphalower} with \prettyref{eq:alphalower} and solving for $ P(\bt^*_L) $ and $ P(\bt^*_R) $ yields both statements of the theorem.
\end{proof}

\subsection{Lower bounds on the node balancedness} \label{sec:lower} Of special interest is $ \Lambda > 0 $, since this provides a nontrivial bound on the distance between \emph{any} optimal split to its parent subnode edges. But can we expect this to hold in most settings? It is conceivable that $ \lambda $ may become extremely small when $ a(\bt) $ and $ b(\bt) $ are arbitrarily close to each other, since after all by \prettyref{thm:mainprob}, its defining quantities---$ \Delta(s^* ; \bt) $ and $ \Delta(s^* ; \bt) + |\overline G(s^*; \bt)|^2 $---expressed through their ratio, both approach zero. We argue that $ \lambda $ is still controlled in this case. To see this, suppose $ a(\bt) $ and $ b(\bt) $ are extremely close to each other. Then the partial dependence function is approximately linear, i.e., $ \expect{Y\mid \bX \in \bt, \; X = s} \approx As+B $ for some constants $ A $ and $ B $ and also $ \prob{X \leq s \mid \bX \in \bt} \approx (s-a(\bt))/(b(\bt)-a(\bt)) $. Hence $ \lambda \approx 1 $, with equality if $ \bX $ is uniform and the conditional regression surface is exactly linear. The statement in \prettyref{thm:lambdalower}, which we now prove, makes this intuition precise. First, we state two lemmas, but defer their proofs until \prettyref{app:proofs}.

\begin{lemma} \label{lmm:flatpartial}
If the features of $ \bX $ are independent and satisfy \prettyref{ass:density} and the regression function satisfies \prettyref{eq:flat}, then $ \sup_{s\in[0,1]}\inf_{r\geq 1}\{ r : \overline F^{(r)}(s; \bt) \neq 0\; \text{for all} \; \bt\} $ is finite and for each node $ \bt $, $ \Delta(s^* ; \bt) > 0 $.
\end{lemma}

\begin{lemma} \label{lmm:deltalower}
%Suppose $ \bX \sim \Unif([0, 1]^d) $ and $ \overline F(s; \bt) \sim \sum_k c_{k}e^{2\pi \mathrm{i} k (s-a)/(b-a)} $, where $ c_0 = \expect{Y \mid \bX \in \bt} $ and $ c_k = \frac{1}{b-a}\int_a^b \overline F(s; \bt)e^{-2\pi \mathrm{i}k (s-a)/(b-a)}ds = \int_0^1\overline F(s(b-a)+a; \bt)e^{-2\pi \mathrm{i}ks}ds $, $ k \neq 0 $. Then,
%\begin{equation} \label{eq:lowerdelta}
%\Delta(s^* ; \bt) \geq 6\left(\sum_{k\neq 0}\frac{|c_{k}|^2}{4\pi^2 k^2}+\bigg|\sum_{k\neq 0}\frac{c_{k}}{2 \pi k}\bigg|^2\right).
%\end{equation}
%In other words, $ \Delta(s^* ; \bt) $ tends to be small if the partial dependence function within $ \bt $ has few low level frequencies in its frequency domain.
Suppose the features of $ \bX $ are independent with marginal densities that are continuous and never vanish, and the regression function satisfies \prettyref{eq:flat}. If $ R = \inf_{r\geq 1}\{ r : \overline F^{(r)}(c; \bt) \neq 0\; \text{for all} \; \bt\} < +\infty $, then
% and $ \overline F^{(R)}(s; \bt) $ is continuous at $ s = c $
\begin{equation} \label{eq:deltalimit}
\liminf_{(a(\bt), b(\bt)) \rightarrow (c,c)} \left\{\frac{\Delta(s^* ; \bt)}{\big(\frac{(b(\bt)-a(\bt))^{R}|\overline F^{(R)}(c; \bt)|}{R!}\big)^2}\right\}  \geq \Delta_R,
\end{equation}
where
$$
\Delta_R = \int_0^1\Delta(s; [0, 1])ds > 0
$$
is the integrated decrease in impurity $ \Delta(\cdot; [0, 1]) $ of the regression function $ f(\bx) = (x_1-1/2)^R $ with respect to the uniform distribution.
\end{lemma}

\begin{proof}[Proof of \prettyref{thm:lambdalower}]
%We first make the observation that $ f_j(s) = \overline F_j(s; \bt) $. 
%Without loss of generality we omit the subscript in $ j $ and work instead with a generic $ \overline F(s; \bt) $. 
By \prettyref{thm:mainprob}, $ \text{MDI}(X; \bt) $ with weights $ w(s^*; \bt) = [|\overline G(s^*; \bt)|^2+\Delta(s^*; \bt)]^{-1} $ is at least $ \Lambda K(\bt) $. Hence we need to show that global balancedness $ \Lambda $ is positive. To this end, the first step in the proof involves showing that
\begin{equation} \label{eq:limit}
\liminf_{(a, b) \rightarrow (c,c)}\lambda(\bt) = \liminf_{(a, b) \rightarrow (c,c)}\{ 4P(\bt^*_L)P(\bt^*_R) \} \geq \left(\frac{4\Delta_R}{R^2}\right)^{1/3},
\end{equation}
where $ R =\inf_{r\geq 1}\{ r : \overline F^{(r)}(c; \bt) \neq 0\; \text{for all} \; \bt\} < +\infty $ (by \prettyref{lmm:flatpartial}) and $ \Delta_R $ is the positive constant from \prettyref{lmm:deltalower}.
This can be accomplished by \prettyref{lmm:deltalower} since
\begin{equation} \label{eq:liminf}
\liminf_{(a, b) \rightarrow (c,c)} \Bigg\{\frac{\Delta(s^* ; \bt)}{\Big(\frac{(b-a)^{R}|\overline F^{(R)}(c; \bt)|}{R!}\Big)^2}\Bigg\} \geq \Delta_R > 0.
\end{equation}
Next, consider an $ R-1 $ term Taylor expansion of $ \overline F'(\cdot; \bt) $. Then, by definition of $ R $ and a Taylor expansion argument, $ |\overline F'(s^*; \bt)| \leq \frac{(b-a)^{R-1}}{(R-1)!}\sup_{s\in[a,b]}|\overline F^{(R)}(s; \bt)| $. Thus, combining \prettyref{eq:liminf} with the fact that $ \lim_{(a,b)\rightarrow (c, c)}(b-a)p(\bt_L) = 1 $ (since $ p_X(c) > 0 $ and $ p_X(\cdot) $ is continuous by assumption) and by continuity of $ \overline F(\cdot; \bt) $ at $ s = c $, we have
$$
\liminf_{(a, b) \rightarrow (c,c)} \left\{\frac{4(p(\bt_L))^2\Delta(s^* ; \bt)}{|\overline F'(s^*; \bt)|^2}\right\} \geq \frac{4\Delta_R}{R^2}.
$$
Finally, \prettyref{eq:lambdaloweruniversal} from \prettyref{thm:second} implies \prettyref{eq:limit}. The assumption of finite $ R' = \sup_{s\in[0,1]}\inf_{r\geq 1}\{ r : \overline F^{(r)}(s; \bt) \neq 0 \; \text{for all} \; \bt\} $ implies that
$$
\inf_c\liminf_{(a, b) \rightarrow (c,c)}\{ 4P(\bt^*_L)P(\bt^*_R) \} \geq \min_{R \leq R'}\left(\frac{4\Delta_R}{R^2}\right)^{1/3} > 0.
$$
%and $ 4P(\bt^*_L)P(\bt^*_R) = \frac{\Delta(s^* ; \bt)}{\Delta(s^* ; \bt)+|\overline G(s^*; \bt)|^2} $ implies that
%$$
%\limsup_{a \rightarrow b} \frac{|G(s^*; \bt)|^2}{\Delta(s^* ; \bt)} \leq 1.
%$$
%Hence, we conclude that \prettyref{eq:limit} holds.
Next, let $ \ba $ (resp. $ \bb $) denote the vector of lower (resp. upper) endpoints of the subnodes of $ \bt $. Now, since the regression function is continuous, it follows that $ (s,\ba, \bb) \mapsto \Delta(s ; \bt) $ and $ (s,\ba, \bb) \mapsto \overline G(s; \bt) $ are both continuous on the domain $ \{ (s,\ba, \bb) \in [0, 1]^{2d+1} : a(\bt) \leq s \leq b(\bt),\; \ba < \bb \} $.\footnote{This can be seen from the generalized mean value theorem for integrals.} Consequently, by Berge's Maximum Theorem \citep[Theorem 17.31]{aliprantisinfinite2006} the mapping $ (\ba, \bb) \mapsto \Delta(s^* ; \bt) $ is continuous on $ \calT  = \{ (\ba, \bb) \in [0, 1]^{2d} : \ba < \bb \} $ and $ (\ba, \bb) \mapsto \overline G(s^*; \bt) $ is an upper hemicontinuous correspondence on $ \calT $. In particular, by \prettyref{thm:mainprob}, $ 4P(\bt^*_L)P(\bt^*_R) = \frac{\Delta(s^* ; \bt)}{\Delta(s^* ; \bt)+|\overline G(s^*; \bt)|^2} $ and hence $ 4P(\bt^*_L)P(\bt^*_R) $ is an upper hemicontinuous correspondence on $ \calT $. Next, note that by \prettyref{lmm:flatpartial} and \prettyref{thm:mainprob}, $ 4P(\bt^*_L)P(\bt^*_R) > 0 $ on $ \calT $, and by \prettyref{eq:limit}, $ 4P(\bt^*_L)P(\bt^*_R) \geq \min_{R \leq R'}\left(\frac{4\Delta_R}{R^2}\right)^{1/3} > 0 $ for all points $ (\ba, \bb) $ arbitrarily close to the boundary of $ \calT $. Hence $ \Lambda > 0 $.
\end{proof}

\section*{Acknowledgment}
	\label{sec:ack}
	The author is indebted to Min Xu, Minge Xie, Samory Kpotufe, Robert McCulloch, Andy Liaw, Richard Baumgartner, and Michael Kosorok for helpful discussion and feedback. The author would also like to thank Joowon Klusowski for her help with the proof of \prettyref{lmm:deltalower}. 
	%Finally, the author is grateful to his mentor, Andrew R. Barron, for many stimulating conversations on high-dimensional modeling over the years.

\bibliographystyle{plainnat}
\bibliography{forest.bib}

\newpage
\thispagestyle{empty}
\appendix
\setcounter{page}{1}
\begin{center}
	\MakeUppercase{\bf \large Supplement to ``Best Split Nodes For Regression Trees''}
	
	\medskip
	
	\MakeUppercase{{By Jason M. Klusowski}
}

\end{center}

\appendix
\section{Supplemental Material}\label{app:proofs}
\addcontentsline{toc}{section}{Appendix}
\renewcommand{\thesubsection}{\Alph{subsection}}

In this appendix, we give proofs of \prettyref{lmm:exppoly}, \prettyref{thm:forest}, \prettyref{lmm:flatpartial}, \prettyref{lmm:deltalower}, \prettyref{thm:sample}, and \prettyref{thm:sampledata}. We also give proofs of the examples from \prettyref{sec:mdilower}, \prettyref{sec:distance}, and \prettyref{sec:classification}.

\subsection{Proofs of main lemmas} \label{app:mainproofs}

\begin{proof}[Proof of \prettyref{lmm:exppoly}]
We proceed by induction. The case $K=1$ is trivial, since $R_1 e^{P_1} = 0$ clearly implies $R_1 = 0$. Now let $K \geq 2 $ be arbitrary and assume that the claim is true for all smaller values of $K$. Let $P_1, \dots, P_K$ be distinct (real or complex) polynomials without constant terms and $R_1, \ldots, R_K$ be (real or complex) polynomials with
\begin{equation} \label{eq:ass}
R_1 e^{P_1} + \cdots + R_K e^{P_K} = 0.
\end{equation}
If all $R_k$ are zero then we are done. Otherwise (without loss of generality) $R_K \neq 0$. First we divide \prettyref{eq:ass} by $ e^{P_K} $, yielding
\begin{equation} \label{eq:div}
 R_1e^{P_1 - P_K} + \cdots + R_{K-1}e^{P_{K-1} - P_K} + R_K = 0.
\end{equation}
Differentiating the identity \prettyref{eq:div} gives
\begin{align} \label{eq:diff}
& \left(R'_1 + R_1(P_1'-P_K')\right)e^{P_1 - P_K} + \cdots + \nonumber \\ & \quad
\left( R'_{K-1} + R_{K-1}(P_{K-1}'-P_K')\right)e^{P_{K-1} - P_K} + R'_K = 0.
\end{align}
Multiply \prettyref{eq:div} by $ R'_K $ and \prettyref{eq:diff} by $ R_K $. Subtracting the two resultant expressions from each other yields
\begin{align*}
& \left(R'_1R_K -R_1R'_K + R_1R_K(P_1'-P_K')\right)e^{P_1 - P_K} + \cdots + \\ & \quad
\left( R'_{K-1}R_K - R_{K-1}R'_K + R_{K-1}R_K(P_{K-1}'-P_K')\right)e^{P_{K-1} - P_K} = 0.
\end{align*}

Now we can apply the induction hypotheses, since the $P_k - P_K$ are distinct polynomials without constant terms. It follows that
\begin{equation} \label{eq:zero}
R'_kR_K -R_kR'_K + R_kR_K(P_k'-P_K') = 0.
\end{equation}
If $ R_k \neq 0 $, then \prettyref{eq:zero} is impossible since $ P_k'-P_K' \neq 0 $ and hence $ \deg(R'_kR_K -R_kR'_K) < \deg(R_kR_K(P_k'-P_K')) $. Therefore $R_k = 0$ for $k=1,\dots, K-1$, which also implies that $R_K = 0$.
\end{proof}

\begin{proof}[Proof of \prettyref{thm:forest}]
We follow the proof of \citep[Lemma 2]{scornet2016asymptotics} for quantile forests, but adapted to our setting. Let $ \bt_k(\bX, \Theta) $ denote the node containing $ \bX $ of the tree built with randomness $ \Theta $ at the $ k^{\Th} $ step. By \citep[Theorem 4.1]{scornet2016asymptotics}, we will be done if we can show that $ \omega(f; \bt(\bX, \Theta)) \rightarrow 0 $ in $ \mathbb{P}_{\bX, \Theta} $-probability, where $ \bt(\bX, \Theta) $ is the terminal node of the tree containing $ \bX $. Since $ f $ is continuous on $ [0, 1]^d $, it is also uniformly continuous. Hence, for each $ \xi > 0 $, there exists $ \delta > 0 $ such that if $ \text{diam}_{\calS}(\bt) \leq \delta $, then $ \omega(f; \bt) \leq \xi $. Hence, we must show that $ \text{diam}_{\calS}( \bt(\bX, \Theta)) \rightarrow 0 $ in probability.
To this end, let $ j \in \calS $, $ H = \{ \bx : x_j = z \} $, and $ D = \{A : A \cap H \neq \emptyset\} $. Let $ j_k(\Theta) $ denote the coordinate selected to split along at the $ k^{\Th} $ step of the tree. Suppose $ \bt_k(\bX, \Theta) \in D $. Then there are two cases:
\begin{enumerate}
\item The next split in $ \bt_k(\bX, \Theta) $ is performed along the $ j^{\Th} $ coordinate and, in that case, one of the two resulting nodes has an empty intersection with $ H $.
\item The next split in $ \bt_k(\bX, \Theta) $ is performed along a coordinate other than the $ j^{\Th} $ and, in that case, the two resultant nodes have a non-empty intersection with $ H $.
\end{enumerate}
Thus,
\begin{align*}
& \prob{ \bt_{k+1}(\bX, \Theta) \in D \mid \Theta} =
\\ & \qquad \expect{ \indc{j_{k+1}(\Theta) = j}\indc{\bt_k(\bX, \Theta) \in D} + \indc{j_{k+1}(\Theta) \neq j}\indc{\bt_k(\bX, \Theta) \in D} \mid \Theta}
\\ & \qquad \leq \indc{j_{k+1}(\Theta) = j}(1-\frac{1}{4}\lambda_j(\bt_k(D, \Theta)))\prob{\bt_k(\bX, \Theta) \in D \mid \Theta} + \\ & \qquad \qquad \indc{j_{k+1}(\Theta) \neq j}\prob{\bt_k(\bX, \Theta) \in D \mid \Theta}
\\ & \qquad = (1-\frac{1}{4}\lambda^2_j(\bt_k(D, \Theta))\indc{j_{k+1}(\Theta) = j})\prob{\bt_k(\bX, \Theta) \in D \mid \Theta}
\\ &\qquad \leq \exp\big\{-\frac{1}{4}\lambda^2_j(\bt_k(D, \Theta))\indc{j_{k+1}(\Theta) = j}\big\}\prob{\bt_k(\bX, \Theta) \in D \mid \Theta},
\end{align*}
where $ \bt_k(D, \Theta) $ is the (unique) node at the $ k^{\Th} $ step of the forest construction that contains $ z $.
%since $ \prob{\widehat j_k(\bX, \Theta) = j} \geq 1/d^{\texttt{mtry}} $. 
By induction and \prettyref{eq:lambdauniversal}, this implies that $ \prob{\bt(\bX, \Theta) \in D} \rightarrow 0 $ if $ \min_{\bt}\text{MDI}(X_j; \bt) \rightarrow + \infty $ with $ \mathbb{P}_{\Theta} $-probability one. Finally, consider a partition of $ [0, 1]^{S} $ into hypercubes of side length $ \epsilon $ with sides determined by the hyperplanes $ \{ \bx : x_{j'} = \ell \epsilon \} $, where  $ j' \in \calS $ and $ \ell = 0, 1, \dots, \lceil \epsilon^{-1} \rceil $. If $ \bt(\bX, \Theta) $ belongs to one of the hypercubes, then $ \text{diam}_{\calS}(\bt(\bX, \Theta)) \leq \sqrt{S}\epsilon $. There are at most $ O(\epsilon^{-S}) $ such hyperplanes and hence
$$
\prob{\text{diam}_{\calS}(\bt(\bX, \Theta)) > \epsilon \sqrt{S}} \leq \prob{\bigcup_D \{\bt(\bX, \Theta) \in D \}} \rightarrow 0,
$$
if $ \min_{\bt}\text{MDI}(X_j; \bt) \rightarrow + \infty $ with $ \mathbb{P}_{\Theta} $-probability one, where $ D $ ranges over all hyperplanes of the form $ \{ \bx : x_{j'} = \ell \epsilon \} $, with $ j' \in \calS $ and $ \ell = 0, 1, \dots, \lceil \epsilon^{-1} \rceil $. The conditions of the theorem and \citep[Theorem 4.1]{scornet2016asymptotics} imply the conclusion.
\end{proof}

\begin{proof}[Proof of \prettyref{lmm:flatpartial}]
Let $ r \geq 1 $ and suppose $ \frac{\partial^r}{\partial x^r_j}f(\bx) $ exists and is continuous for all $ \bx_{\setminus j} \in [0, 1]^{d-1} $. By Leibniz's integral rule, we have
\begin{equation} \label{eq:derivativeeq}
\frac{\partial^r}{\partial x^r_j}\overline F_j(x_j; \bt) = \int \frac{\partial^r}{\partial x^r_j}f(x_j, \bx_{\setminus j})\mathbb{P}_{\bX_{\setminus j}\mid \bX \in \bt}(d\bx_{\setminus j}).
\end{equation}
By \prettyref{eq:derivativeeq} and the generalized mean value theorem for integrals, there exists $ \bx'_{\setminus j} \in \bt_{\setminus j} $ such that $\frac{\partial^r}{\partial x^r_j} f(x_j,\bx'_{\setminus j}) =  \frac{\partial^r}{\partial x^r_j}\overline F_j(x_j; \bt) $. By assumption that \prettyref{eq:flat} holds there exists a finite integer $ R \geq 1 $ such that for each $ x_j \in [a_j(\bt), b_j(\bt)] $, there exists an integer $ 1 \leq r \leq R $ such that $ \frac{\partial^r}{\partial x^r_j} f(x_j,\bx_{\setminus j}) \neq 0 $ for all $ \bx_{\setminus j} \in [0, 1]^{d-1}$. In particular, $ \frac{\partial^r}{\partial x^r_j} F_j(x_j; \bt) = \frac{\partial^r}{\partial x^r_j} f(x_j,\bx'_{\setminus j}) \neq 0 $ and hence $ \sup_{s\in[0, 1]}\inf_{r\geq 1}\{ r: \overline F^{(r)}(s; \bt) \neq 0\; \text{for all}\; \bt \} $ is finite. Since $ x_j $ was arbitrary in $ [a_j(\bt), b_j(\bt)] $, this implies that $ \overline F_j(\cdot; \bt) $ is nonconstant on $ [a_j(\bt), b_j(\bt)] $. Finally, it is easy to show that if $ \Delta(j, s^*; \bt) = 0 $, then $ \overline F_j(\cdot; \bt) = \expect{Y \mid \bX \in \bt} $.
\end{proof}

\begin{proof}[Proof of \prettyref{lmm:deltalower}]
%We will first prove \prettyref{eq:lowerdelta}. 
First, note that
$$
%\Delta(s; \bt) = \frac{(\int_a^{s}\overline G(s'; \bt)ds')^2}{(s-a)(b-s)}, \quad s \in [a,b].
\Delta(s; \bt) = \frac{(\int_a^{s}p(s|\bt)\overline G(s'; \bt)ds')^2}{P(s|\bt)(1-P(s|\bt))}, \quad s \in [a,b].
$$
%Consider the prior density $ \Pi(s) = \frac{6(s-a)(b-s)}{(b-a)^3}\indc{s\in(a, b)} $ on the splits $ s \in [a, b] $. 
%Consider the uniform prior density $ \Pi(s) = \frac{1}{b-a}\indc{s\in(a, b)} $ on the splits $ s \in [a, b] $. 
Since a maximum is larger than an average, for any prior $ \Pi $ on $ [0, 1] $ with density $ \pi $,
%\begin{align*}
%\Delta(s^*; \bt) & \geq \int_a^b\Pi(s)\frac{(\int_a^{s}\overline G(s'; \bt)ds')^2}{(s-a)(b-s)}ds \\
%& = \frac{6}{(b-a)^3}\int_a^b\left(\int_a^s\overline G(s'; \bt)dx\right)^2ds \\
%& = \frac{6}{b-a}\int_a^b\bigg|\sum_{k\neq 0} \frac{c_{k}}{2\pi \mathrm{i} k}(e^{2\pi \mathrm{i} k(s-a)/(b-a)}-1)\bigg|^2ds \\
%& = 6\left(\sum_{k\neq 0}\frac{|c_{k}|^2}{4\pi^2 k^2}+\bigg|\sum_{k\neq 0}\frac{c_{k}}{2 \pi k}\bigg|^2\right).
%& \geq \frac{3}{2\pi^2}\sum_{k \neq 0}\frac{|c_{k}|^2}{k^2},
%\end{align*}
%where the penultimate line follows from Parseval's identity.
%\begin{align*}
\begin{align*}
\Delta(s^*; \bt) & \geq \int_a^b\frac{(\int_a^{s}p(s|\bt)\overline G(s'; \bt)ds')^2}{P(s|\bt)(1-P(s|\bt))}\frac{\pi((s-a)/(b-a))}{b-a}ds \\
%& = \int_a^b\frac{1}{b-a}\frac{(\int_a^{s}\overline G(s'; \bt)ds')^2}{(s-a)(b-s)}ds \\
& = \int_0^1\frac{(\int_0^{s}(b-a)p(a+s'(b-a)|\bt)\overline G(a+s'(b-a); \bt)ds')^2}{P(a+s(b-a)|\bt)(1-P(a+s(b-a)|\bt))}\Pi(ds).
\end{align*}
In particular, we choose the uniform prior, i.e., $ \pi(s) = \indc{s\in[0, 1]} $. 
Next, by assumption, $ p_X(\cdot) $ is positive and continuous and hence
$$
\lim_{(a,b)\to(c,c)}(b-a)p(a+s(b-a)|\bt) = \frac{p_X(c)}{p_X(c)} = 1
$$
and
$$
\quad \lim_{(a,b)\to(c,c)}P(a+s(b-a)|\bt) = s\frac{p_X(c)}{p_X(c)} = s,
$$
where the convergence is uniform. Thus, we assume henceforth that $ \bX $ is uniform. The proof for general $ \bX $ follows similarly.

Let $ D(s) $ denote the divided difference $ \frac{\overline F(a+s(b-a); \bt)-\overline F(c; \bt)}{(a+s(b-a)-c)^R} $ and $ \delta = \frac{c-a}{b-a} $. Then, we can rewrite $ (b-a)^{-R}\overline G(a+s'(b-a); \bt) $ via
$$
D(s')\left(s'-\delta\right)^R-\int_0^1D(s'')\left(s''-\delta\right)^Rds''.
$$
Fix $ s\in [a, b] $ and use a Taylor expansion of $ \overline F(\cdot; \bt) $ about the point $ s = c $ and continuity of $ \overline F^{(R)}(\cdot; \bt) $ at $ s = c $ to argue that
$$
\lim_{(a,b)\to(c,c)}D(s) = \frac{\overline F^{(R)}(c; \bt)}{R!},
$$
where the convergence is uniform and the limit is nonzero by assumption. Without loss of generality, assume $ \overline F^{(R)}(c; \bt) > 0 $. By uniform continuity, there exists $ \xi > 0 $ such that if $ \sqrt{(c-a)^2+(b-c)^2} < \xi $, then
\begin{equation} \label{eq:cont}
\Big|D(s)-\frac{\overline F^{(R)}(c; \bt)}{R!}\Big| < \min\Big\{\frac{\overline F^{(R)}(c; \bt)}{2 R!}, \frac{1}{\delta^2} \Big\}.
\end{equation}
Using the fact that $ s(1-s) \leq 1/4 $ and Jensen's inequality, the expression
$$
\int_0^1\frac{\big(\int_0^{s}D(s')\big(s'-\delta\big)^Rds'-s\int_0^1D(s'')\big(s''-\delta\big)^Rds''\big)^2}{s(1-s)}ds
$$
is at least
$$
4\big(\int_0^1\big(\int_0^{s}D(s')\big(s'-\delta\big)^Rds'-s\int_0^1D(s'')\big(s''-\delta\big)^Rds''\big)ds \big)^2,
$$
which by Fubini's theorem is equal to
\begin{equation} \label{eq:lowerdiv}
4\big(\int_0^1D(s)(s-1/2)\big(s-\delta\big)^Rds\big)^2.
\end{equation}
The leading terms in $ \delta $ in the integrand of \prettyref{eq:lowerdiv} are, up to signs, $ \delta^R\int_0^1D(s)(s-1/2)ds $ and $ R\delta^{R-1}\int_0^1D(s)s(s-1/2)ds $. However, note that
$$
\delta^R\Big|\int_0^1D(s)(s-1/2)ds\Big| = \delta^R\Big|\int_0^1\Big(D(s)-\frac{\overline F^{(R)}(c; \bt)}{R!}\Big)(s-1/2)ds\Big|,
$$
which, per \prettyref{eq:cont}, is at most $ \delta^R\int_0^1|D(s)-\frac{\overline F^{(R)}(c; \bt)}{R!}||s-1/2|ds \leq \delta^{R-2}/4 $. Furthermore, per \prettyref{eq:cont}, $ D(s) $ is between $ \frac{\overline F^{(R)}(c; \bt)}{2R!} $ and $ \frac{3\overline F^{(R)}(c; \bt)}{2R!} $. Hence, we have that $ \int_0^1D(s)s(s-1/2)ds \geq \frac{1}{48}\frac{\overline F^{(R)}(c; \bt)}{R!} > 0 $. This shows that \prettyref{eq:lowerdiv} is $ \Omega(|\delta|^{R-1}) $ as $ |\delta| \rightarrow +\infty $. Hence, we may assume that $ \delta = \frac{c-a}{b-a} $ is bounded in the limit infimum \prettyref{eq:deltalimit}. Using this, we have that $ \liminf_{(a,b)\to(c,c)}\frac{\Delta(s^*; \bt)}{\big(\frac{(b-a)^R\overline F^{(R)}(c; \bt)}{R!}\big)^2} $ is at least
\begin{equation} \label{eq:inf}
\inf_{\delta\in\mathbb{R}}\Big\{\int_0^1 \frac{\big( \int_0^s (s'-\delta)^Rds'-s\int_0^1(s''-\delta)^Rds'' \big)^2}{s(1-s)}ds\Big\}.
\end{equation}

Next, let us evaluate the infimum \prettyref{eq:inf}. In fact, we will show that it is achieved at $ \delta = 1/2 $. 

Towards this end, elementary calculations reveal that the expression inside the infimum \prettyref{eq:inf} is equal to
\begin{equation} \label{eq:express}
\frac{1}{(R+1)^2}\int_0^1 \frac{\big(s(\delta-1)^{R+1}+(1-s)\delta^{R+1}-(\delta-s)^{R+1}\big)^2}{s(1-s)}ds.
\end{equation}

Define $ V_R(s, \delta) = s(\delta-1)^{R}+(1-s)\delta^{R}-(\delta-s)^{R} $, so that the integral \prettyref{eq:express} can be written as
\begin{equation} \label{eq:expressV}
\frac{1}{(R+1)^2}\int_0^1 \frac{V^2_R(\delta, s)}{s(1-s)}ds.
\end{equation}
We first catalogue some facts about $ V_R $.
\begin{enumerate}
\item If $ R $ is even, then $ V_R(\delta, s) \geq 0 $.
\item $ V_R(\delta, s) = (-1)^RV_R(1-\delta, 1-s) $.
\item If $ R $ is even, $ \delta \geq 1/2 $, and $ 0 \leq s \leq 1/2 $, then $ V_R(\delta, s) \geq V_R(\delta, 1-s) $.
\item $ \frac{\partial}{\partial \delta} V_R(\delta, s) = RV_{R-1}(\delta, s) $.
\item $ \frac{\partial}{\partial \delta} V^2_R(\delta, s) = 2RV_R(\delta, s)V_{R-1}(\delta, s) $.
%\item $ \frac{\partial^2}{\partial \delta^2} V^2_R(\delta, s) = 2R^2V^2_{R-1}(\delta, s) + 2R(R-1)V_R(\delta, s)V_{R-2}(\delta, s) $.
\end{enumerate}
By the second fact and the representation \prettyref{eq:expressV}, it follows that \prettyref{eq:express} is symmetric about $ \delta = 1/2 $. Thus, it can be assumed that $ \delta \geq 1/2 $.

Using the fifth fact, we have that the derivative of \prettyref{eq:express} with respect to $ \delta $ is
\begin{equation} \label{eq:express2}
\frac{2}{R+1}\int_0^1 \frac{V_{R+1}(\delta, s)V_{R}(\delta, s)}{s(1-s)}ds.
\end{equation}
Assume without loss of generality that $ R $ is even. By the first and fourth facts, $ V_{R+1}(\delta, s) $ is increasing in $ \delta $ and $ V_R $ is nonnegative. Hence \prettyref{eq:express2} is at least
\begin{equation} \label{eq:express3}
\frac{2}{R+1}\int_0^1 \frac{V_{R+1}(1/2, s)V_{R}(\delta, s)}{s(1-s)}ds.
\end{equation}
Note that $ V_{R+1}(1/2, s) = (1-2s - (1-2s)^{R+1})/2^{R+1} $ is odd about $ s = 1/2 $ and positive when $ s \leq 1/2 $. Using this observation, we have that \prettyref{eq:express3} is equal to
$$
\frac{2}{R+1}\int_0^{1/2}\frac{V_{R+1}(1/2, s)(V_{R}(\delta, s)-V_{R}(\delta, 1-s))}{s(1-s)}ds, 
$$
and nonnegative by the third fact.
Therefore, \prettyref{eq:express} is increasing when $ \delta \geq 1/2 $ and hence minimized at $ \delta = 1/2 $. Thus, the infimum \prettyref{eq:inf} is equal to
\begin{equation} \label{eq:intdelta}
\int_0^1 \frac{\big( \int_0^s (s'-1/2)^Rds'-s\int_0^1(s''-1/2)^Rds'' \big)^2}{s(1-s)}ds = \int_0^1\Delta(s, [0, 1])ds.
\end{equation}
Routine calculations also reveal that \prettyref{eq:intdelta} is $ \Omega(4^{-R}/R^2) $.

Finally, let us verify all five facts above.
The second, fourth, and fifth facts are straightforward. The first fact holds since $ V_R $ is the difference between a point and a chord that lies above it on the convex function $ s \mapsto s^{R} $. To show the third fact, note that
$$
\frac{\partial^2}{\partial s^2}(V_{R}(\delta, s)-V_{R}(\delta, 1-s)) = R(R-1)((1-\delta-s)^{R-2}-(\delta-s)^{R-2}),
$$
which is negative since $ \delta-s \geq |1-\delta-s| $ for $ \delta \geq 1/2 $ and $ s \leq 1/2 $. Since $ V_{R}(\delta, s)-V_{R}(\delta, 1-s) $ has roots at $ s = 0 $ and $ s = 1/2 $, it follows that $ V_{R}(\delta, s) \geq V_{R}(\delta, 1-s) $ in this regime.
\end{proof}

\begin{remark}
The same argument also works if the uniform prior $ \Pi $ is replaced by any symmetric prior about $ 1/2 $.
\end{remark}

\begin{proof}[Proof of \prettyref{thm:sample}]
Combining \prettyref{eq:length} and \prettyref{eq:close}, it can be deduced via the triangle inequality that
\begin{equation} \label{eq:sineq}
%\left|\hat s -\frac{a+b}{2}\right| \leq \frac{b-a}{2}(1-\Gamma^2),
\frac{b-a}{2}\Gamma^2 + a \leq \hat s \leq b - \frac{b-a}{2}\Gamma^2,
\end{equation}
with probability at least $ 1-\delta $. Next, by \prettyref{eq:p} and continuity of $ \bX $, we have that
$$
p = \prob{X \leq Q(p) \mid \bX \in \bt}.
$$
%Next, we use the Galois inequalities,
%$$
%Q(p) \leq s \quad \text{implies} \quad \prob{X \leq s \mid \bX \in \bt} \geq p,
%$$
%and
%$$
%Q(p) > s \quad \text{implies} \quad \prob{X > s \mid \bX \in \bt} > 1- p,
%$$
We use this identity to derive lower bounds on $ P(\bt_L) $ and $ P(\bt_R) $ over split points $ s \in [a+c,b-c] $ with $ c = \frac{b-a}{2}\Gamma^2 $. Now, from \prettyref{eq:sineq}, we have that
\begin{align*}
s & \geq a + \frac{b-a}{2}\Gamma^2 \\
& = a + (b-a)q_2(p_L) \\
& \geq Q(p_L),
\end{align*}
and hence $ P(\bt_L) = \prob{X \leq s \mid \bX \in \bt} \geq \prob{X \leq Q(p_L) \mid \bX \in \bt} = p_L $. Furthermore,
\begin{align*}
s & \leq b - \frac{b-a}{2}\Gamma^2 \\
& = b - (b-a)(1-q_1(1-p_R)) \\
& \leq Q(1-p_R),
\end{align*}
and hence $ P(\bt_R) = \prob{X > s \mid \bX \in \bt} \geq \prob{X > Q(1-p_R) \mid \bX \in \bt} = p_R $.
We deduce that if $ s\in [a+c, b-c ] $, then 
\begin{equation} \label{eq:problower}
P(\bt_L) \geq p_L \quad \text{and} \quad P(\bt_R) \geq p_R.
\end{equation}

By the assumption that $ \bt $ is independent of the training data $ \calD_n $ and \prettyref{lmm:cond} in \prettyref{app:proofs}, given $ N(\bt) $, $ N(\bt) \widehat P(\bt_L) $ is distributed $ \Binom(N(\bt) , P(\bt_L)) $ and $ N(\bt) \widehat P(\bt_R) $ is distributed $ \Binom(N(\bt) , P(\bt_R)) $ (the features of $ \bX $ need not be independent for this to hold). If $ W \sim X \mid \bX \in \bt $, then given $ N(\bt) $, $ \widehat P(\bt_L) $ has the same distribution as the empirical distribution function of a sample $ \{W_i\}_{i=1}^{N(\bt)} $ from the distribution of $ W $. Hence by the Dvoretzky-Kiefer-Wolfowitz inequality \citep{massart1990tight},
%appealing to VC-theory for one-dimensional half-spaces $ \{x: x \leq s \} $ and $ \{x: x > s \} $, 
it can be shown that with probability at least $ 1-2\exp\{-2N\epsilon^2 p^2\} $, uniformly over all split points $ s \in [a+c,b-c] $,
$$
\widehat P(\bt_L) \geq (1-\epsilon) P(\bt_L) \quad \text{and} \quad \widehat P(\bt_R) \geq (1-\epsilon) P(\bt_R),
$$
where $ \epsilon \in (0, 1/2) $. In particular, for $ \bt_L = \hat \bt_L $ and $ \bt_R = \hat \bt_R $, with probability at least $$ 1-\delta - 2\exp\{-Np^2/2\}, $$
\begin{equation} \label{eq:frac}
\widehat P( \hat \bt_L) \geq \frac{1}{2}p_L \quad \text{and} \quad \widehat P( \hat \bt_R) \geq \frac{1}{2}p_R.
% \sim \frac{1}{256\Lambda^4}, \quad \text{as}\;\; \Lambda \rightarrow +\infty.
\end{equation}
Since $ \widehat P(\bt_L) = N(\bt_L)/N(\bt) $ and $ \widehat P(\bt_R) = N(\bt_R)/N(\bt) $, the quantities in \prettyref{eq:frac} are interpretable as lower bounds on the fraction of data points in the optimal daughter node ($\hat \bt_L$ and $\hat \bt_R$) that are contained in the parent node $ \bt $. A consequence of this analysis is that if $ \Lambda $ is not too small, this fraction is non-negligible. 
%Intuitively, this means that if one desires to predict $ \expect{Y \mid \bX = \bx} $ at a new point $ \bx $, then if $ A_n(\bx, \Theta) $ denotes the cell containing $ \bx $, we would like $ \Prob_{\bX, \Theta}[\bX \in A_n(\bx, \Theta)] $ to be sufficiently small.
\end{proof}

\begin{proof}[Proof of \prettyref{thm:sampledata}]
The proof is similar to \prettyref{thm:sample}, although we must use a stronger concentration inequality to control an empirical process over a collection of nodes. That is, the proof is based on simultaneous control of the empirical processes
$$
\left\{\frac{1}{n}\sum_{i=1}^n \indc{\bX_i \in \bt} :  \bt\in\calT \right\} \quad \text{and} \quad \left\{\frac{1}{n}\sum_{i=1}^n \indc{\bX_i \in \bt, \; X \leq s} : \bt\in\calT, \; s \in [0, 1] \right\},
$$
where $ \calT $ is the collection of all nodes in $ d $ dimensions.
To this end, define $ P(\bt) = \prob{\bX \in \bt} $ and $ P(s, \bt) = \prob{\bX \in \bt, \; X \leq s} $ so that $ P(\bt_L) = P(s, \bt)/P(\bt) $. We make use of the inequality
\begin{equation} \label{eq:triangle}
| \widehat P(\bt_L) - P(\bt_L)| \leq \frac{| N(\bt_L)/n-P(s, \bt_L)|}{ N(\bt) /n} + \frac{| N(\bt) /n-P(\bt)|}{ N(\bt) /n},
\end{equation}
which can be deduced from the triangle inequality. We would like to obtain an upper bound on the probability that
\begin{equation} \label{eq:event}
\widehat P( \hat \bt_L) < (1-\epsilon)p_L \quad \text{or} \quad \widehat P( \hat \bt_R) < (1-\epsilon)p_R,
\end{equation}
for $ \epsilon = 1/2 $.
 On an event with probability at least $ 1-\delta $, it holds that $ N(\bt) \geq n \alpha $ and $ P( \hat \bt_L) \geq p_L $ and $ P( \hat \bt_R) \geq p_R $ (using \prettyref{eq:problower} from the proof of \prettyref{thm:sample}), and hence \prettyref{eq:event} is contained in the event
$$
| \widehat P( \hat \bt_L) - P( \hat \bt_L)| \geq \epsilon p,
$$
with probability at least $ 1-\delta $, where $ p = \min\{ p_L, p_R \} $. Using \prettyref{eq:triangle}, this event is also contained in
$$
|N(\hat \bt_L)/n-P(\hat s, \bt)| \geq \epsilon \alpha p/2 \quad \text{or} \quad |N(\bt) /n-P(\bt)| \geq \epsilon \alpha p/2,
$$
with probability at least $ 1-\delta $.
By \prettyref{lmm:prob}, each event above has probability at most $ 8(n^{2d}+1)\exp\{-n\epsilon^2 \alpha^2 p^2/128\} $ for a total probability of at most $ 16(n^{2d}+1)\exp\{-n\epsilon^2 \alpha^2 p^2/128\} $. The proof is completed by choosing $ \epsilon = 1/2 $.
\end{proof}

\begin{lemma} \label{lmm:prob}
Let $ \calT $ be the set of all nodes in $ \mathbb{R}^{d} $. Let $ \bX $ and $ \{ \bX_i \}_{i=1}^n $ be i.i.d. random vectors in $ \mathbb{R}^{d} $. Then for all $ \epsilon > 0 $,
$$
\prob{\sup_{\bt \in \calT}\left|\frac{1}{n}\sum_{i=1}^n \indc{\bX_i \in \bt} - \prob{\bX \in \bt} \right| > \epsilon } \leq 8s(\calT, n)\exp\{-n\epsilon^2/32\},
$$
where $ s(\calT, n) \leq n^{2d}+1 $.
\end{lemma}

\begin{proof}[Proof of \prettyref{lmm:prob}]
This follows from \citep{vapnik1971uniform} and the fact that the VC-dimension of $ \calT $ is $ 2d $.
\end{proof}

\subsection{Proofs of example regression functions}
In this subsection, we give proofs of \prettyref{ex:poly}, \prettyref{ex:sine}, and \prettyref{ex:friedman} from \prettyref{sec:mdilower} and \prettyref{ex:logistic} from \prettyref{sec:classification}.
%\begin{theorem} \label{thm:poly}
%Suppose $ f(\bx) = \sum_{j=1}^d c_j x_j^{k_j} $ for integer $ k_j \geq 1 $. Then
%$$
%\Lambda_j \geq \left(\frac{1}{k_j(k_j+1)}\right)^{2/3}.
%$$
%Consequently, the diameter of the largest cell shrinks to zero when $ M/k \rightarrow +\infty $.
%\end{theorem}
\begin{proof}[Proof of \prettyref{ex:poly}]
Without loss of generality, we will prove the theorem when $ f(x) = x^k $. The objective function $ \Delta(s; \bt) $ can be expressed as
\begin{align*}
\sqrt{\Delta(s; \bt)} & = \frac{1}{(k+1)\sqrt{(s-a)(b-s)}}\left[ \frac{b-s}{b-a}a^{k+1} + \frac{s-a}{b-a}b^{k+1} - s^{k+1} \right].
%& = \frac{(s-a)}{(k+1)}\left[ \sum_{j=1}^k[a^{k-j}(s^j -b^j)]\right].
\end{align*}
%with maximizer
%$$
%y^* = \left(\frac{b^{k+1}-a^{k+1}}{(b-a)(k+1)}\right)^{1/k} = b\left(\frac{\sum_{j=0}^k (a/b)^{k-j}}{k+1}\right)^{1/k}.
%$$
Note that $ \sqrt{\Delta(s^*; \bt)} $ is at least
$$
\sqrt{\Delta((a+b)/2; \bt)} = \frac{b-a}{k+1}\sum_{j=0}^{k-1} (1-2^{-(j+1)})a^{k-1-j}(b-a)^j \binom{k+1}{j+2},
$$
%$$
%|\overline G(b)| = b^k - \frac{b^{k+1}-a^{k+1}}{(b-a)(k+1)} = \frac{b-a}{k+1}\sum_{j=0}^{k-1}a^{k-1-j}(b-a)^j(j+1)\binom{k+1}{j+2},
%$$
%and
%$$
%|\overline G(a)| = \frac{b^{k+1}-a^{k+1}}{(b-a)(k+1)} - a^k = \frac{b-a}{k+1}\sum_{j=0}^{k-1}a^{k-1-j}(b-a)^j\binom{k+1}{j+2}.
%$$
and the derivative of the partial dependence function at $ s^* $ is
$$
\overline F'(s^*; \bt) \leq \overline F'(b; \bt) = kb^{k-1}.
$$
Thus, by \prettyref{eq:lambdaloweruniversal},
%\begin{align*}
%\lambda^{3/2} & \geq \frac{2\sqrt{\Delta((a+b)/2; \bt)}}{(b-a)|\overline F_j(b)|} \\
%& = \frac{\sum_{j=0}^{k-1} (1-2^{-(j+1)})a^{k-1-j}(b-a)^j \binom{k+1}{j+2}}{\sum_{j=0}^{k-1}a^{k-1-j}(b-a)^j(j+1)\binom{k+1}{j+2}} \\
%& \geq \frac{1}{2k}.
%\end{align*}
\begin{align*}
\lambda^{3/2} & \geq \frac{2\sqrt{\Delta((a+b)/2; \bt)}}{(b-a)|\overline F'(b; \bt)|} \\
& = \frac{2\sum_{j=0}^{k-1} (1-2^{-(j+1)})a^{k-1-j}(b-a)^j \binom{k+1}{j+2}}{k(k+1)b^{k-1}} \\
& \geq \frac{\sum_{j=0}^{k-1}(a/b)^{k-1-j}(1-a/b)^j \binom{k-1}{j}}{k(k+1)} \\
& = \frac{1}{k(k+1)}.
\end{align*}
The penultimate line above follows from the inequality $ \binom{k+1}{j+2} = \frac{k(k+1)}{(j+1)(j+2)}\binom{k-1}{j} \geq \binom{k-1}{j} $ for $ j = 0, 1, \dots, k-1 $ and the binomial theorem.
Since $ a < b $ is arbitrary, it follows that $ \Lambda \geq \left(\frac{1}{k(k+1)}\right)^{2/3} $.
\end{proof}

\begin{remark} \label{rmk:poly}
If $ a = 0 $ and $ b = 1 $, note that if $ s = k/(k+1) $, then $ \sqrt{\Delta(s; \bt)} \sim (1-e^{-1})/\sqrt{k} $ as $ k \rightarrow + \infty $. Since $ \sup_{s\in[0,1]}|\overline F'(s; \bt)| \leq k $, this shows that $ \lambda([0,1]) \geq (1-e^{-1})^{2/3}/k $. Although we have not given a formal proof, this calculation provides evidence that $ \Lambda = \Omega(1/k) $.
\end{remark}

%The next theorem shows how ``wiggly'' functions affect $ \Lambda $. It shows that $ \Lambda $ is superlinear in the angular frequency of a sinusoidal waveform.

%\begin{theorem} \label{thm:sine}
%Suppose $ f(\bx) = \sum_{j=1}^d c_j \sin(2\pi m_j x_j) $ for integer $ m_j \geq 1 $. There exists a universal constant $ C > 0 $ such that
%$$
%\Lambda_j \geq \frac{C}{m_j^{4/3}}.
%$$
%Consequently, the diameter of the largest cell shrinks to zero when $ M/m \rightarrow +\infty $.
%\end{theorem}

%\begin{remark}
%Similar conclusions can be made for other smooth oscillatory functions.
%\end{remark}

\begin{proof}[Proof of \prettyref{ex:sine}]
Without loss of generality, we will prove the theorem when $ f(x) = \sin(2\pi mx) $. In this case, $ \sqrt{\Delta(s; \bt)} $ is equal to
\begin{align*}
\frac{1}{2\pi m \sqrt{(s-a)(b-s)}}|\underbrace{\cos(2\pi m s) - \frac{b-s}{b-a}\cos(2\pi m a) - \frac{s-a}{b-a}\cos(2\pi m b)}_{\text{Jensen gap}}|.
%& = \frac{1}{\pi m}\left[ \sin(\pi m(s-a))\cos(\pi m(s+a)) - \frac{s-a}{b-a}(\sin(\pi m(b-a))\cos(\pi m(a+b))) \right] \\
%& = \frac{1}{\pi m}\left[ \sin(\pi m(s-b))\cos(\pi m(s+b)) - \frac{s-b}{b-a}(\sin(\pi m(b-a))\cos(\pi m(a+b))) \right] \\
%& = \frac{1}{\pi m}\left[ \frac{b-s}{b-a}\sin(\pi m(s-a))\cos(\pi m(s+a)) - \frac{s-a}{b-a}\sin(\pi m(b-s))\cos(\pi m(s+b))) \right].
\end{align*}
Note the term which is equal to the difference of $ \cos(2\pi m s) $ and the line segment between $ \cos(2\pi m a) $ and $ \cos(2\pi m b) $ at $ s $, or the so-called ``Jensen gap''.\footnote{Typically, this terminology is reserved for convex or concave functions, but we nevertheless use it here.} A major task of the proof is in choosing a suitable $ s $ so that the Jensen gap is large.
%I think you should try $ x = a + 1/(8m) $ if $ b-a > 1/m $ and $ a+b < 1/(8m) $; $ x = b - 1/(8m) $ if $ b-a > 1/m $ and $ a+b > 1/(8m) $; and $ x = (a+b)/2 $ if $ b-a < 1/m $ and $ a+b < 1/(8m) $. Maybe try $ x = a+(b-a)/4 $ if $ b-a < 1/m $ and $ a+b > 1/(8m) $. 
Next, note that
\begin{equation} \label{eq:singap}
\sup_{s\in[a,b]}|\overline G(s; \bt)| \leq \min\left\{1, 2\pi m\int_{a}^b |\cos(2\pi m s)|ds \right\},
\end{equation}
which combines a pointwise inequality of one and the total variation inequality \prettyref{eq:TV} from \prettyref{lmm:variation}.

We break the proof into two parts, depending on whether $ a-b \leq 1/(4m) $ or $ a-b > 1/(4m) $. 

\paragraph{{\bf Case I: $ a-b \leq 1/(4m) $}} Suppose that $ b-a \leq 1/(4m) $ and that $ \sin(2\pi m s) $ is monotone on $ [a, b] $. Consider $ \sqrt{\Delta(s; \bt)} $ at $ s = (a+b)/2 $. Then it can be shown that
\begin{equation} \label{eq:sinmid}
\sqrt{\Delta((a+b)/2; \bt)} = \frac{2\sin^2(\pi m(b-a)/2)|\cos(\pi m(a+b))|}{\pi m(b-a)}.
\end{equation}
Since $ \sin(2\pi m s) $ is monotone on $ [a,b] $, it follows that its total variation is equal to $ |\sin(2\pi m b)- \sin(2\pi m a)| $ and hence
\begin{align}
\sup_{s\in[a,b]}|\overline G(s; \bt)| & \leq |\sin(2\pi m b) - \sin(2 \pi m a)| \nonumber \\ & = 2|\sin(\pi m(b-a) )\cos(\pi m(a+b))|. \label{eq:sinvar}
\end{align}
Combining the estimates \prettyref{eq:sinmid} and \prettyref{eq:sinvar} and using $ |\sin(z)| \leq 2|\sin(z/2)| $ for $ z\in [0, 2\pi] $ and $ \sin(z) \geq (2\sqrt{2}/\pi)z $ for $ z\in [0, \pi/4] $, we conclude from \prettyref{eq:lambdaosc} that
\begin{align*}
\lambda^{1/2} & \geq \frac{\frac{2\sin^2(\pi m(b-a)/2)|\cos(\pi m(a+b))|}{\pi m(b-a)}}{2|\sin(\pi m(b-a) )\cos(\pi m(a+b))|} \\
& \geq \frac{|\sin(\pi m(b-a)/2 )|}{2\pi m(b-a)} \\
& \geq \frac{\sqrt{2}}{2\pi}.
\end{align*}
Thus, $ \lambda \geq \frac{1}{2\pi^2} $.
Next, suppose that $ \sin(2 \pi m s) $ is neither increasing or decreasing on $ [a, b] $. This means that for some positive integer $ k $, the point $ s' = (2k-1)/(4m) $ belongs to $ [a, b] $. Thus, there are choices $ s \in [a,b] $ such that
$$
|\cos(2\pi m s) - \frac{b-s}{b-a}\cos(2\pi m a) - \frac{s-a}{b-a}\cos(2\pi m b)| \geq Cm^3(b-a)^3.
$$ 
Since $ \sqrt{(s-a)(b-s)} \leq (b-a)/2 $, it follows from \prettyref{eq:lambdaosc} that $ \sqrt{\Delta(s; \bt)} \geq (C/\pi)m^2(b-a)^2 $.
Moreover, using the total variation bound in \prettyref{eq:singap},
\begin{align*}
\sup_{s\in[a,b]}|\overline G(s; \bt)| & \leq |2 - \sin(2\pi m b) - \sin(2 \pi m a)| \\ & = |[\sin(2\pi m s') - \sin(2\pi m b)] + [\sin(2\pi m s') - \sin(2 \pi m a)]| \\ & \leq 2\pi^2m^2[(b-s')^2+(s'-a)^2] \\
& \leq 2\pi^2m^2(b-a)^2,
\end{align*}
where the penultimate line follows from a Taylor expansion argument. Thus, by \prettyref{eq:lambdaosc}, $ \lambda^{1/2} \geq \frac{(C/\pi)m^2(b-a)^2}{2\pi^2m^2(b-a)^2} = C/(2\pi^3) $ and hence $ \lambda \geq C^2/(4\pi^6) $.

\paragraph{{\bf Case II: $ a-b > 1/(4m) $}} Next, we investigate when $ b-a > 1/(4m) $. In this regime, $ \cos(2\pi m s) $ is allowed to make at least a quarter period on $ [a, b] $. This means that there exists a universal constant $ C>0 $ and $ s \in [a, b] $ such that the Jensen gap
$$
|\cos(2\pi m s) - \frac{b-s}{b-a}\cos(2\pi m a) - \frac{s-a}{b-a}\cos(2\pi m b)| \geq C,
$$ 
and hence $ \sqrt{\Delta(s; \bt)} \geq C/m $. Furthermore, $ \sup_{s\in[a,b]}|\overline F'(s; \bt)| \leq 2\pi m $. Thus, by \prettyref{eq:lambdaloweruniversal}, $ \lambda \geq (C^2/(\pi^2m^4))^{1/3} $.
%Note that $ \lambda \leq (\pi/2)\sqrt{m} $. My feeling is that $ \lambda \leq \sqrt{Cm} $ for all $ a < b $, but I don't know how to prove it at this point.
\end{proof}

\begin{remark}
If $ a = 0 $ and $ b = 1 $, and $ s = 1/m $ then $ \sqrt{\Delta(s; \bt)} = \frac{1}{2\pi}\sqrt{\frac{1-1/m}{m}} $. Furthermore, $\sup_{s\in[0,1]}|\overline G(s; \bt)| \leq 1 $. Together these estimates imply that $ \lambda([0,1]) \geq \frac{1-1/m}{4\pi^2m} $. As with the case of polynomials (c.f., \prettyref{rmk:poly}), it is likely that the bound is improvable to $ \Lambda = \Omega(1/m) $, and we leave its proof as an open question for future investigation.
\end{remark}

\begin{lemma} \label{lmm:square}
Suppose $ f(x) = (x-1/2)^2 $. Then $ \Lambda \geq (\frac{1}{12})^{2/3} $.
\end{lemma}

\begin{proof}
In this case,
$$
\sqrt{\Delta(s; \bt)} = \frac{1}{6}\sqrt{(s-a)(b-s)}|2a+2b+2s-3|.
$$
%Using the total variation bound $ \sup_{s\in[a,b]}|\overline G(s; \bt)| \leq \int_a^b |\overline F'(s; \bt)|ds $, if $ b \leq 1/2 $ or $ a \geq 1/2 $, we have
The derivative of the partial dependence function is $ \overline F'(s; \bt) = 2s-1 $. If $ b \leq 1/2 $ or $ a \geq 1/2 $, then a lower bound on $ \lambda $ is easy to state. In this case, $ |\overline F'(s; \bt)| \leq |a+b-1| $.

%$$
%\sup_{s\in[a,b]}|\overline G(s; \bt)| \leq (b-a)|a+b-1|.
%$$
Choosing $ s = (a+b)/2 $, we have from \prettyref{eq:lambdaloweruniversal} that
$$
\lambda^{3/2} \geq \frac{2\sqrt{\Delta((a+b)/2; \bt)}}{(b-a)\max_{s\in[a,b]}|\overline F'(s; \bt)|} \geq \frac{\frac{1}{2}(b-a)|a+b-1|}{(b-a)|a+b-1|} = 1/2.
$$
%then the total variation bound becomes
%$$
%\sup_{s\in[a,b]}|\overline G(s; \bt)| \leq ((1/2-a)^2+(b-1/2)^2) \leq \sqrt{2}(b-a)\max\{1/2-a, b-1/2\}.
%$$
If $ b > 1/2 $ or $ a < 1/2 $, then $ |\overline F'(s; \bt)| \leq 2\max\{1/2-a, b-1/2\} $ and there are two cases to consider for obtaining lower bounds on $ \Delta(s; \bt) $.
\paragraph{{\bf Case I: $ a + b \geq 1 $}}
Choose $ s = a + \delta(b-a) $, where $ \delta > 1/2 $. Then,
$$
\sqrt{\Delta(s; \bt)} \geq \frac{b-a}{3}\sqrt{\delta(1-\delta)}(2\delta-1)(b-1/2).
$$
Let $ \delta = \frac{1}{2} + \frac{1}{2\sqrt{2}} $. Then, $ \sqrt{\Delta(s; \bt)} \geq \frac{b-a}{12}(b-1/2) $, and hence by \prettyref{eq:lambdaloweruniversal}, $ \lambda \geq (\frac{1}{12})^{2/3} $.
\paragraph{{\bf Case II: $ a + b < 1 $}} Choose $ s = a + \delta(b-a) $, where $ \delta < 1/2 $. Then,
$$
\sqrt{\Delta(s; \bt)} \geq \frac{b-a}{3}\sqrt{\delta(1-\delta)}(1-2\delta)(1/2-a).
$$
Choosing $ \delta = \frac{1}{2} - \frac{1}{2\sqrt{2}} $ yields $ \sqrt{\Delta(s; \bt)} \geq \frac{b-a}{12}(1/2-a) $, and hence \prettyref{eq:lambdaloweruniversal}, $ \lambda \geq (\frac{1}{12})^{2/3} $.
\end{proof}

%Suppose $ g : \mathbb{R}^d \to \mathbb{R} $ is a continuous and bounded function and let $ f(\bx) = g(A \bx) $ for $ A $ invertible and $ \bx \in [0, 1]^d $. In this case, it is likely true that $ \Lambda \leq C|\det(A)| $ or $ \Lambda \leq C\sqrt{|\det(A)|} $.

%\subsubsection{More general models}

%Since a linear combination of sine and cosine waves is equivalent to a single sine wave with a phase shift and scaled amplitude, it follows that if $ d = 2 $ and $ f(\bx) = \sin(\langle \btheta, \bx \rangle) = \sin(\theta_1x_1)\cos(\theta_2 x_2) + \sin(\theta_2 x_2)\cos(\theta_1 x_1) $, then $ \expect{Y \mid \bX \in \bt, \; X = s} = c_1\sin(\theta_1 x_1+c_2) $, where $ c_1 = \sqrt{(\int_{a_2}^{b_2}\cos(\theta_2 x_2)dx_2)^2 + (\int_{a_2}^{b_2}\sin(\theta_2 x_2)dx_2)^2} $ and $ \tan c_2 = \frac{\int_{a_2}^{b_2}\sin(\theta_2 x_2)dx_2}{\int_{a_2}^{b_2}\cos(\theta_2 x_2)dx_2} $. Hence, when the norm of $ \btheta $ is large, we can expect a similar conclusion to \prettyref{thm:sine}.

\begin{lemma} \label{lmm:twodim}
Suppose $ f(\bx) = \sin(\pi x_1 x_2) $. Then $ \Lambda \geq \frac{1}{2\pi^2} $.
\end{lemma}
\begin{proof}
First, note that $ \expect{Y \mid \bX \in \bt,\; X_1 = s} $ is equal to
\begin{align*} -\frac{\cos(\pi b_2s)-\cos(\pi a_2s)}{(b_2-a_2)\pi s} & = \frac{\sin(\pi(a_2+b_2)s/2)\sin(\pi(b_2-a_2)s/2)}{(b_2-a_2)\pi s/2} \\ & \approx \sin(\pi(a_2+b_2)s/2). \end{align*}
In this case,
$$
 \sqrt{\Delta((a_1+b_1)/2; \bt)} = \frac{\sin^2(\pi(a_2+b_2)(b_1-a_1)/8)\cos(\pi(a_2+b_2)(b_1+a_1)/4)}{\pi(a_2+b_2)(b_1-a_1)/8}.
$$

Using the total variation bound \prettyref{eq:singap}, one can easily show that 
\begin{align*} \sup_{s\in[a_1,b_1]}|\overline G(s; \bt)| & \leq 2\sin(\pi(a_2+b_2)(b_1-a_1)/4)\cos(\pi(a_2+b_2)(b_1+a_1)/4)
 \\ & \leq 4\sin(\pi(a_2+b_2)(b_1-a_1)/8)\cos(\pi(a_2+b_2)(b_1+a_1)/4),
\end{align*}
where the last line follows from $ \sin(z) \leq 2\sin(z/2) $ for $ z\in [0, \pi] $. Using $ \sin(z) \geq (2\sqrt{2}/\pi)z $ for $ z\in [0, \pi/4] $, we conclude from \prettyref{eq:lambdaosc} that
$$
\lambda^{1/2} \geq \frac{\sqrt{\Delta((a_1+b_1)/2; \bt)}}{\sup_{s\in[a_1,b_1]}|\overline G(s; \bt)|} \geq \frac{1}{4}\frac{\sin(\pi(a_2+b_2)(b_1-a_1)/8)}{\pi(a_2+b_2)(b_1-a_1)/8} \geq \frac{\sqrt{2}}{2\pi}.
$$
This implies that $ \Lambda \geq \frac{1}{2\pi^2} $.
\end{proof}

\begin{proof}[Proof of \prettyref{ex:logistic}]
Consider a generic coefficient of $ \bbeta $, say $ \beta $. First, note that we may assume without loss of generality that $ \beta > 0 $ (since otherwise, we can consider $ \prob{Y=-1\mid \bX = \bx} = \frac{1}{1+e^{\beta_0+\langle \bx, \bbeta \rangle}} $).
In this case, $ \sqrt{\Delta(s^*; \bt)} $ is at least
\begin{align}
\sqrt{\Delta((a+b)/2; \bt)} & = \frac{1}{\beta}\int_{0}^{\beta}[\overline F(a+x(b-a)/(2\beta); \bt)- \overline F(a+x(b-a)/\beta; \bt)]dx\nonumber \\
& \geq \frac{1}{\beta^2}\int_{0}^{\beta}\overline F'(a+x(b-a)/(2\beta); \bt)(1-e^{-x(b-a)/2})dx \nonumber \\
& \geq \frac{\overline F'(a; \bt)}{\beta^2}\int_{0}^{\beta}e^{-x(b-a)/2}(1-e^{-x(b-a)/2})dx, \label{eq:logistic1}
\end{align}
and
\begin{equation} \label{eq:logistic2}
\overline F'(s^*; \bt) \leq \overline F'(a; \bt).
\end{equation}
The above inequality \prettyref{eq:logistic2} is due to the fact that $ \overline F(\cdot; \bt) $ is decreasing.
Combining inequalities \prettyref{eq:logistic1} and \prettyref{eq:logistic2} with \prettyref{eq:lambdaloweruniversal}, we have that
\begin{equation} \label{eq:logistic3}
\lambda^{3/2} \geq \frac{2\sqrt{\Delta((a+b)/2; \bt)}}{(b-a)|\overline F'(a; \bt)|} \geq \frac{2}{(b-a)\beta^2}\int_0^{\beta}e^{-x(b-a)/2}(1-e^{-x(b-a)/2})dx.
\end{equation}
We now consider two cases for evaluating the integral in \prettyref{eq:logistic3}.
\paragraph{{\bf Case I: $ \beta(b-a) \geq 2\log 2 $}}
In this regime,
\begin{align*}
\int_0^{\beta}e^{-x(b-a)/2}(1-e^{-x(b-a)/2})dx & \geq \frac{1}{2}\int_0^{\beta}e^{-x(b-a)/2}dx \\
& = \frac{(1-e^{-\beta(b-a)/2})}{b-a} \\
& \geq 1/2.
\end{align*}
Thus, $ \lambda \geq \beta^{-4/3} $.

\paragraph{{\bf Case II: $ \beta(b-a) < 2\log 2 $}}
Alternatively, in this regime,
\begin{align*}
\int_0^{\beta}e^{-x(b-a)/2}(1-e^{-x(b-a)/2})dx & \geq \frac{1}{2}\int_0^{\beta}(1-e^{-x(b-a)/2})dx \\
& = \frac{1}{2}(\beta-\frac{2}{(b-a)}(1-e^{-\beta(b-a)/2})) \\
& \geq \frac{\beta^2(b-a)}{16}.
\end{align*}
Thus, $ \lambda \geq (1/8)^{2/3} $. In summary, $ \lambda \geq \min\{ \beta^{-4/3}, (1/8)^{2/3} \} $. Since $ a < b $ was arbitrary, this shows that $ \Lambda \geq \min\{ \beta^{-4/3}, (1/8)^{2/3} \} $.
\end{proof}

\subsection{Proof of example distributions} \label{app:exampledist}
Below we list the examples from \prettyref{sec:distance} and give their proofs.
\begin{enumerate}[(a)]
%\item If each $ X $ is i.i.d. $ \Unif(0, 1) $, then $ \prob{X \leq s \mid \bX \in \bt} = \frac{s-a}{b-a} $ and hence $ Q(p) = a + p(b-a) $. 
\item[(b)] If each $ X $ is i.i.d. $ \Beta(2, 1) $, then $ \prob{X\leq s \mid \bX \in \bt} = \frac{s^2-a^2}{b^2-a^2} $ and $ Q(p) = \sqrt{a^2 + p(b^2-a^2)} $ and hence we can write
$$
\frac{Q(p)-a}{b-a} = \frac{\sqrt{pb^2 + (1-p)a^2}-a}{b-a} = \frac{p(a+b)}{\sqrt{pb^2 + (1-p)a^2}+a}.
$$
By concavity of the square root function, $ \sqrt{pb^2 + (1-p)a^2} \geq pb + (1-p)a $. Thus, $ \frac{\sqrt{pb^2 + (1-p)a^2}-a}{b-a} \geq p $ for all $ a < b $ and so $ q_1(p) = p $.
For the other direction, we can choose $ q_2(p) = \sqrt{p} $ if we can show that $ \frac{\sqrt{p}(a+b)}{\sqrt{pb^2 + (1-p)a^2}+a} \leq 1 $, or equivalently, that 
\begin{align} \label{eq:sqrroot}
& (\sqrt{pb^2 + (1-p)a^2}+a)^2 - (\sqrt{p}(a+b))^2 \nonumber \\ \quad & = 2a(\sqrt{pb^2 + (1-p)a^2}+a - p(a+b)) \\ \nonumber & \geq 0.
\end{align} 
By concavity of the square root function, the expression in \prettyref{eq:sqrroot} is at least $ 4a^2(1-p) \geq 0 $.
%\item This result can be deduced by first bounding $ \prob{X \leq s \mid \bX \in \bt} $ between the same quantity for (b) and for (a):
%$$
%\frac{s^2-a^2_1}{b^2_1-a^2_1} \leq \frac{(s-a_1)(s+a_1+b_2+a_2)}{(b_1-a_1)(b_1+a_1+b_2+a_2)} \leq \frac{s-a_1}{b_1-a_1},
%$$
%which in turn, using $ p = \prob{X \leq Q(p) \mid \bX \in \bt} $, implies the desired estimates for $ Q(p) $.
\item[(c)] If each $ X $ is i.i.d. $ \Beta(1/2, 1) $, then $ \prob{X\leq s \mid \bX \in \bt} = \frac{\sqrt{s}-\sqrt{a}}{\sqrt{b}-\sqrt{a}} $ and $ Q(p) = (\sqrt{a} + p(\sqrt{b}-\sqrt{a}) )^2 $ and hence we can write
$$
\frac{Q(p)-a}{b-a} = p\left(1-(1-p)\frac{\sqrt{b}-\sqrt{a}}{\sqrt{b}+\sqrt{a}}\right).
$$
The proof is finished by noting that $ 0 \leq \frac{\sqrt{b}-\sqrt{a}}{\sqrt{b}+\sqrt{a}} \leq 1 $ for all $ a< b $.
%\item The proof strategy is similar to (c) in that we bound $ \prob{X \leq s \mid \bX \in \bt} $ between the same quantity for (a) and for (d):
%$$
%\frac{s-a_1}{b_1-a_1} \leq \frac{(\sqrt{s}-\sqrt{a_1})\left(\sqrt{s}+\sqrt{a_1}+\sqrt{b_2}+\sqrt{a_2}\right)}{(\sqrt{b_1}-\sqrt{a_1})\left(\sqrt{b_1}+\sqrt{a_1}+\sqrt{b_2}+\sqrt{a_2}\right)} \leq \frac{\sqrt{s}-\sqrt{a_1}}{\sqrt{b_1}-\sqrt{a_1}},
%$$
%which in turn, using $ p = \prob{X \leq Q(p) \mid \bX \in \bt} $, implies the desired estimates for $ Q(p) $.
\end{enumerate}

\subsection{Miscellaneous lemmas}

\begin{lemma} \label{lmm:cond}
Suppose $ A $ and $ B $ are events and let $ W $ and $ \{ W_i \}_{i=1}^n $ be i.i.d. random variables. If $ M = \sum_{i=1}^n \indc{W_i \in A} $ and $ M' = \sum_{i=1}^n \indc{W_i \in A\cap B} $, then $ M' \mid M = m \sim \Binom(m, \prob{W \in B\mid W \in A}) $.
\end{lemma}
\begin{proof}
Let $ p_A = \prob{W \in A} $ and $ p_{AB} = \prob{W \in A \cap B} $. Note that $ p_{AB}/p_A = \prob{W \in B\mid W \in A} $. Then
\begin{align*}
\prob{M' = m' \mid M = m} & = \frac{\prob{M' = m', \; M = m}}{\prob{M=m}} \\ & = \frac{\binom{n}{m}\binom{m}{m'}p_{AB}^{m'}(p_A-p_{AB})^{m-m'}(1-p_A)^{n-m}}{\binom{n}{m}p_A^m(1-p_A)^{n-m}} \\
& = \binom{m}{m'}(p_{AB}/p_A)^{m'}(1-p_{AB}/p_A)^{m-m'},
\end{align*}
which is the mass function of $ \Binom(m, \prob{W \in B\mid W \in A}) $.
\end{proof}

Before we state the next lemma, let us first introduce some notation. For a function $ g $, we write $ g(x-) $ (resp. $ g(x+) $) to denote the left (resp. right) side limits of $ g $ at $ x $, i.e., $ g(x-) = \lim_{z \uparrow x}g(z) $ and $ g(x+) = \lim_{z \downarrow x}g(z) $.

\begin{lemma}
Let $ F $ be a distribution function and $ Q $ its quantile function, i.e., $ Q(p) = \inf\{ x \in \mathbb{R}: p \leq F(x) \} $. Then,
\begin{equation} \label{eq:s}
Q(F(x)) \leq x \leq Q(F(x)+), \quad x\in\mathbb{R},
\end{equation}
and
\begin{equation} \label{eq:p}
F(Q(p)-) \leq p \leq F(Q(p)), \quad p \in (0, 1).
\end{equation}
Furthermore, if $ F $ is continuous and strictly increasing, then all inequalities are equalities.
\end{lemma}

\begin{proof}
These are standard facts from probability theory and can be deduced from the Galois inequalities. See, for example, %\citep[Section 3.12]{williams1991probability}.
\citep[Section 2.5.2]{resnick2003probability}.
\end{proof}

\end{document}